\documentclass[letterpaper]{article}
\usepackage{ijcai15}
\usepackage{times}
\pdfoutput=1
\usepackage{helvet}
\usepackage{courier}
\usepackage[table]{xcolor}
\usepackage{indentfirst,amssymb,amsmath,amsthm,latexsym,graphicx}
\usepackage{todonotes}

\newtheorem{theorem}{Theorem}
\newtheorem{proposition}{Proposition}
\newtheorem{corollary}{Corollary}
\newtheorem{lemma}{Lemma}

\newtheorem{definition}{Definition}
\newtheorem{example}{Example}

\usepackage{macros}

\title{On the Complexity of Semantic Integration of OWL Ontologies}
\author{Yevgeny~Kazakov\\ Institute of Artificial Intelligence,\\ University of Ulm, Germany \\ yevgeny.kazakov@uni-ulm.de \And Denis~Ponomaryov\\ Institute of Informatics Systems,\\ Novosibirsk State University, Russia\\ ponom@iis.nsk.su }
\begin{document}

\pdfinfo{
/Title (On the Complexity of Semantic Integration of OWL Ontologies)
/Subject ()
/Author (Yevgeny Kazakov and Denis Ponomaryov)
}

\maketitle

% %====================================================================%%
\begin{abstract}
  We propose a new mechanism for integration of OWL ontologies using 
  semantic import relations. In contrast to the standard OWL importing,
  we do not require all axioms of the imported ontologies to be
  taken into account for reasoning tasks, but only their logical
  implications over a chosen signature. This property comes natural
  in many ontology integration scenarios, especially when the number of
  ontologies is large. In this paper, we study the complexity of reasoning
  over ontologies with semantic import relations and establish a
  range of tight complexity bounds for various fragments of OWL.

\end{abstract}

%%====================================================================%%

%%====================================================================%%
\section{Introduction and Motivation}
\label{sec:introduction}
%%====================================================================%%

Logic-based ontology languages such as OWL and OWL~2
\cite{GHM+:08:OWL} are becoming increasingly popular means for
representation, integration, and querying of information, particularly
in life sciences, such as Biology and Medicine. For example, a
repository of Open Biological and Biomedical Ontologies
\cite{SAR+:07:OBO} is comprised of over eighty specialised ontologies
on such diverse topics as molecular functions, biological processes,
and cellular components. Ontology integration, in particular, aims at
organizing information on different domains in a modular way so that
information from one ontology can be reused in other ontologies. For
example, the ontology of diseases may reference anatomical structures
to describe the location of diseases, or genes with which the diseases
are likely to be correlated.

Integration of multiple ontologies in OWL is organized via
\emph{importing}:
an OWL ontology can refer to one or several other OWL ontologies,
whose axioms must be implicitly present in the ontology. The importing
mechanism is simple in that it does not require any
significant modification of the underlying reasoning algorithms: in
order to answer a query over an ontology with an import declaration, it
is sufficient to apply the algorithm to the \emph{import closure}
consisting of the axioms of the ontology plus the axioms of the
ontologies that are imported (possibly indirectly). For example, if
ontology $\O_1$ imports ontologies $\O_2$ and $\O_3$, each of which,
in turn, imports ontology $\O_4$, then the import closure of $\O_1$
consists of all axioms of $\O_1-\O_4$. Provided these axioms altogether are expressible in the same fragment of OWL as each $\O_1-\O_4$ is, a reasoning algorithm for this fragment can be used to answer queries over the import closure of $\O_1$. Then, since the size of such import closure is the same as the combined size of all ontologies involved, the computational complexity of reasoning over ontologies with imports remains the same as for ontologies without imports.

Although the OWL importing mechanism may work well for simple ontology
integration scenarios, it may cause some undesirable side effects if
used in complex import situations. To illustrate the problem,
suppose that in the above example, $\O_4$ is an ontology describing a
typical university. It may include concepts such as \cn{Student},
\cn{Professor}, \cn{Course}, and axioms stating, e.g., that each
professor must teach some course and that students are disjoint with
professors:
\begin{align}
\cn{Professor}&\sqsubseteq\exists\rn{teaches}.\cn{Course},
\label{eq:university:teaches}\\
\cn{Student}\sqcap\cn{Professor}&\sqsubseteq\bot.
\label{eq:university:disjointness}
\end{align}
Now suppose that $\O_2$ and $\O_3$ are ontologies describing respectively,
Oxford and Cambridge universities that use $\O_4$ as a prototype.
For example, $\O_2$ may include \emph{mapping axioms}
\begin{align}
\cn{OxfordStudent}&\equiv\cn{Student},
\label{eq:oxford:student:mapping}\\
\cn{OxfordProfessor}&\equiv\cn{Professor},
\label{eq:oxford:professor:mapping}\\
\cn{OxfordCourse}&\equiv\cn{Course},
\label{eq:oxford:course:mapping}
\end{align}
from which, due to 
\eqref{eq:university:teaches}, it is now possible to conclude that
each Oxford professor must teach some Oxford course: 
\begin{equation}
\cn{OxfordProfessor}\sqsubseteq\exists\rn{teaches}.\cn{OxfordCourse}.
\end{equation}
Likewise, using similar mapping axioms in $\O_3$, it is possible to
obtain that Cambridge students and professors are disjoint:
\begin{equation}
\cn{CambridgeStudent}\sqcap\cn{CambridgeProfessor}\sqsubseteq\bot.
\end{equation}
Finally, suppose that $\O_1$ is an ontology aggregating information
about UK universities, importing, among others, the ontologies $\O_2$
and $\O_3$ for Oxford and Cambridge universities.

Although the described scenario seems plausible, there will be some
undesirable consequences in $\O_1$ due to the mapping axioms of $\O_2$ and $\O_3$
occurring in the import closure:
\begin{align}
\cn{OxfordStudent}&\equiv\cn{CambridgeStudent},\\
\cn{OxfordProfessor}&\equiv\cn{CambridgeProfessor},\\
\cn{OxfordCourse}&\equiv\cn{CambridgeCourse}.
\end{align}
The main reason for these consequences is that the ontologies $\O_2$
and $\O_3$ happen to reuse the same ontology $\O_4$ in two different
and incompatible ways. Had they instead used two different `copies' of
$\O_4$ as prototypes (with concepts renamed apart), no such problem would take
place. Arguably, the primary purpose of $\O_2$ and $\O_3$ is to
provide semantic description of the vocabulary for Oxford and
Cambridge universities, and the means of how it is achieved---either
by writing the axioms directly or reusing third party ontologies such
as $\O_4$---should be an internal matter of these two ontologies and
should not be exposed to the ontologies that import them.

Motivated by the described scenario, in this paper we consider a
refined mechanism for importing of OWL ontologies called
\emph{semantic importing}. The main difference with the standard OWL
importing, is that each import is limited only to a subset of symbols.
Intuitively, only logical properties of these symbols entailed
by the imported ontology should be imported.
These symbols can be regarded as the public (or external) vocabulary
of the imported ontologies. For example, ontology $\O_2$ may declare
the symbols $\cn{OxfordStudent}$, $\cn{OxfordProfessor}$,
$\cn{OxfordCourse}$, and $\cn{teaches}$ public, leaving the remaining
symbols only for the internal use.

The main results of this paper are tight complexity bounds for
reasoning over ontologies with semantic imports. We consider
ontologies formulated in different fragments of OWL starting from the
propositional (role-free) Horn fragment $\mathcal{H}$, full propositional (role-free) fragment $\mathcal{P}$, and concluding with the
Description Logic (DL) $\mathcal{SROIQ}$, which corresponds to OWL~2. We
also distinguish the case of acyclic imports, when ontologies cannot
(possibly indirectly) import themselves.
Our completeness results for ranges of DLs are summarized in the following table, where $\mathfrak{a}$ and $\mathfrak{c}$ denote the case of acyclic/cyclic imports respectively:

{\footnotesize
\begin{tabular}{c|c|c}
DLs & Completeness & Theorems \\ \hline
$\EL$ -- $\ELpp$ & $\EXPTIME$ $\mathfrak{a}$&  \ref{Teo_Hardness_EL-acyclic},  \ref{theorem:upper:acyclic} \\ \hline

containing $\EL$ & $\re$ (undecidable) $\mathfrak{c}$ & \ref{Teo_Hardness_EL-cyclic}, \ref{Teo_UpperBoundCyclic} \\ \hline

$\ALC$ -- $\SHIQ$ & $\DEXPTIME$ $\mathfrak{a}$& \ref{Teo_Hardness_ALC}, \ref{theorem:upper:acyclic} \\ \hline

$\R$ -- $\SRIQ$ & $\TEXPTIME$ $\mathfrak{a}$ &  \ref{Teo_Hardness_R}, \ref{theorem:upper:acyclic}\\ \hline

$\ALCHOIF$ -- $\SHOIQ$ & $\coNDEXPTIME$ $\mathfrak{a}$ & \ref{Teo_Hardness_ALCHOIF}, \ref{theorem:upper:acyclic} \\ \hline

$\ROIF$ -- $\SROIQ$ & $\coNTEXPTIME$ $\mathfrak{a}$ & \ref{Teo_Hardness_ROIF}, \ref{theorem:upper:acyclic} \\ \hline

$\mathcal{H}$ - $\mathcal{P}$ & $\EXPTIME$ $\mathfrak{c}$ & \ref{Teo_Hardness_HGeneral}, \ref{Teo_UpperBound_PGeneral} \\ \hline

$\mathcal{H}$ - $\mathcal{P}$ & $\PSPACE$ $\mathfrak{a}$ &  \ref{Teo_Hardness_HAcyclic}, \ref{Teo_UpperBound_PAcyclic} \\ \hline

\end{tabular}
}

The paper is organized as follows. In Sections \ref{sec:related}-\ref{sec:preliminaries} we describe related work and introduce basic notations. In Section \ref{sec:importing} we formulate the problem of entailment in ontology networks. In Section \ref{Sect_Expressiveness} we prove results  on the expressiveness of ontology networks and use them to show hardness of entailment in Section \ref{Sect_Hardness}. Finally, in Section \ref{Sect_Reduction2ClassicalEntailment} we demonstrate that entailment in ontology networks reduces to standard entailment in Description Logics and use these results to prove in Section \ref{Sect_Membership} that the obtained complexity bounds are tight.   %Section \ref{Sect_Conclusions} concludes.

%\begin{enumerate}
  
%  \item The logics between the propositional Horn and the propositional
%  logic are $\PSPACE$-complete for acyclic imports and
%  $\EXPTIME$-complete for arbitrary imports;
  
%  \item The logics containing the DL $\EL$ are undecidable ($\RE$-complete) for arbitrary imports;
  
%  \item The logics between the DLs $\EL$ and $\ELpp$ are $\EXPTIME$-complete for acyclic imports;
   
%  \item The logics between the DLs $\ALC$ and $\SHIQ$ are
%  $\DEXPTIME$-complete for acyclic imports;
  
%  \item The logics between the DLs $\ALCHOIF$ and $\SHOIQ$ are
%  $\NDEXPTIME$-complete for acyclic imports;
  
 % \item The logics between the DLs $\SR$ and $\SRIQ$ are $\TEXPTIME$-complete for acyclic imports;
  
%  \item The logics between the DLs $\SROIF$ and $\SROIQ$ are $\NTEXPTIME$-complete for acyclic imports.
  
%\end{enumerate}

% \begin{table}
% \caption{Summary of the complexity results for entailment in ontology networks (all complexity bounds are exact)}
% 
% \rowcolors{2}{gray!15}{}
% \begin{tabular}{|l|l|l|}
% \hline
% & Acyclic Networks & Cyclic Networks\\
% \hline
% $\mathcal{H}$ & \PSPACE & \EXPTIME \\
% $\mathcal{P}$ & \PSPACE & \EXPTIME \\
% $\EL$ & \EXPTIME & \RE \\
% $\ALC$ & \DEXPTIME & \RE \\
% $\SHIQ$ & \DEXPTIME & \RE \\
% $\SHOIF$ & \NDEXPTIME & \RE \\
% $\SRIQ$ & \TEXPTIME & \RE \\
% $\SROIQ$ & \NTEXPTIME & \RE \\
% \hline
% \end{tabular}
% \end{table}

%%====================================================================%%

%%====================================================================%%
\section{Related Work}
\label{sec:related}
%%====================================================================%%

Frameworks for combining ontologies share the natural view that interpretations of linked ontologies must satisfy certain correspondence constraints. Most existing approaches (e.g. see an overview in \cite{TowardsFormalComparison}) consider a model of a combination of ontologies as a tuple of interpretations, one for each ontology, with \textit{correspondence relations} between the interpretation domains. These relations allow for propagation of semantics of entities (e.g., concepts, roles) from one ontology to another by providing interpretation for constructs stating links between ontology entities. The constructs are bridge rules in DDL \cite{DDL-Journal}, local/foreign symbol labels in PDL \cite{PDL-Chapter} and the approach of \cite{Pan-SemanticImport}, link properties in $\mathcal{E}$-connections \cite{EConnections-Chapter}, and alignment relations in \cite{AlignmentBasedModules}. The last approach originates from the field of ontology matching \cite{OntologyMatchingSurvey}, which is a neighbour topic lying out of the scope of our paper, since it concentrates on computing matchings, but not on reasoning with them.  %E-connections are not suitable for refinement. From Cuenca-Grau: It is worth emphasizing here that E-Connections are not a suitable technique for combining ontologies dealing with highly overlapping domains, which pre- vents its use in some important Knowledge Engineering applications, such as ontology refinement.
The semantics for a combination of ontologies proposed in these approaches are in general not compatible with the conventional OWL importing mechanism. If an ontology $\O_1$ references some ontology $\O_2$, then %the semantics of 
correspondence relations guarantee propagation of certain entailments expressible in the language of $\O_2$ into the ontology $\O_1$. %In many approaches, reasoning over a combination of theories wrt proposed semantics %Reasoning algorithms that have been developed in the scope of the mentioned approaches essentially implement exchange of these entailments (see also papers \cite{LocalTableaux,AspectsDistributedOntReasoning,TableauxFederatedAlgo,
%WorkingWithMultipleOntologies,TableauxAlgosEConnections}. 
As a rule, the class of propagated entailments is not broad enough to simulate entailment form the union of ontologies, as required in OWL importing. The approach of \cite{ImportByQuery} tries to bridge this gap by putting restrictions on combined ontologies, i.e. by considering conservative importing. %situations. 
%Propagation of some class of entailments expressible in a given DL is also the underlying principle in \cite{ImportByQuery}, although the approach is not related to the above mentioned semantics. For a simple example, if $\O_1$ borrows some concepts $C,D$ from $\O_2=\{C\dleq D\}$, then $\O_1$ entails $C\dleq D$. However, correspondence relations may rule out propagation of some semantic relationships which are important in real-world importing situations. For assume that a finance department of a university develops a classifier ontology $\O_1=\{WorkingGroup \dleq StateFundedUnit\}$ on top of the university ontology $\O_2$, relying on the semantics of concept names given in the axioms $\{GroupHead \dleq \ex{ leads}.WorkingGroup, \ \ex{ leads}.StateFundedUnit \dleq AdvanceHolder\}$ of $\O_2$. Given that every $WorkingGroup$ of the university is a $StateFundedUnit$ from the point of view of the finance department, it holds that every $GroupHead$ is $AdvanceHolder$. In fact, $\O_2$ entails the implication $WorkingGroup\dleq StateFundedUnit \rightarrow GroupHead \dleq AdvanceHolder$ and it does not appear possible to propagate this semantic relationship from into $\O_1$ when using correspondence relations.
%Nevertheless, these relations are still able to propagate information not expressible in the underlying ontology languages, with the typical case being non-emptiness of concept interpretations. 
Semantics given by tuples of interpretations may cause undesired effects, when combining two ontologies $\O_1,\O_2$ which both refine the same ontology $\O$ (cf. the motivating example from the introduction). %For example, if $\O_1$ says that interpretation of a concept A from $\O$ is empty, while $\O_2$ says the opposite, then there may not exist a consistent combination of $\O,\O_1,\O_2$. 
Ontologies $\O_i$ may refine concepts of $\O$ in different ways, which may conflict to each other, while being consistent separately. The semantics given by tuples of interpretations makes supporting such integration scenarios problematic, since a single interpretation of $\O$ must be in correspondence with interpretations of both ontologies, $\O_1$ and $\O_2$. 
The integration mechanism proposed in this paper is conceptually  simple. In order to reference an external ontology $\O$ from a local one, one has only to specify an import relation, which defines a set of symbols, whose semantics should be borrowed from $\O$. The symbols can be used freely in the axioms of the local ontology, no additional language constructs are required. This resembles the approaches \cite{PDL-Chapter,Pan-SemanticImport}, although theoretically, we do not distinguish between local and external symbols in ontologies (we note that this feature can be easily integrated). Importantly, in our approach every ontology has its own view on ontologies it refines and the views are independent between ontologies unless coordinated by the `topology' of import relations.

%%====================================================================%%

%%====================================================================%%
\section{Preliminaries}
\label{sec:preliminaries}
%%====================================================================%%
We assume that the reader is familiar with the family of Description Logics from $\EL$ to $\SROIQ$, for which the syntax is defined using a recursively enumerable alphabet consisting of infinite disjoint sets $\Nc$, $\Nr$, $\Ni$ of \emph{concept names} (or \emph{primitive concepts}), \emph{roles}, and \emph{nominals}, respectively. %Complex \emph{concepts} and \emph{axioms} are defined inductively using conjunction ($C \sqcap D$) and existential restriction ($\exists r. C$), where $C, D$ are concepts and $r$ is a role. An \emph{ontology} is a finite non-empty set of (\emph{concept inclusion}) axioms of the form $C \sqsubseteq D$. 
%An ontology $\O$ is called \emph{role-free} if no role occurs in axioms of $\O$. %For an ontology $\O$, we denote by $\sub(\O)$ the set of all complex and primitive concepts occurring in axioms of $\O$. 
We also consider DLs $\mathcal{H}$ and $\mathcal{P}$, which are the role-free fragments of $\EL$ and $\ALC$, respectively. Thus, $\mathcal{P}$ corresponds to the classical propositional logic and $\mathcal{H}$ corresponds to the Horn fragment thereof.

The semantics of DLs is given by means of (first-order) interpretations. An \emph{interpretation} ${\cal{I}}=\langle \Delta, \cdot^{\cal{I}}\rangle$ consists of a non-empty set $\Delta$, the \emph{domain} of $\cal{I}$, and an \emph{interpretation function} $\cdot^{\cal{I}}$, that assigns to each $A \in N_C$ a subset $A^{\cal{I}} \subseteq \Delta$, to each $r \in N_R$ a binary relation $r^{\cal{I}} \subseteq \Delta \times \Delta$, and to each $a\in\Ni$ an element of the domain $\Delta$. %This assignment is extended to complex concepts by interpreting $C \sqcap D$ as $C^{\cal{I}} \cap D^{\cal{I}}$ and $(\exists r.C)^{\cal{I}}$ as $\{a\in \Delta \ |\  \exists b \in\Delta : (a,b)\in r^{\cal{I}} \  \text{and} \  b\in  C^{\cal{I}}\}$. 
An interpretation $\cal{I}$ \emph{satisfies} a concept inclusion $C \sqsubseteq D$, written ${\cal{I}} \models C \sqsubseteq D$, if $C^{\cal{I}} \subseteq D^{\cal{I}}$ holds. An ontology is a set of concept inclusions which are called ontology \emph{axioms}. $\cal{I}$ is a \emph{model} of an ontology $\cal{O}$, written $\cal{I} \models \cal{O}$, if $\cal{I}$ satisfies all axioms of $\cal{O}$. An ontology $\cal{O}$ \emph{entails} a concept inclusion $C \sqsubseteq D$, in symbols ${\cal{O}} \models C \sqsubseteq D$, if every model of $\cal{O}$ satisfies $C \sqsubseteq D$. As usual, for concepts $C,D$, the equivalence $C\equiv D$ stands for the pair of concept inclusions $C\dleq D$ and $D\dleq C$.

A \emph{signature} is a subset of $\Nc\cup\Nr\cup\Ni$. Interpretations $\I$ and $\J$ are said to \emph{agree} on a signature $\Si$, written as $\I=_\Si \J$, if the domains of $\I$ and $\J$ coincide and the interpretation of $\Si$-symbols in $\I$ is the same as in $\J$. We denote the reduct of an interpretation $\I$ onto a signature $\Sigma$ as $\I |_\Sigma$. The signature of a concept $C$, denoted as $\sig(C)$, is the set of all concept names, roles, and nominals occurring in $C$. The signature of a concept inclusion or an ontology is defined identically.  %If $\O$ and $\O'$ are ontologies and $\Sigma$ a signature, then the notation $\O=_\Sigma\O'$ means that $\{\I |_{\Sigma} \mid \I\models\O\}=\{\I |_{\Sigma} \mid \I\models \O'\}$.

%%%%%% UNCOMMENT, IF YOU NEED THIS FOR HARDNESS PROOFS FOR P- AND H- ONTOLOGY NETWORKS
%Let $\I=(\Delta, \cdot^\I)$ be an interpretation and $r$ be a role. If $x\in\Delta$ and $\Si$ is a signature then $x^{\Si}$ denotes the set of all concept names $A\in\Si$ such that $x\in A^\I$. Every time we use the notation $x^{\Si}$, the interpretation $\I$ will be clear from the context. For a natural number $m\geqslant 1$, an element $y\in\Delta$ is called \emph{$(r,m)$-successor of $x\in\Delta$} if there is a sequence $x_0,\ldots , x_m$ of elements in $\Delta$ such that $x_0=x$, $x_m=y$ and $\langle x_i,x_{i+1}\rangle\in r^\I$, for all $0\leqslant i < m$. An element $y$ is called \emph{$r$-reachable from $x$} if $y$ is $(r,m)$-successor of $x$ for some $m$.

%%====================================================================%%

%%====================================================================%%
\section{Semantic Importing}
\label{sec:importing}
%%====================================================================%%

Given a signature $\Sigma$, suppose one wants to import into an ontology $\O_1$ the semantics of $\Sigma$-symbols defined by some other ontology $\O_2$, while ignoring the rest of the 
symbols from $\O_2$. Intuitively, importing the semantics of $\Sigma$-symbols means
reducing the class of models of $\O_1$ by removing those models that violate the
restrictions on interpretation of these symbols, which are imposed by the axioms of $\O_2$:

\begin{definition}
\label{def:import:relation}
A \emph{(semantic) import relation} is a tuple $\pi=\tuple{\O_1,\Sigma,\O_2}$
where $\O_1$ and $\O_2$ are ontologies and $\Sigma$ a signature.
In this case, we say that \emph{$\O_1$ imports $\Sigma$ from $\O_2$}. 
We say that a model $\I\models\O_1$ satisfies the import relation
$\pi$ if there exists a model $\J\models\O_2$ such that $\I=_\Sigma \J$.
\end{definition}

\begin{example}
Consider the import relation $\pi=\tuple{\O_1,\Sigma,\O_2}$, with
$\O_1=\set{B\sqsubseteq C}$, $\O_2=\set{A\sqsubseteq \exists
r.B,\,\exists r.C\sqsubseteq D}$, and $\Sigma=\set{A,B,C,D}$.
It can be easily shown using Definition~\ref{def:import:relation} that
a model $\I\models\O_1$ satisfies $\pi$ if and only if $\I\models
A\sqsubseteq D$.
\end{example}

Note that if $\Sigma$ contains all symbols in $\O_2$ then
$\I\models\O_1$ satisfies $\pi=\tuple{\O_1,\Sigma,\O_2}$ if and only
if $\I\models\O_1\cup\O_2$. That is, the standard OWL import relation
is a special case of the semantic import relation, when the signature
contains all the symbols from the imported ontology.

If $\O$ has several import relations
$\phi_i=\tuple{\O,\Sigma_i,\O_i}$, $(1\le i\le n)$, one can define the
entailment from $\O$ by considering only those models of $\O$ that satisfy all
imports: $\O\models\alpha$ if $\I\models\alpha$ for every
$\I\models\O$ which satisfies all $\pi_1,\dots,\pi_n$. In practice,
however, import relations can be nested: imported ontologies can
themselves import other ontologies and so on. The following definition
generalizes entailment to such situations.

\begin{definition}
\label{def:ontology:network}

An \emph{ontology network} is a finite 
set $\N$ of import relations between ontologies. For a DL $\LL$, a \emph{$\LL$-ontology network} is a network, in which every ontology is a set of $\LL$-axioms. 
A \emph{model agreement} for $\N$ (over a domain $\Delta$) 
is a mapping $\mu$ that assigns to
every ontology $\O$ occurring in $\N$ a class $\mu(\O)$ of models of
$\O$ with domain $\Delta$ such that for every $\tuple{\O_1,\Sigma,\O_2}\in\N$ and every
  $\I_1\in\mu(\O_1)$ there exists $\I_2\in\mu(\O_2)$ such that $\I_1=_\Sigma\I_2$. 
An interpretation $\I$ is a model of $\O$ in the network $\N$
(notation $\I\models_\N \O$) if there exists a model agreement $\mu$
for $\N$ such that $\I\in\mu(\O)$. An ontology $\O$ entails a concept inclusion
$\varphi$ in the network $\N$ (notation $\O\models_\N\varphi$) if
$\I\models\varphi$, whenever $\I\models_\N\O$.

\end{definition}

An ontology network can be seen as a labeled directed multigraph in
which nodes are labeled by ontologies and edges are labeled by sets of
signature symbols. Each edge in this graph, thus, represents an
import relation between two ontologies. Note that
Definition~\ref{def:ontology:network} also allows for \emph{cyclic
networks} if this graph is cyclic. That is, an ontology may refer to
itself through a chain of import relations. Note that if
$\O\models\varphi$ then $\O\models_\N\varphi$, for every network $\N$.

%Since in a network $\N$ an ontology $\O$ may import symbols not present in $\sig(\O)$, we introduce a special notation $\sig_\N(\O)$ for the signature of $\O$ in $\N$ defined as $\sig(\O)\cup\Sigma$, where $\Sigma$ is the union of signatures, which $\O$ imports in $\N$. 

\begin{example}
Consider the following (cyclic) network
$\N=\set{\tuple{\O,\Sigma',\O'},\tuple{\O',\Sigma,\O}}$, where
\begin{itemize}
  \item $\O=\set{A\sqsubseteq B,\,A\equiv A',\,B\equiv B'}$
  \item $\O'=\set{A\equiv\exists r.A',\,B\equiv\exists r.B'}$
  \item $\Sigma=\set{A,B,r}$, $\Sigma'=\set{A',B',r}$  
\end{itemize}
Let $\mu$ be any model agreement for $\N$.
Since $\tuple{\O',\Sigma,\O}\in\N$, by
Definition~\ref{def:ontology:network}, for every $\I'\in\mu(\O')$
there exists $\I\in\mu(\O)$ such that $\I'=_{\Sigma}\I$. Since $\I\models\O\models A\sqsubseteq B$ and
$\set{A,B}\subseteq\Sigma$, we have $\I'\models A\sqsubseteq
B$. As $\I'\models\O'$, it also holds $\I'\models\ex{r}.A'\dleq\ex{r}.B'$ for every $\I'\in\mu(\O')$ and thus: 
\begin{equation}
\O'\models_\N\ex{r}.A'\dleq\ex{r}.B'. 
\end{equation}
Similarly, since $\tuple{\O,\Sigma,\O'}\in\N$, for every
$\I\in\mu(\O)$ there exists $\I'\in\mu(\O')$ such that
$\I=_{\Sigma'}\I'$.
Since $\I'\models \ex{r}.A'\dleq\ex{r}.B'$ and
$\set{A',B',r}\subseteq\Sigma'$, we have $\I\models
\ex{r}.A'\dleq\ex{r}.B'$ and
since $\I\models\O$, it holds $\I\models \ex{r}.A\dleq\ex{r}.B$, for every $\I\in\mu(\O)$, and hence: 
\begin{equation}
\O\models_\N \ex{r}.A\dleq\ex{r}.B.
\end{equation}
By repeating these arguments we similarly obtain:
\begin{align}
\O'&\models_\N \ex{r}.\ex{r}.A'\dleq\ex{r}.\ex{r}.B',\\
\O&\models_\N \ex{r}.\ex{r}.A\dleq\ex{r}.\ex{r}.B,\\
\O'&\models_\N \ex{r}.\ex{r}.\ex{r}.A'\dleq\ex{r}.\ex{r}.\ex{r}.B',\\
\O&\models_\N \ex{r}.\ex{r}.\ex{r}.A\dleq\ex{r}.\ex{r}.\ex{r}.B,
\end{align}
and so on. Note the matching nestings of $\exists r$ in these axioms.
\end{example}

In this paper, we are concerned with the complexity of \emph{entailment in ontology networks}, that is, given a network $\N$, an ontology $\O$ and an axiom $\varphi$, decide whether
$\O\models_\N\varphi$. We study the complexity of this problem wrt the \emph{size} of an ontology network $\N$, which is defined as the total length of axioms (considered as strings) occurring in ontologies from $\N$.

%%====================================================================%%

%%====================================================================%%

\section{Expresiveness of Ontology Networks}\label{Sect_Expressiveness}

%\Todo{Try not to nest theorems / lemmas / claims etc}

We illustrate the expressiveness of ontology networks by showing that acyclic networks allow for succinctly representing axioms with nested concepts and role chains of exponential size, while cyclic ones allow for succinctly representing infinite sets of axioms of a special form. \medskip

For a natural number $n\geqslant 0$, let $\exists (r,C)^{n}.D$ be a shortcut for the nested concept 
\begin{equation}\label{Eq_ELExptimeExponentConcepts}
\underbrace{\exists r.(C\sqcap\exists r.(C\sqcap\cdots\sqcap\exists r.(C}_{n~\text{times}}{}\sqcap D)\cdots))
\end{equation}

where $C$, $D$ are DL concepts and $r$ a role (in case $n=0$ the above concept is set to be $D$). For $n\geqslant 1$, let $(r)^n$ denote the role chain 
\begin{equation}
\underbrace{r\circ\ldots\circ r}_{n~\text{times}}
\end{equation}

For a given $n\geqslant 0$, let $\EXP{1}{n}$ be the notation for $2^n$ and for $k\geqslant 1$, let $\EXP{(k+1)}{n}=2^\EXP{k}{n}$. Then $\exists (r,C)^{\EXP{k}{n}}.D$ (respectively, $(r)^{\EXP{k}{n}}$) stands for a nested concept (role chain) of the form above having size exponential in $n$. 

In the following, we use abbreviations $\exists (r,C)^{n}:=\exists(r,C)^{n}.\top$ and $\exists
r^{n}.C:=\exists(r,\top)^{n}.C$. For roles $r,s$ and $n\geqslant 2$, let $(r)^{< n}\dleq s$ be an abbreviation for the set of role inclusions $\{(r)^k\dleq s \mid 1\leqslant k < n\}$. For $n\geqslant 1$, the expression $\forall r^{n}.C$ will be used as a shortcut for $\neg \exists r. \exists r^{n-1}.\neg C$ and for $n\geqslant 2$, $\forall r^{< n}.C$ will stand for $\bigsqcap_{1\leqslant m < n}\forall r^m.C$. 

%Let $\O$, $\O'$ be ontologies and $\N$ an ontology network. We say that $\O$ is expressible by $\O'$ if $\{\I |_{\sig(\O)} \mid \I\models\O\}=\{\I |_{\sig(\O)} \mid \I\models \O'\}$. 
Let $\O$ be an ontology and $\N$ an ontology network. $\O$ is said to be \emph{expressible} by $\N$ if there is an ontology $\O_\N$ in $\N$ such that $\{\I |_{\sig(\O)} \mid \I\models_\N\O_\N\}=\{\J |_{\sig(\O)} \mid \J\models \O\}$. In other words, it holds $\O_\N\models_\N\O$ and any model $\J\models\O$ can be expanded to a model $\I\models_\N\O_\N$. Note that this yields that $\O\models\varphi$ iff $\O_\N\models_\N\varphi$, for any concept inclusion $\varphi$ such that $\sig(\varphi)\subseteq\sig(\O)$. In case we want to stress the role of ontology $\O_\N$ in the network $\N$, we say that $\O$ is $(\N,\O_\N)$-\emph{expressible}. An axiom $\varphi$ is expressible by a network $\N$ if so is ontology $\O=\{\varphi\}$.

The next two lemmas follow immediately from the definition of expressibility.

\begin{lemma}\label{Lem_IteratedExpressibility}
Every ontology $\O$ is $(\N,\O)$-expressible, where $\N$ is the network consisting of the single import relation $\langle \O, \varnothing, \varnothing\rangle$.
If an ontology $\O_i$ is $(\N_i,\O_i')$-expressible, for a network $\N_i$, ontology $\O_i'$, and $i=1,2$, then $\O_1\cup\O_2$ is $(\N,\O_\N)$-expressible, for ontology $\O_\N=\varnothing$ and a network\footnote{Note that the size of $\N$ is linear in the sizes of $\N_1$ and $\N_2$.} $\N=\N_1\cup\N_2\cup\{\langle \O_\N, \sig(\O_i), \O_i' \rangle\}_{i=1,2}$. 
\end{lemma}

For an axiom $\varphi$ and a set of concepts $\{C_1,\ldots, C_n\}$, $n\geqslant 1$, let us denote by $\varphi[C_1\mapsto D_1,\ldots , C_n\mapsto D_n]$ the axiom obtained by substituting every concept $C_i$ with a concept $D_i$ in $\varphi$. For an ontology $\O$, let $\O[C_1\mapsto D_1,\ldots , C_n\mapsto D_n]$ be a notation for $\bigcup_{\varphi\in\O}\varphi[C_1\mapsto D_1,\ldots , C_n\mapsto D_n]$.

\begin{lemma}\label{Lem_Expressibility_SimpleSubstitution}
Let $\LL$ be a DL and $\O$ an ontology, which is expressible by a $\LL$-ontology network $\N$. Let $C_1,\ldots, C_n$ be $\LL$-concepts and $\{A_1,\ldots, A_n\}$ a set of concept names such that $A_i\in\sig(\O)$, for $i=1,\ldots ,n$ and $n\geqslant 1$. Then ontology $\tilde{\O}=\O[A_1\mapsto C_1,\ldots , A_n\mapsto C_n]$ is expressible by a $\LL$-ontology network, which is acyclic if so is $\N$ and has size polynomial in the size of $\N$ and $C_i$,  $i=1,\ldots ,n$.
\end{lemma}
\begin{proof}
Denote $\Sigma=\{A_1,\ldots, A_n\}$ and let $\Sigma'=\{A'_1,\ldots, A'_n\}$ be a set of fresh concept names. Consider ontology $\O'=\O\cup\{A_i\equiv A'_i\}_{i=1,\ldots ,n}$. By Lemma \ref{Lem_IteratedExpressibility}, $\O'$ is $(\N',\O_{\N'})$-expressible, for an ontology $\O_{\N'}$ and an acyclic $\LL$-ontology network $\N'$ having a linear size (in the size of $\N$). Consider ontology network $\N''=\langle \O_{\N''}, (\sig(\O')\setminus\Sigma)\cup\Sigma', \O_{\N'} \rangle$, where $\O_{\N''}=\varnothing$. Then obviously, ontology $\O''=\O[A_1\mapsto A'_1,\ldots , A_n\mapsto A'_n]$ is $(\N'',\O_{\N''})$-expressible. Similarly, by Lemma \ref{Lem_IteratedExpressibility}, ontology $\O''_C=\O''\cup\{A'_i\equiv C_i\}_{i=1,\ldots ,n}$ is $(\tilde{\N},\O_{\tilde{\N}})$-expressible, for an ontology $\O_{\tilde{\N}}$ and an acyclic $\LL$-ontology network $\tilde{\N}$ having a linear size (in the size of $\N$ and $C_i$, for $i=1,\ldots ,n$). Clearly, it holds $\O_{\tilde{\N}}\models_{\tilde{\N}}\tilde{\O}$. On the other hand, any model $\I\models\tilde{\O}$ can be expanded to a model $\J$ of ontology $\O''_C$ by setting $(A'_i)^\J=(C_i)^\I$, for $i=1,\ldots ,n$. Since $\O''_C$ is $(\tilde{\N},\O_{\tilde{\N}})$-expressible, it follows that $\I$ can be expanded to a model $\tilde{J}\models_{\tilde{\N}}\O_{\tilde{\N}}$ and therefore, $\tilde{\O}$ is $(\tilde{\N},\O_{\tilde{\N}})$-expressible. 
\end{proof}

Now we proceed to the results on the succinct representation of exponentially long axioms and infinite sets of axioms of a special form. We begin with lemmas on the expressiveness of acyclic $\EL$- and $\ALC$-ontology networks.

\begin{lemma}\label{Lem_Expressibility_EL+Exp}
An axiom $\varphi$ of the form $Z\equiv\exists (r,A)^{\IEXP{n}}.B$, where $Z,A,B\in\Nc$, and $n\geqslant 0$, is expressible by an acyclic $\EL$-ontology network of size polynomial in $n$.
\end{lemma}

\begin{proof}
We prove by induction on $n$ that there exists an acyclic $\EL$-ontology network $\N_n$ and ontology $\O_n$ such that $\varphi$ is $(\N_n,\O_n)$-expressible. For $n=0$, we define $\N_0=\{\langle \O_0, \varnothing, \varnothing\rangle\}$ and $\O_0=\{Z\equiv \ex{r}.(A\dcap B)\}$. 

%\begin{figure*}\label{Fig_BLBxptimeRepresentingNetwork}
%\makebox[17.5cm] {
%\put(-160,0){\includegraphics[height=0.8in,keepaspectratio]{Images/pspace_start_simplified.pdf}}

%\put(55,27){{\Large $\ldots$}}

%\put(85,0){\includegraphics[height=0.8in,keepaspectratio]{Images/pspace_end.pdf}}
%\put(-159,45){{\tiny $\O_{init}$}}
%\put(-153,27){{\small $u$}}
%\put(-143,35){{\tiny $\Si\!\cup\!\{\!A,\!\!B\!\}$}}
%\put(-115,48){{\tiny $\O_{main}$}}
%\put(-110,26){{\small $\O_1$}}
%\put(-82,47){{\small $\hat{\gamma}$}}
%\put(-70,36){{\tiny $\O^{12}_{copy}$}}
%\put(-43,47){{\small $\gamma$}}
%\put(-84,10){{\small $\hat{\sigma}$}}
%\put(-70,20){{\tiny $\O^{22}_{copy}$}} 
%\put(-43,10){{\small $\sigma$}}

%\put(-25,48){{\tiny $\O_{main}$}}
%\put(-21,26){{\small $\O_2$}}
%\put(6,47){{\small $\hat{\gamma}$}}
%\put(20,36){{\tiny $\O^{13}_{copy}$}}
%\put(20,20){{\tiny $\O^{23}_{copy}$}}
%\put(6,10){{\small $\hat{\sigma}$}}

%\put(85,34){{\tiny $\O^{1m}_{copy}$}}
%\put(112,45){{\small $\gamma$}}
%\put(85,20){{\tiny $\O^{2m}_{copy}$}}
%\put(112,10){{\small $\sigma$}}
%\put(130,48){{\tiny $\O_{main}$}}
%\put(131,26){{\small $\O_m$}}

%}
%\caption{Acyclic ontology network; the renamed import closure of $\O_1$ has size $2^m$.}
%\end{figure*}

In the induction step, let $\{Z\equiv\exists (r,A)^{\IEXP{n-1}}.B\}$ be $(\N_{n-1},\O_{n-1})$-expressible, for $n\geqslant 1$. Consider ontologies:
\begin{equation}
\O^1_{copy}=\{B\equiv U\}
\end{equation}\vspace{-0.7cm}

\begin{equation}
\O^2_{copy}=\{U\equiv Z\}
\end{equation}\vspace{-0.7cm}

\begin{equation}
\O_{n}=\varnothing
\end{equation}

Let $\N_{n}$ be the union of $\N_{n-1}$ with the set of the following import relations: $\langle \O^i_{copy}, \{Z,A,B,r\}, \O_{n-1}\rangle$, for $i=1,2$, $\langle \O_{n}, \{Z,A,U,r\}, \O^1_{copy} \rangle$, and $\langle \O_{n}, \{U,A,B,r\}, \O^2_{copy} \rangle$. %Intuitively, every model $\I\models_{\N_{n}}\O_{n-1}$ defines an $r$-chain of length $2^{n-1}$ labelled with concepts $D,B$, while ontologies $\O^{i}_{copy}$, $i=1,2$ are used to make two ``copies'' of this $r$-chain and the models of $\O_{n}$ in $\N_n$ implement coupling of these copies.

Let us verify that $\{\I |_{\sig(\varphi)} \mid \I\models_{\N_n}\O_n\}=\{\I |_{\sig(\varphi)} \mid \I\models\varphi\}$. By the induction assumption we have $\O_{n-1}\models_{\N_{n}} Z\equiv\exists (r,A)^{\IEXP{n-1}}.B$. Then by the definition of $\N_n$, it holds $\O^{1}_{copy}\models_{\N_n} Z\equiv\exists (r,A)^{\IEXP{n-1}}.U$ and  $\O^{2}_{copy}\models_{\N_n} U\equiv\exists (r,A)^{\IEXP{n-1}}.B$ and thus, $\O_n\models_{\N_n} \{ Z\equiv\exists (r,A)^{\IEXP{n-1}}.U,$ $U\equiv\exists (r,A)^{\IEXP{n-1}}.B\}$, which yields $\O_n\models_{\N_n} \varphi$. 

We now show that any model $\I\models\varphi$ can be expanded to a model $\J_n\models_{\N_n}\O_n$. Let us define $\J_n$ as an expansion of $\I$ by setting $U^{\J_n}=(\exists (r,A)^{\IEXP{n-1}}.B)^\I$. Clearly, it holds $\J_n\models\O_{n}$ and there exists a model $\I_2\models\O^2_{copy}$ such that $\I_2=_{\{U,A,B,r\}}\J_n$ and $Z^{\I_2}=U^{\I_2}$. We have $\I_2\models Z\equiv\exists (r,A)^{\IEXP{n-1}}.B$ and hence by the induction assumption, there is a model $\J^2_{n-1}\models_{\N_n}\O_{n-1}$ such that $\I_2=_{\{Z,A,B,r\}}\J^2_{n-1}$. Similarly, there is a model $\I_1\models\O^1_{copy}$ such that $\I_1=_{\{Z,A,U,r\}}\J_n$ and $B^{\I_1}=(\exists (r,A)^{\IEXP{n-1}})^{\I_1}$. We have $\I_1\models Z\equiv\exists (r,A)^{\IEXP{n-1}}.B$, thus by the induction assumption, there is a model $\J^1_{n-1}\models_{\N_n}\O_{n-1}$ such that $\I_1=_{\{Z,A,B,r\}}\J^1_{n-1}$. It follows that there exists a model agreement $\mu$ for $\N_n$ such that $\mu(\O_{n-1})=\{\J^1_{n-1},\J^2_{n-1}\}$, $\mu(\O^i_{copy})=\{\I_i\}$, for $i=1,2$, and $\mu(\O_n)=\{\J_n\}$, which means that $\J_n\models_{\N_n}\O_n$. 
\end{proof}

The following two statements are proved identically to Lemma \ref{Lem_Expressibility_EL+Exp}:

\begin{lemma}\label{Lem_Expressibility_EL+dleqExp}
An axiom of the form $Z\dleq\exists (r,A)^{\IEXP{n}}.B$, where $Z,A,B\in\Nc$ and $n\geqslant 0$, is expressible by an acyclic $\EL$-ontology network of size polynomial in $n$.
\end{lemma}

\begin{lemma}\label{Lem_Expressibility_ALC+Exp}
An axiom of the form $Z\equiv \forall r^{\IEXP{n}}.A$, where $Z,A\in\Nc$ and $n\geqslant 0$, is expressible by an acyclic $\ALC$-ontology network of size polynomial in $n$.
\end{lemma}

Next, we demonstrate the expressiveness of acyclic $\R$-ontology networks.

\begin{lemma}\label{Lem_Expressibility_RI+2Exp}
An axiom $\varphi$ of the form $(r)^{\IIEXP{n}}\dleq s$, where $r,s$ are roles and $n\geqslant 0$, is expressible by an acyclic $\R$-ontology network of size polynomial in $n$.
\end{lemma}

\begin{proof}
We use the idea of the proof of Lemma \ref{Lem_Expressibility_EL+Exp} and show by induction on $n$ that there exists an acyclic $\R$-ontology network $\N_n$ and ontology $\O_n\in\N_n$ such that $\varphi$ is $(\N_n,\O_n)$-expressible. For $n=0$, we define $\N_0=\{\langle \O_0, \varnothing, \varnothing\rangle\}$ and $\O_0=\{r\circ r \dleq s\}$. 

In the induction step, let $(r)^{\IIEXP{n-1}} \dleq s$ be $(\N_{n-1}, \O_{n-1})$-expressible, for $n\geqslant 1$. Consider ontologies:
\begin{equation}
\O^{1}_{copy} = \{s\dleq u\}
\end{equation}\vspace{-0.7cm}

\begin{equation}
\O^{2}_{copy} = \{u\dleq r\}
\end{equation}\vspace{-0.7cm}

\begin{equation}
\O_{n}=\varnothing
\end{equation}%\vspace{-0.7cm}

Let $\N_{n}$ be the union of $\N_{n-1}$ with the set of the following import relations: $\langle \O^i_{copy}, \{r,s\}, \O_{n-1}\rangle$, for $i=1,2$, $\langle \O_{n}, \{r,u\}, \O^1_{copy} \rangle$, and $\langle \O_{n}, \{u,s\}, \O^2_{copy} \rangle$.

We show that $\{\I |_{\sig(\varphi)} \mid \I\models_{\N_n}\O_n\}=\{\I |_{\sig(\varphi)} \mid \I\models\varphi\}$. By the induction assumption,  we have $\O_{n-1}\models_{\N_{n}} (r)^{\IIEXP{n-1}} \dleq s$. Hence by the definition of $\N_n$, it holds $\O^{1}_{copy}\models_{\N_n} (r)^{\IIEXP{n-1}} \dleq u$ and $\O^{1}_{copy}\models_{\N_n} (u)^{\IIEXP{n-1}} \dleq s$ and therefore, $\O_n\models_{\N_n} \{ (r)^{\IIEXP{n-1}} \dleq u,$ $(u)^{\IIEXP{n-1}}\dleq s\}$, which means that $\O_n\models_{\N_n} \varphi$. By using an argument like in the proof of Lemma \ref{Lem_Expressibility_EL+Exp} one can verify that any model $\I\models\varphi$ can be expanded to a model $\J_n\models_{\N_n}\O_n$ by setting $u^{\J_n}=((r)^{\IIEXP{n-1}})^\I$. 
\end{proof}

\begin{lemma}\label{Lem_Expressibility_RI+Less2Exp}
An ontology $\O$ given by the set of axioms $(r)^{<\IIEXP{n}}\dleq s$, where $r,s$ are roles and $n\geqslant 0$, is expressible by an acyclic $\R$-ontology network of size polynomial in $n$.
\end{lemma}

\begin{proof}
We use a modification of the proof of Lemma \ref{Lem_Expressibility_RI+2Exp} and show by induction on $n$ that there exists an acyclic $\R$-ontology network $\N_n$ and ontology $\O_n\in\N_n$ such that $\O$ is $(\N_n,\O_n)$-expressible. For $n=0$, we define $\N_0=\{\langle \O_0, \varnothing, \varnothing\rangle\}$ and $\O_0=\{r\dleq s\}$. 

In the induction step, suppose ontology  $\{(r)^{m}\dleq s \mid 1\leqslant m \leqslant \IIEXP{n-1}-1\}$ is $(\N_{n-1}, \O_{n-1})$-expressible, for $n\geqslant 1$. Consider ontologies:
\begin{equation}
\O^{1}_{copy} = \{s\dleq u_1\}
\end{equation}\vspace{-0.7cm}

\begin{equation}
\O^{2}_{copy} = \{u_1\dleq r, s\dleq u_2\}
\end{equation}\vspace{-0.7cm}

\begin{equation}
\O_{n}=\{u_1\dleq u_3, \  u_1\circ u_1\dleq u_3, \ u_2\circ u_3\dleq s, \ r\dleq s\}
\end{equation}%\vspace{-0.7cm}

Let $\N_{n}$ be the union of $\N_{n-1}$ with the set of the following import relations: $\langle \O^i_{copy}, \{r,s\}, \O_{n-1}\rangle$, for $i=1,2$, $\langle \O_{n}, \{r,u_1\}, \O^1_{copy} \rangle$, and $\langle \O_{n}, \{u_1,u_2\}, \O^2_{copy} \rangle$.

We show that $\{\I |_{\sig(\O)} \mid \I\models_{\N_n}\O_n\}=\{\I |_{\sig(\O)} \mid \I\models\O\}$. By the induction assumption, we have $\O_{n-1}\models_{\N_{n}} (r)^{< \IIEXP{n-1}}\dleq s$ and hence, by the definition of $\N_n$, it holds $\O_n\models_{\N_n} (r)^{<\IIEXP{n-1}}\dleq u_1$ and $\O_n\models_{\N_n} (u_1)^{< \IIEXP{n-1}}\dleq u_2$. Then
 $\O_n\models_{\N_n} (r)^{m}\dleq u_2$, for $1\leqslant m \leqslant (\IIEXP{n-1}-1)^2$. 

Since $\O_n\models_{\N_n} (r)^{< \IIEXP{n-1}}\dleq u_1$ and $\{u_1\dleq u_3, \  u_1\circ u_1\dleq u_3\}\subseteq\O_n$, it holds $\O_n\models_{\N_n} (r)^{m}\dleq u_3$, for $1\leqslant m \leqslant 2(\IIEXP{n-1}-1)$. Therefore, since $u_2\circ u_3\dleq s\in\O_n$, we obtain $\O_n\models_{\N_n} \{(r)^k\dleq s \mid 2\leqslant k \leqslant m\}$, for $m=(\IIEXP{n-1}-1)^2+2(\IIEXP{n-1}-1)=\IIEXP{n}-1$. Since $r\dleq s\in\O_n$, we conclude that $\O_n\models_{\N_n} (r)^{<\IIEXP{n}}\dleq s$.

By using an argument like in the proof of Lemma \ref{Lem_Expressibility_EL+Exp} one can verify that any model $\I\models (r)^{<\IIEXP{n}}\dleq s$ can be expanded to a model $\J\models_{\N_n}\O_n$ by setting $u_i^{\J}=\bigcup_{1\leqslant k \leqslant m_i} ((r)^k)^\I$, for $i=1,2,3$, where $m_1=\IIEXP{n-1}-1$, $m_2=(\IIEXP{n-1}-1)^2$, and $m_3=2(\IIEXP{n-1}-1)$.
\end{proof}

\begin{lemma}\label{Lem_Expressibility+dleq2ExpForall_r.A}
An axiom $\varphi$ of the form $Z\dleq \forall r^{\IIEXP{n}}.A$, where $Z,A\in\Nc$ and $n\geqslant 0$, is expressible by an acyclic $\R$-ontology network of size polynomial in $n$.
\end{lemma}
\begin{proof}
Consider ontology $\O$ consisting of axioms
\begin{equation}
Z \dleq \forall s.A, \ \ (r)^{\IIEXP{n}} \dleq s
\end{equation}
Clearly, $\O\models\varphi$ and any model $\I\models\varphi$ can be expanded to a model  $\J\models\O$ by setting $s^\J=((r)^{\IIEXP{n}})^\I$. By Lemmas \ref{Lem_IteratedExpressibility}, \ref{Lem_Expressibility_RI+2Exp}, $\O$ is expressible by an acyclic $\R$-ontology network of size polynomial in $n$, from which the claim follows.
\end{proof}

The following statement is a direct consequence of Lemmas \ref{Lem_IteratedExpressibility} and \ref{Lem_Expressibility_RI+Less2Exp} and is proved identically to Lemma \ref{Lem_Expressibility+dleq2ExpForall_r.A}:

\begin{lemma}\label{Lem_Expressibility+dleq2ExpMinus1Forall(r,A)}
An axiom of the form $Z\dleq \forall r^{< \IIEXP{n}}.A$, where $Z,A\in\Nc$ and $n\geqslant 0$, is expressible by an acyclic $\R$-ontology network of size polynomial in $n$.
\end{lemma}

Now we are ready to prove the next statement, which is an analogue of Lemma \ref{Lem_Expressibility_EL+dleqExp} for the case of double exponent.

\begin{lemma}\label{Lem_Expressibility+dleq2ExpExists(r,A).B}
An axiom $\varphi$ of the form $Z\dleq \exists (r,A)^{\IIEXP{n}}.B$, where $Z,A\in\Nc$ and $n\geqslant 0$, is expressible by an acyclic $\R$-ontology network of size polynomial in $n$.
\end{lemma}
\begin{proof}
Consider ontology $\bar{\O}$ consisting of axioms 
\begin{equation*}
Z\dleq \exists s.\top, \ \ Z\dleq \forall s^{< \IIEXP{n}}.X, \ \ Z\dleq \forall s^{\IIEXP{n}}.Y
\end{equation*}
\begin{equation*}
s\dleq r
\end{equation*}
By Lemmas \ref{Lem_IteratedExpressibility}, \ref{Lem_Expressibility+dleq2ExpForall_r.A}, \ref{Lem_Expressibility+dleq2ExpMinus1Forall(r,A)}, $\bar{\O}$ is expressible by an acyclic $\R$-ontology network of size polynomial in $n$. Then by Lemma \ref{Lem_Expressibility_SimpleSubstitution}, so is ontology $\O=\bar{\O}[X\mapsto A\dcap \exists s.\top, \ Y\mapsto A\dcap B]$. 

Clearly, we have $\O\models\varphi$. Now let $\I$ be an arbitrary model of $\varphi$ and for $m=\IIEXP{n}$, let $x_0,\ldots , x_m$ be arbitrary domain elements such that $x_0\in Z^\I$, $\langle x_0,x_{1} \rangle\in r^\I$, and $\langle x_i,x_{i+1} \rangle\in r^\I$, $x_i\in A^\I $, for  $1\leqslant i < m$, and $x_m\in A^\I\dcap B^\I$. Let $\J$ be an expansion of $\I$ in which $s^\J=\{\langle x_i,x_{i+1}\rangle\}_{0\leqslant i < m}$. Then we have $\J\models\O$, from which the claim follows.
\end{proof}

%\begin{lemma}\label{Lem_Expressibility+equiv2ExpForall(r,A).B}
%An axiom $\varphi$ of the form $Z\dleq \forall (r,A)^{\IIEXP{n}}.B$, where $Z,A\in\Nc$ and $n\geqslant 0$, is expressible by an acyclic $\R$-ontology network of size polynomial in $n$.
%\end{lemma}
%\begin{proof}
%Note that $\{\varphi\}$ is equivalent to the ontology consisting of axioms 
%\begin{equation*}
%Z \dleq \forall (r,A)^{\IIEXP{n}-1}, \ Z\dleq \forall r^{\IIEXP{n}}.A, \  Z\dleq \forall r^{\IIEXP{n}}.B
%\end{equation*}
%Applying Lemmas \ref{Lem_IteratedExpressibility}, \ref{Lem_Expressibility+dleq2ExpForall_r.A}, %\ref{Lem_Expressibility+dleq2ExpMinus1Forall(r,A)} proves the claim.
%\end{proof}

\begin{lemma}\label{Lem_Expressibility_R+2Exp}
An axiom $\varphi$ of the form $Z \equiv \forall r^{\IIEXP{n}}.A$, where $Z,A\in\Nc$ and $n\geqslant 0$, is expressible by an acyclic $\R$-ontology network of size polynomial in $n$.
\end{lemma}
\begin{proof}
Consider ontology $\bar{\O}$ consisting of axioms 
\begin{equation*}
Z \dleq \forall r^{\IIEXP{n}}.A, \ \ \ \bar{Z} \dleq \exists r^{\IIEXP{n}}.\bar{A}
\end{equation*}
By Lemmas \ref{Lem_IteratedExpressibility}, \ref{Lem_Expressibility+dleq2ExpForall_r.A}, \ref{Lem_Expressibility+dleq2ExpExists(r,A).B}, $\bar{\O}$ is expressible by an acyclic $\R$-ontology network of size polynomial in $n$ and by Lemma \ref{Lem_Expressibility_SimpleSubstitution}, so is ontology $\O=\bar{\O}[\bar{Z}\mapsto \neg Z, \  \bar{A}\mapsto \neg A]$. It remains to note that $\O$ and $\{\varphi\}$ are equivalent, so the claim is proved.
\end{proof}

\begin{lemma}\label{Lem_Expressibility_ForallExpSubstitution}
Let $\LL$ be a DL and $\O$ an ontology, which is expressible by a $\LL$-ontology network $\N$. Let $C_1,\ldots, C_m$ be $\LL$-concepts and $\{A_1,\ldots, A_m\}$ a set of concept names such that $A_i\in\sig(\O)$, for $i=1,\ldots ,m$ and $m\geqslant 1$. Then for $k=1,2$ and $n\geqslant 0$, ontology $\tilde{\O}=\O[A_1\mapsto \forall r^{\EXP{k}{n}}.C_1,\ldots , A_m\mapsto \forall r^{\EXP{k}{n}}.C_m]$ is expressible by a $\LL'$-ontology network, which is acyclic if so is $\N$ and has size polynomial in the size of $\N$, $n$, and $C_i$, for $i=1,\ldots ,m$, where:
\begin{itemize}
\item  $\LL'=\LL$ if $\LL$ contains $\ALC$ and $k=1$;
\item  $\LL'=\LL$ if $\LL$ contains $\R$ and $k=2$.
\end{itemize}
\end{lemma}
\begin{proof}
The proof uses Lemmas \ref{Lem_Expressibility_ALC+Exp}, \ref{Lem_Expressibility_R+2Exp} and is identical to the proof of Lemma \ref{Lem_Expressibility_SimpleSubstitution}.
\end{proof}

The next statement is shown similarly by using Lemma \ref{Lem_Expressibility_EL+Exp}:

\begin{lemma}\label{Lem_Expressibility_ExistExpSubstitution}
In the conditions of Lemma \ref{Lem_Expressibility_ForallExpSubstitution}, for $n\geqslant 0$, ontology $\tilde{\O}=\O[A_1\mapsto \exists r^{\EXP{1}{n}}.C_1,\ldots , A_m\mapsto \exists r^{\EXP{1}{n}}.C_m]$ is expressible by a $\LL'$-ontology network, which is acyclic if so is $\N$ and has size polynomial in the size of $\N$, $n$, and $C_i$, for $i=1,\ldots ,m$, where $\LL'=\LL$ if $\LL$ contains $\EL$.
\end{lemma}

\begin{lemma}\label{Lem_Expressibility_infty}
Let $\LL$ be a DL and $\O$ an ontology, which is $(\N,\O_\N)$-expressible, for a $\LL$-ontology network $\N$ and an ontology $\O_\N$. Let $\{A_1,\ldots, A_n\}$, $n\geqslant 1$, be concept names such that $A_i\in\sig(\O)$, for $i=1,\ldots ,n$, and let $\{C_1,\ldots, C_n\}$ be $\LL$-concepts, where every $C_i$ is of the form $\exists (r,D)^p.A_i$, for some role $r$, concept name $D$, and $p\geqslant 1$. Then ontology $\tilde{\O}=\bigcup_{m\geqslant 0}\O_m$, where $\O_0=\O$ and $\O_{m+1}=\O_m[A_1\mapsto C_1,\ldots ,A_n\mapsto C_n]$, for all $m\geqslant 0$, is expressible by a cyclic $\LL$-ontology network.
\end{lemma}
\begin{proof}
Let $\sigma=\{B_1,\ldots ,B_k\}=\bigcup_{i=1,\ldots ,n}(\sig(C_i)\cap\Nc)$ and $\sigma'=\{B'_1,\ldots ,B'_k\}$ be a set of fresh concept names, which is disjoint with $\sigma$ and $\sig(\O)$. Let $\{C'_1,\ldots ,C'_n\}$ be `copy' concepts obtained from $C_1,\ldots, C_n$ by replacing every $B_i$ with $B'_i$, for  $i=1,\ldots ,k$. Consider ontologies 
\begin{align*}
\tilde{\O}_{\N'} & =\{\ B_i\equiv B'_i \ \}_{i=1,\ldots ,k} \\
\O' & =\{\ A_i\equiv C'_i\ \}_{i=1,\ldots ,n}
\end{align*}
and an ontology network $\N'$ given by the union of $\N$ with the set of import relations 
\begin{align*}
\langle \tilde{\O}_{\N'},\sig(\O),\O_\N\rangle, \ \langle \tilde{\O}_{\N'},\Sigma',\O'\rangle, \ \langle \O',\Sigma,\tilde{\O}_{\N'}\rangle
\end{align*}
where $\Sigma=\sig(\O)\cup\bigcup_{i=1,\ldots ,n}\sig(C_i)$ and $\Sigma'=(\Sigma\setminus\sigma)\cup\sigma'$. We claim that ontology $\tilde{\O}$ is $(\N', \tilde{\O}_{\N'})$-expressible. Denote $\tilde{\O'}=\bigcup_{m\geqslant 0}\O'_m$, where $\O'_m=\O_m[B_1\mapsto B'_1,\ldots ,B_k\mapsto B'_k]$, for all $m\geqslant 0$. \medskip

First, we show by induction that $\tilde{\O}_{\N'}\models_{\N'}\O_m$, for all $m\geqslant 0$. The induction base for $n=0$ is trivial, since we have $\O_0=\O$ and $\O$ is $(\N,\O_\N)$-expressible, $\langle \tilde{\O}_{\N'},\sig(\O),\O_\N\rangle\in\N'$, and thus, $\tilde{\O}_{\N'}\models_{\N'}\O$. Suppose $\tilde{\O}_{\N'}\models_{\N'} \O_m$, for some $m\geqslant 0$. Since $\langle \O',\Sigma,\tilde{\O}_{\N'}\rangle\in\N'$, we have $\O'\models_{\N'}\O_m$ and thus, by the equivalences in $\O'$, it holds $\O'\models_{\N'} \O'_{m+1}$. Since $\langle \tilde{\O}_{\N'},\Sigma',\O'\rangle$, we have $\tilde{\O}_{\N'}\models_{\N'}\O'_{m+1}$ and hence, by the equivalences in $\tilde{\O}_{\N'}$, it holds that $\tilde{\O}_{\N'}\models_{\N'}\O_{m+1}$. \medskip

Now let $\I$ be an arbitrary model of ontology $\tilde{\O}$ and $\I_1$ be an expansion of $\I$, in which every $B'_i$ is interpreted as $(B_i)^\I$, for $i=1,\ldots ,k$. Clearly, it holds $\I_1\models\tilde{\O}_{\N'}$ and thus, we have $\I_1\models\tilde{\O}_{\N'}\cup\tilde{\O}$. We show that $\I_1\models_{\N'}\tilde{\O}_{\N'}$, i.e. there exists a model agreement $\mu'$ for $\N'$ such that $\I_1\in\mu'(\tilde{\O}_{\N'})$. We define families of interpretations $\{\I_m\}_{m\geqslant 1}$ and $\{\I'_m\}_{m\geqslant 1}$ such that for all $m\geqslant 1$, $\I_m\models\tilde{\O}_{\N'}$ and $\I'_m\models\O'$, and it holds $\I_m=_{\Sigma'}\I'_m$ and $\I'_m=_\Sigma\I_{m+1}$. The families of interpretations are defined by induction on $m$ by showing that for any interpretation $\I_m$ such that $\I_m\models\tilde{\O}_{\N'}\cup\tilde{\O}$ there exist the corresponding interpretations $\I'_m$ and $\I_{m+1}$ such that both of them are also models of $\tilde{\O}$. 

Given $\I_m$ as above, for $m\geqslant 1$, let $\I'_m$ be an interpretation, which agrees with $\I_m$ on $\Sigma'$ and in which every $A_i$ is interpreted as $(C'_i)^{\I_m}$, for $i=1, \ldots ,n$. Then $\I'_m\models\O'$ and since $\I_m\models B_i\equiv B'_i$, for $i=1,\ldots ,k$, we have  $\I'_m\models\tilde{\O'}$. Then $\I'_m$ is a model of every concept inclusion obtained from an axiom of $\tilde{\O'}$ by substituting every occurrence of $C'_i$ with $A_i$, for $i=1,\ldots ,n$, and therefore, from the definition of $\tilde{\O}$, we conclude that $\I'_m\models\tilde{\O}$. Now let $\I_{m+1}$ be an interpretation, which agrees with $\I'_m$ on $\Sigma$ and in which every $B'_i$ is interpreted as $B_i$, for $i=1,\ldots ,k$. Then $\I_{m+1}$ is a model of $\tilde{\O}_{\N'}$ and $\tilde{\O}$. 

Since $\tilde{\O}\models\O$, we have $\I_m\models\O$, for all $m\geqslant 1$. As ontology $\O$ is $(\N,\O_\N)$-expressible, there exists a model agreement $\mu$ for the network $\N\cup\{\langle \tilde{\O}_{\N'},\sig(\O),\O_\N\rangle\}$ such that $\mu(\tilde{\O}_{\N'})=\{\I_m\}_{m\geqslant 1}$. Let us define a mapping $\mu'$ such that $\mu'(\O')=\{\I'_m\}_{m\geqslant 1}$ and the values of $\mu'$ and $\mu$ coincide on all other ontologies in $\N'$. Then $\mu'$ is the required model agreement for $\N'$.
\end{proof}

\section{Hardness Results}\label{Sect_Hardness}

%----------------------------------------------------------------------%

We use reductions from the word problem for Turing machines (TMs) and alternating Turing machines (ATMs) to obtain most of the results in this section. We use the following conventions and notations related to these computation models. A Turing Machine (TM) is a tuple $M=\tuple{Q,\mathcal{A},\delta}$, where $Q$ is a set of states, with $\qacc\in Q$ being the accepting state, $\mathcal{A}$ is an alphabet, and $\delta: Q\times\mathcal{A} \mapsto  Q\times\mathcal{A}\times\{-1,1\}$ is a transition function. We assume w.l.o.g. that configuration of $M$ is a word in the alphabet $Q\cup\mathcal{A}$ which contains exactly one state symbol $q\in Q$. An initial configuration is a word of the form $\b\ldots\b \qo\b\ldots \b$, where $\qo\in Q$ and $\b\in\mathcal{A}$ is the blank symbol. For a configuration $\c$, the notion of successor configuration is defined by $\delta$ in a usual way and is denoted as $\delta(\c)$. A configuration $\c$ is said to be accepting if there is a sequence of configurations $\c_0,\ldots , \c_k$, $k\geqslant 0$, where $\c_0=\c$, $\c_k=v\qacc w$, and for all $0 \leqslant i < k$, $\c_{i+1}$ is a successor of $\c_i$. It is a well-known property of the transition functions of Turing machines that the symbol $c'_{i}$ at position $i$ of a configuration $\delta(\c)$ is uniquely determined by a 4-tuple of symbols $c_{i-2}, c_{i-1}, c_{i}, c_{i+1}$ at positions $i-2$, $i-1$, $i$, and $i+1$ of a configuration $\c$. We assume that this correspondence is given by the (partial) function $\delta'$ %$\delta'(c_{i-2}, c_{i-1}, c_{i}, c_{i+1})=c'_{i}$ 
and use the notation $c_{i-2} c_{i-1} c_{i} c_{i+1} \overset{\delta'}{\mapsto} c'_{i}$.

An Alternating Turing Machine (ATM) is a tuple $M=\langle Q,\mathcal{A},\delta_1,\delta_2\rangle$, where $Q=Q_\forall \cup Q_\exists \cup \{\qrej\}$ is a set of states, with  $\qrej\in Q$ being the rejecting state, $\mathcal{A}$ is an alphabet containing the blank symbol $\b\in\mathcal{A}$, and for $\alpha=1,2$, $\delta_\alpha: Q\times\mathcal{A}\times Q\times\mathcal{A}\times\{-1,1\}$ is a transition function. We assume w.l.o.g. that configuration of ATM $M$ is a word in the alphabet $Q\cup\mathcal{A}$ which contains at most one symbol $q\in Q$. For a configuration $\c$, the notion of successor configuration (wrt $\delta_\alpha$, $\alpha=1,2$) is defined in a usual way. A configuration $\c=vqw$ is (inductively) defined as \textit{rejecting} if either $q=\qrej$, or $q\in Q_\forall$ and there is a successor configuration of $\c$ which is rejecting, or $q\in Q_\exists$ and any successor configuration of $\c$ is rejecting. A \textit{rejecting run tree} of an ATM $M$ for an initial configuration $w$ is a tree in which the nodes are rejecting configurations of $M$, $w$ is the root node, every child node is a successor configuration of its parent node, every leaf is a configuration with the state $\qrej$, and if there is a node $\c=uqw$, then the following holds: if $q\in Q_\forall$, then $\c$ has a at least one child, and if $q\in Q_\exists$ then $\c$ has two children. For $k\geqslant 0$, a configuration $\c$ of $M$ is said to be \textit{$k$-rejecting} if $M$ has a rejecting run tree of height $k$ with the root $\c$. The notions of \emph{accepting run tree} (with every node not being a rejecting configuration) and \emph{accepting configuration} are defined dually.  

Similarly to ordinary TMs, we assume that the correspondence between a configuration $\c$ and the successor configuration $\c_\alpha$ of $\c$ (wrt $\delta_\alpha$, $\alpha=1,2$) determined by 4-tuples of symbols is given by functions $\delta'_\alpha$, for $\alpha=1,2$. If every configuration of an ATM is a finite word of length $n$, then we assume w.l.o.g. that for $\alpha=1,2$ this correspondence is given as follows (for a word $w$ of length $n$ and $1\leqslant i \leqslant n$, we denote by $w[i]$ the $i$-th symbol in $w$):

\begin{equation*}
\c[i-2] \c[i-1] \c[i] \c[i+1] \overset{\delta'_\alpha}{\mapsto} \c'[i], \ \ \ \text{for} \ \  1\leqslant i \leqslant n-3
\end{equation*}
\begin{equation*}
\b \b \c[1] \c[2] \overset{\delta'_\alpha}{\mapsto} \c'[1], \ \ \ \text{for} \ \ \c[1]\not\in Q
\end{equation*}
\begin{equation*}
\b \c[1] \c[2] \c[3] \overset{\delta'_\alpha}{\mapsto} \c'[2]
\end{equation*}
\begin{equation*}
\c[n-2] \c[n-1] \c[n] \b \overset{\delta'_\alpha}{\mapsto} \c'[n]
\end{equation*}

\medskip

% BEGIN HARDNESS PROOFS FOR EL-NETWORKS

\begin{theorem}\label{Teo_Hardness_EL-acyclic}
Entailment in acyclic $\EL$-ontology networks is $\EXPTIME$-hard.
\end{theorem}

\Proofsketch We reduce the word problem for TMs making exponentially many steps to entailment in $\EL$-ontology networks. Let $M=\langle Q,{\mathcal{A}},\delta \rangle$ be a TM and $n=\IEXP{m}$ an exponential, for some $m\geqslant 0$. Consider an ontology $\O$ defined for $M$ and $n$ by the axioms (\ref{Eq_ELExptimeRchain})-(\ref{Eq_ELExptime_PropagateH}) below. 

The first axiom gives a $r$-chain containing $n+1$ segments of length $2n+3$, which are used to store fragments of consequent configurations of $M$:

\begin{equation}\label{Eq_ELExptimeRchain}
A \dleq \ex{r^{n\cdot (2n+3)}}.(\qo\dcap\ex{(r,\b)^{2n+2}} )
\end{equation}
($A\not\in Q\cup\mathcal{A}$).

We assume the following enumeration of segments in the $r$-chain: \vspace{-0.2cm}

\begin{equation*}
\word{\ldots}{s_{n}}{\ldots}{\ldots}{s_1}\underbrace{\qo\b\ldots\b}_{s_0}
\end{equation*}

i.e., $s_0$ represents a fragment of the initial configuration $\c_0$ of $M$. For $0\leqslant i < n$, every $i$-th and $(i+1)$-st segments in the $r$-chain are reserved for a pair of configurations $\c_i,\c_{i+1}$ such that $\c_{i+1}$ is a successor of $\c_{i}$. 

The next family of axioms represents transitions of $M$ and defines the `content' of $(i+1)$-st segment based on the `content' of $i$-th segment:

\begin{equation}\label{Eq_ELExptimeTransition}
\ex{r^{2n}}(X \dcap \ex{r}.(Y \dcap \ex{r}.(U \dcap \ex{r}.Z))) \dleq W,
\end{equation}
for all $X,Y,U,Z,W\in Q\cup{\mathcal{A}}$ such that $XYUZ\overset{\delta'}{\mapsto} W$. 
%(thus, position $j$ in $(i+1)$-st segment is `labelled' with $W$ iff the labels of $(j-3), (j-2), (j-1), j$-positions in $i$-th segment are $X,Y,U,Z$, respectively, and yield function $\delta'$ applies).

Finally, the following axioms are used to initialise the halting marker $H$ and propagate it to the `left' of the $r$-chain:
\begin{equation}\label{Eq_ELExptime_PropagateH}
\ex{r}.\qacc \dleq H, \ \ \ex{r}.H \dleq H 
\end{equation}

%\begin{equation}\label{Eq_ELUndecid_DefectMarker}
%X \dcap Y \dleq	 D, \ \ \ex{r}.D \dleq D
%\end{equation}
%for all concepts $X,Y\in \mathcal{A}$, with $X\neq Y$ (to initialize the defect marker and propagate it to the ``left'');

%\begin{equation}\label{Eq_ELUndecid_GoodMarker}
%A \dcap D \dleq	 G, \ \ A \dcap H \dleq G
%\end{equation}
%(the marker $G$ is initialised, whenever an element from interpretation of $A$ belongs to interpretation of either $D$ or $H$).

\medskip

The definition of ontology $\O$ is complete. We claim that $M$ accepts the empty word in $n$ steps iff $\O\models A\dleq H$. 

For the `only if' direction we assume there is a sequence of configurations $\c_0,\ldots ,\c_n$ such that for all $0\leqslant i < n$, $\c_{i+1}$ is a successor configuration of $\c_i$ and $\qacc$ is the state symbol in $\c_n$.  Let $\I$ be a model of $\O$ and $a$ domain element such that $a\in A^\I$. Then by axiom (\ref{Eq_ELExptimeRchain}), there is an $r$-chain outgoing from $x$, which contains segments $s_0,\ldots ,s_n$ of length $2n+3$, where $s_0$ represents a fragment of $\c_{0}$. It can be shown by induction that due to axioms (\ref{Eq_ELExptimeTransition}), every segment $s_i$ represents a fragment of $\c_i$, for $1\leqslant i \leqslant n$, and contains the state symbol from $\c_i$. Then by axiom (\ref{Eq_ELExptime_PropagateH}), it follows that $a\in H^\I$. %This is proved by induction and by showing that for all $1\leqslant i \leqslant n$, the last two positions in a segment $s_i$ corresponding to configuration $\c_i$, are ``labelled'' with $\b$. This guarantees that the rule (\ref{Eq_ELExptimeTransition}) correctly defines the labels for the first three positions of segment $s_{i+1}$ depending on the labels of the first three positions of $s_{i}$ and the last three positions of $s_{i+1}$. 

For the `if' direction, one can provide a model $\I$ of $\O$ such that $A^\I=\{a\}$ is a singleton, $\qacc^\I=H^\I=\varnothing$, and $\I$ gives an $r$-chain outgoing from $a$, which contains $n+1$ segments representing fragments of consequent configurations of $M$, neither of which contains $\qacc$. 

\smallskip

To complete the proof of the theorem we show that ontology $\O$ is expressible by an acyclic $\EL$-ontology network of size polynomial in $m$. Note that $\O$ contains axioms (\ref{Eq_ELExptimeRchain}), (\ref{Eq_ELExptimeTransition}) with concepts of size exponential in $m$. Consider axiom (\ref{Eq_ELExptimeRchain}) and a concept inclusion $\varphi$ of the form

\begin{equation*}
A\dleq \ex{r^{n\cdot(2n+3)}}.B
\end{equation*}
where $B$ is a concept name. Observe that it is equivalent to

\begin{equation*}
A\dleq \underbrace{\ex{r^p}.\ex{r^p}}_{2~\text{times}}.\underbrace{\ex{r^n}.\ex{r^n}.\ex{r^n}}_{3~\text{times}}.B
\end{equation*}
where $p=\IEXP{2m}$. Consider axiom $\psi$ of the form $A\dleq B$. By iteratively applying Lemma \ref{Lem_Expressibility_ExistExpSubstitution} we obtain that $\psi[B\mapsto \ex{r^p}.\ex{r^p}.B]$ is expressible by an acyclic $\EL$-ontology network of size polynomial in $m$. By repeating this argument we obtain the same for $\varphi$. Further, by Lemma \ref{Lem_Expressibility_SimpleSubstitution}, the axiom $\theta=\varphi[B\mapsto \qo \dcap B]$ is   expressible by an acyclic $\EL$-ontology network of size polynomial in $m$. Again, by iteratively applying Lemma \ref{Lem_Expressibility_ExistExpSubstitution} together with Lemma \ref{Lem_Expressibility_SimpleSubstitution} we conclude that $\theta[B\mapsto \exists (r,\b)^{2n+2}]$ is expressible by an acyclic $\EL$-ontology network of size polynomial in $m$ and thus, so is axiom (\ref{Eq_ELExptimeRchain}). The expressibility of axioms  (\ref{Eq_ELExptimeTransition}) is shown identically. The remaining axioms of ontology $\O$ are $\EL$-axioms whose size does not depend on $m$. By applying Lemma \ref{Lem_IteratedExpressibility} we obtain that there exists an acyclic $\EL$-ontology network  $\N$ of size polynomial in $m$ and an ontology $\O_\N$ such that $\O$ is $(\N,\O_\N)$-expressible and thus, it holds $\O_\N\models_\N A\dleq H$ iff $M$ accepts the empty word in $\IEXP{m}$ steps. \QED

\commentout{ %%%%%% OLD PROOF %%%%%%%
} %% commentout %%%%%% OLD PROOF %%%%%%%

\begin{theorem}\label{Teo_Hardness_EL-cyclic}
Entailment in cyclic $\EL$-ontology networks is $\RE$-hard. 
\end{theorem}

\Proofsketch For a TM $M=\langle Q,{\mathcal{A}},\delta \rangle$, we define an infinite ontology $\O$, which contains variants of axioms (\ref{Eq_ELExptimeRchain}-\ref{Eq_ELExptimeTransition}) from Theorem \ref{Teo_Hardness_EL-acyclic} and additional axioms for a correct implementation of transitions of $M$. %The axioms are built using expressions of the form $\exists (r,C)^{ax+b}.D$, for $a\geqslant 1$ and $b\geqslant 0$, which is a shortcut for the parametrized concept $\exists (r,C)^{ax}.\exists(r,C)^b.D$. 

Axioms (\ref{Eq_ELUndecStartRchain}) give an infinite family of $r$-chains, each having a `prefix' of length $k+1$, for $k\geqslant 0$ (reserved for fragments of consequent configurations of $M$), and a `postfix' containing a chain of length $2l+3$, for $l\geqslant 0$, which represents a fragment of the initial configuration $\c_{0}$: \vspace{-0.2cm}

\begin{equation}\label{Eq_ELUndecStartRchain}
A\dleq\exists r^k.(\exists v^l.L\dcap\varepsilon\dcap\ex{r}.(\qo\dcap\exists (r,\b)^{2l+2}))
\end{equation}

%\begin{equation}\label{Eq_ELUndecRchain}
%\forall y  \ \ B \dleq \exists v^y.L \dcap \varepsilon \dcap \ex{r}.(\qo\dcap\exists (r,\b)^{2y+2}) 
%\end{equation}
for $A,L,\varepsilon\not\in Q\cup\mathcal{A}$ and all $k,l\geqslant 0$. \smallskip

Propagated to the `left' by the next family of axioms, concept $\exists v^k.L$, $k\geqslant 0$, indicates the length of the `postfix' for $\c_{0}$ on every $r$-chain given by axioms (\ref{Eq_ELUndecStartRchain}):

\begin{equation}\label{Eq_ELUndecPropagateLengthMarker}
\ex{r}.\exists v^k.L \dleq \exists v^k.L, \ \ k\geqslant 0
\end{equation}

The concept $\varepsilon$ is used to separate fragments of consequent configurations of $M$ and is therefore propagated as follows:

\begin{equation}\label{Eq_ELUndecPropagateFragmentMarker}
\exists v^k.L \dcap \exists r^{2k+4}.\varepsilon \dleq \varepsilon, \ \ k\geqslant 0
\end{equation}

The next families of axioms, with $X,Y,U,Z,W\in Q\cup{\mathcal{A}}$, implement transitions of $M$.\vspace{-0.2cm}

\begin{align}\label{Eq_ELUndecTransition}
\exists v^k. & L \ \dcap \\ & \dcap \exists r^{2k+1}. (X \dcap \ex{r}.(Y \dcap \ex{r}.(U \dcap \ex{r}.Z))) \dleq W \nonumber
\end{align} %\nonumber
for $XYUZ\overset{\delta'}{\mapsto} W$ and all $k\geqslant 0$. 
%for all $X,Y,U,Z,W\in Q\cup{\mathcal{A}}$ such that $XYUZ\overset{\delta'}{\mapsto} W$.  
%%(thus, position $j$ in segment $i+1$ is ``labelled'' with $W$ iff the labels of $(j-2), (j-1), j, (j+1)$-positions in segment $i$ are $X,Y,U,Z$, respectively, and yield function $\delta'$ applies).

Concept $\exists v^k.L$ guarantees that transitions have effect only along $r$-chains which represent a fragment of $\c_{0}$ of length $2k+3$. Since $\varepsilon\not\in Q\cup\mathcal{A}$, the transitions involving $\varepsilon$ are implemented separately by  the following families of axioms:  \vspace{-0.2cm}

\begin{align}\label{Eq_ELUndecEpsilonTransition1} 
\exists v^k. & L \ \dcap  \\ &  \dcap \exists r^{2k}. (\b \dcap \ex{r}.(\varepsilon \dcap \ex{r}.(Y \dcap \ex{r}.(U \dcap \ex{r}.Z)))) \dleq W \nonumber
\end{align} %\nonumber
for $\b YUZ\overset{\delta'}{\mapsto} W$ and all $k\geqslant 0$; \vspace{-0.2cm}

\begin{align}\label{Eq_ELUndecEpsilonTransition2}
\exists v^k. & L \ \dcap \\ & \dcap \exists r^{2k}. (\b \dcap \ex{r}.(\b \dcap \ex{r}.(\varepsilon \dcap \ex{r}.(U \dcap \ex{r}.Z)))) \dleq W \nonumber
\end{align} %\nonumber
for $\b\b UZ\overset{\delta'}{\mapsto} W$ and all $k\geqslant 0$; \vspace{-0.2cm}

\begin{align}\label{Eq_ELUndecEpsilonTransition3}
\exists v^k. & L \ \dcap \\ & \dcap \exists r^{2k}. (X \dcap \ex{r}.(\b \dcap \ex{r}.(\b \dcap \ex{r}.(\varepsilon \dcap \ex{r}.Z)))) \dleq W \nonumber
\end{align} %\nonumber
for $X\b\b Z\overset{\delta'}{\mapsto} W$ and all $k\geqslant 0$. 

\medskip

The last axiom of $\O$ is used to initialize the halting marker $H$ and propagate it to the `left' of a $r$-chain: \vspace{-0.2cm}

\begin{equation}\label{Eq_ELUndecid_PropagateH}
\qacc \dleq H, \ \ \ex{r}.H \dleq H 
\end{equation}

\noindent The definition of ontology $\O$ is complete. \medskip

The more involved implementation of transitions (in comparison to Theorem \ref{Teo_Hardness_EL-acyclic}) allows to prevent defect situations, when there are two consequent segments $s_i$,$s_{i+1}$ of an $r$-chain, which represent fragments of configurations $\c_i$,$\c_{i+1}$ of $M$, respectively, but $\c_{i+1}$ is not a successor of $\c_{i}$. In Theorem \ref{Teo_Hardness_EL-acyclic}, the prefix of length $n\cdot (2n+3)$ given by axiom (\ref{Eq_ELExptimeRchain}) guarantees a correct implementation of up to $n$ transitions of the TM. The situation is different in the infinite case, since the prefix reserved for fragments of consequent configurations of $M$ can be of any length, due to axioms (\ref{Eq_ELUndecStartRchain}). \medskip

We prove that $M$ halts iff $\O\models A\dleq H$. Suppose that $\c_{0}$ is an accepting configuration and $M$ halts in $n$ steps; w.l.o.g. we assume that $n>1$. Let $\I$ be a model of $\O$ and $a\in A^\I$ a domain element. Due to axioms (\ref{Eq_ELUndecStartRchain}), $\I$ is a model of the concept inclusion: \vspace{-0.2cm}

\begin{align*}\label{Eq_ELUndecConcreteRchain}
A \dleq \ex{r^{n\cdot (2n+4)}}.(\ex{v^n}.L \dcap \varepsilon \dcap \ex{r}.(\qo\dcap\ex{(r,\b)^{2n+2}}) )
\end{align*}

and thus, $\I$ gives a $r$-chain containing $n+1$ segments of length $2n+3$ separated by $\varepsilon$. By using arguments from the proof of Theorem \ref{Teo_Hardness_EL-acyclic}, it can be shown that due to axioms (\ref{Eq_ELUndecPropagateLengthMarker}) - (\ref{Eq_ELUndecEpsilonTransition3}), these segments represent fragments of consequent configurations of $M$, starting with $\c_{0}$, and there is an element $b$ in the $r$-chain such that $b\in \qacc^\I$. Then by axiom (\ref{Eq_ELUndecid_PropagateH}), it holds $a\in H^\I$. 

For the `if' direction, one can show that if $M$ does not halt, then there exists a model $\I$ of $\O$ such that $\qacc^\I=H^\I=\varnothing$, $A^\I=\{a\}$ is a singleton and there are infinitely many disjoint $r$-chains $\{R_{m,n}\}_{m,n\geqslant 1}$ outgoing from $a$, such that every $R_{m,n}$ represents a fragment of $\c_{0}$ of length $n$ and has a prefix of length $m+1$ representing fragments of consequent configurations of $M$, each having length $\leqslant 2n+3$. 

\smallskip

To complete the proof of the theorem we show that ontology $\O$ is expressible by a cyclic $\EL$-ontology network. Let us demonstrate that so is the family of axioms (\ref{Eq_ELUndecStartRchain}). Let $\varphi=A\dleq B$ be a concept inclusion and $B, B1, B_2$ concept names. By Lemma \ref{Lem_Expressibility_infty}, ontology $\O_1=\{\varphi[B\mapsto \exists r^k.B] \mid k\geqslant 0\}$ is expressible by a cyclic $\EL$-ontology network. Then by Lemma \ref{Lem_Expressibility_SimpleSubstitution}, ontology $\O_2=\O_1[B\mapsto B_1 \dcap \varepsilon \dcap \ex{r}.(\qo \dcap B_2)]$ is expressible by a cyclic $\EL$-ontology network. By applying Lemma \ref{Lem_Expressibility_infty} again, we conclude that so is ontology $\O_3=\bigcup_{l\geqslant 0}\O_2[B_1\mapsto \exists v^l.B_1, \ B_2\mapsto \exists (r,\b)^{2l}.B_2]$, i.e., the ontology given by axioms
\begin{equation*}
A\dleq\exists r^k.(\exists v^l.B_1\dcap\varepsilon\dcap\ex{r}.(\qo\dcap\exists (r,\b)^{2l}.B_2))
\end{equation*}
for $k,l\geqslant 0$. Further, by Lemma \ref{Lem_Expressibility_SimpleSubstitution}, we obtain that $\O_2[B_1\mapsto L, \ B_2\mapsto \exists (r,\b)^2]$ is expressible by a cyclic $\EL$-ontology network and hence, so is the family of axioms (\ref{Eq_ELUndecStartRchain}). A similar argument shows the expressibility of ontologies given by axioms (\ref{Eq_ELUndecPropagateLengthMarker})-(\ref{Eq_ELUndecEpsilonTransition3}). The remaining subset of axioms (\ref{Eq_ELUndecid_PropagateH}) of $\O$ is finite. By Lemma \ref{Lem_IteratedExpressibility}, there exists a cyclic $\EL$-ontology network $\N$ and an ontology $\O_\N$ such that $\O$ is $(\N,\O_\N)$-expressible and thus, it holds $\O_\N\models_\N A\dleq H$ iff $M$ halts. \QED

\commentout{ %%%%%%% OLD PROOF %%%%%%%
}%\commentout{ %%%%%%% OLD PROOF %%%%%%%

%END OF HARDNESS PROOFS FOR EL-NETWORKS
%--------------------------------------------%

\begin{theorem}\label{Teo_Hardness_ALC}
Entailment in $\ALC$-ontology networks is $\DEXPTIME$-hard.
\end{theorem}

\Proofsketch The result is based on the construction from the proof of Theorem 7 in \cite{Kazakov:08:RIQ:SROIQ}, where it is shown that the word problem for $\IEXP{n}$-space bounded ATMs, for $n\geqslant 0$, reduces to satisfiability of $\mathcal{R}$-ontologies. We demonstrate that under a minor modification the construction used in that theorem shows that there is a $\ALC$-ontology $\O$ containing nested concepts of exponential size and a concept name $A$ such that $\O\not\models A\dleq\bot$ iff a given $\IEXP{n}$-space bounded ATM accepts the empty word. The ontology contains axioms with concepts of the form $\ex (r,C)^{\IEXP{n}}.D$ and $\all r^{\IEXP{n}}.D$. Using axioms of the form $Z\dleq \exists (r,C)^{\IEXP{n}}.D$ it is possible to encode consequent exponentially long $r$-chains for storing consequent configurations of ATM. With axioms of the form $Z\dleq \forall  r^{\IEXP{n}}.D$ it is possible to encode transitions between configurations by defining correspondence of interpretations of concept names (encoding the alphabet of ATM) on two consequent $\IEXP{n}$-long $r$-chains. The rest of the concept inclusions in $\O$ are $\ALC$-axioms, whose size does not depend on $n$ and which are used to represent the initial configuration, existential/universal types configurations, and describe additional conditions for implementation of transitions. 
By using Lemmas \ref{Lem_Expressibility_EL+dleqExp}, \ref{Lem_Expressibility_ForallExpSubstitution}, we demonstrate that every axiom of $\O$ containing concepts of size exponential in $n$ is expressible by an acyclic $\ALC$-ontology network $\N$ of size polynomial in $n$. Then by applying Lemma \ref{Lem_IteratedExpressibility} we obtain that there exists an acyclic $\ALC$-ontology network $\N$ of size polynomial in $n$ and an ontology $\O_\N$ such that $\O$ is $(\N,\O_\N)$-expressible and thus, it holds $\O_\N\models_\N A\dleq \bot$ iff $M$ accepts the empty word. Since $\AEXPSPACE=\DEXPTIME$, we obtain the required statement. \QED

\medskip

\begin{theorem}\label{Teo_Hardness_R}
Entailment in $\R$-ontology networks is $\TEXPTIME$-hard.
\end{theorem}

\Proofsketch
The proof is by reduction of the word problem for $\IIEXP{n}$-space bounded  ATMs to entailment in $\R$-ontology networks. Given such TM $M$ and a number $n\geqslant 0$, we consider ontology $\O$ from the proof of Theorem \ref{Teo_Hardness_ALC} for $M$ and let $\O'$ be the ontology obtained from $\O$ by replacing every nested concept of the form $\ex(r,C)^{\IEXP{n}}.D$ and $\all r^{\IEXP{n}}.D$ with $\ex(r,C)^{\IIEXP{n}}.D$ and $\all r^{\IIEXP{n}}.D$, respectively. Then a repetition of the proof of Theorem \ref{Teo_Hardness_ALC} gives that $\O'\models A\dleq\bot$ iff $M$ accepts the empty word. By applying Lemmas \ref{Lem_Expressibility+dleq2ExpExists(r,A).B}, \ref{Lem_Expressibility_ForallExpSubstitution} we show that every axiom of $\O'$ containing  concepts of size double exponential in $n$ is expressible by an acyclic $\R$-ontology network of size polynomial in $n$. The remaining axioms of $\O'$ are $\ALC$ axioms, whose size does not depend on $n$. Then by applying Lemma \ref{Lem_IteratedExpressibility} we obtain that there exists an acyclic $\R$-ontology network $\N$ of size polynomial in $n$ and an ontology $\O_\N$ such that $\O'$ is $(\N,\O_\N)$-expressible and thus, it holds $\O_\N\models_\N A\dleq \bot$ iff $M$ accepts the empty word. Since $\ADEXPSPACE=\TEXPTIME$, we obtain the required statement. \QED

\begin{theorem}\label{Teo_Hardness_ALCHOIF}
Entailment in $\ALCHOIF$-ontology networks is $\coNDEXPTIME$-hard.
\end{theorem}

\Proofsketch The result is based on the construction from the proof of Theorem 5 in \cite{Kazakov:08:RIQ:SROIQ}, where it is shown that the $\NDEXPTIME$-hard problem of existence of a domino tiling of size $\IIEXP{n} \times \IIEXP{n}$, $n\geqslant 0$, reduces to satisfiability of $\ROIF$-ontologies. We demonstrate that under a minor modification the construction used in that theorem shows that there exists a $\ALCHOIF$-ontology $\O$ containing concepts of an exponential size and a concept name $A$ such that $\O\not\models A\dleq \bot$ iff a given domino system admits a tiling of size $\IIEXP{n} \times \IIEXP{n}$, for $n\geqslant 0$. Ontology $\O$ contains axioms with concepts of the form $\ex r^{\IEXP{n}}.C$ and $\all r^{\IEXP{n}}.C$. Axioms of the form $Z\dleq \exists r^{\IEXP{n}}.C$ allow one to encode $\IEXP{n}$-long $r$-chains. Using a variant of the binary counter technique together with axioms of the form $Z\dleq \forall  r^{\IEXP{n}}.C$ and role hierarchies allows one to encode $\IIEXP{n}$-many consequent $r$-chains of this kind, thus obtaining sequences of $\IIEXP{n}$-many end points of $r$-chains. With nominals and inverse functional roles it is possible to enforce coupling of these sequences to obtain a grid of size $\IIEXP{n} \times \IIEXP{n}$. Finally, $\ALC$ axioms with concepts of the form $\forall r^{\IEXP{n}}.C$ allow one to represent the initial and matching conditions of the domino tiling problem. 
By using the same arguments as in the proof of Theorem \ref{Teo_Hardness_ALC} we show that there is a $\ALCHOIF$-ontology network $\N$ of size polynomial in $n$ and an ontology $\O_\N$ such that $\O_\N\models_\N A\dleq \bot$ iff the domino system does not admit a tiling of size $\IIEXP{n} \times \IIEXP{n}$. \QED

\begin{theorem}\label{Teo_Hardness_ROIF}
Entailment in $\ROIF$-ontology networks is $\coNTEXPTIME$-hard.
\end{theorem}

\Proofsketch The theorem is proved by a reducing the $\NTEXPTIME$-hard problem of domino tiling of size $\IIIEXP{n} \times \IIIEXP{n}$ to entailment in $\ROIF$-ontology networks. Given an instance of this problem, we consider ontology $\O$ defined in the proof of Theorem \ref{Teo_Hardness_ALCHOIF} and let $\O'$ be the ontology obtained from $\O$ by replacing every nested concept $\ex r^{\IEXP{n}}.C$ and $\all r^{\IEXP{n}}.C$ with $\ex r^{\IIEXP{n}}.C$ and $\all r^{\IIEXP{n}}.C$, respectively. A repetition of the proof of Theorem \ref{Teo_Hardness_ALCHOIF} shows that $\O\not\models A\dleq \bot$ iff a given domino system admits a tiling of size $\IIIEXP{n} \times \IIIEXP{n}$. By applying Lemmas \ref{Lem_IteratedExpressibility}, \ref{Lem_Expressibility+dleq2ExpExists(r,A).B}, \ref{Lem_Expressibility_ForallExpSubstitution} we obtain that there exists a $\ROIF$-ontology network $\N$ of a polynomial size and an ontology $\O_\N$ such that $\O_\N\models_\N A\dleq \bot$ iff the given domino system does not admit a tiling of size $\IIIEXP{n} \times \IIIEXP{n}$. \QED

%--------------------------------------------%
%Next, we show that role-free ontology networks can simulate Alternating Turing Machines.

\begin{theorem}\label{Teo_Hardness_HGeneral} 
Entailment in $\mathcal{H}$-Networks is $\EXPTIME$-hard.
\end{theorem}

\Proofsketch We show that the word problem for ATMs working with words of a polynomial length $n$ reduces to entailment in cyclic $\mathcal{H}$-ontology networks. Then, since $\APSPACE=\EXPTIME$, the claim follows.

Let $M=\langle Q,\mathcal{A},\delta_1, \delta_2 \rangle$ be an ATM. %For the purpose of the proof, let \emph{configuration} of $M$ be a word of length $m=2n+4$ in the alphabet $Q\cup \mathcal{A}$, where $n$ is a polynomial.  Then, given a configuration $\c$, the notion of successor configuration is naturally induced by $\delta_1$ and $\delta_2$. 
We call the word of the form $\b\qo\b\ldots\b$
%\begin{equation*}
%\word{\b\ldots \b}{n+2 \ \text{times}}{\qo}{\b\ldots\b}{n+1 \ \text{times}}
%\end{equation*}
\emph{initial configuration} of $M$. %and for a configuration $\c$ we denote by $\c[j]$ the $j$-the symbol in $\c$. 
Consider a signature $\sigma$ consisting of concept names $B_{ai}$, for $a\in Q\cup{\mathcal{A}}$ and $1\leqslant i \leqslant n$ (with the informal meaning that the i-th symbol in a configuration of $M$ is $a$). Let $\sigma^1$, and $\sigma^2$ be `copies' of signature $\sigma$ consisting of the above mentioned concept names with the superscripts $^1$ and $^2$, respectively.

For $\alpha=1,2$, let $\O^\alpha$ be an ontology consisting of the axioms below. The family of axioms (\ref{Eq_Step1})-(\ref{Eq_Step4}) implements transitions of $M$ (while respecting the end positions of configurations): \vspace{-0.2cm}

\begin{equation} \label{Eq_Step1} 
B^{\alpha}_{Xi-2} \dcap B^{\alpha}_{Yi-1} \dcap B^{\alpha}_{Ui} \dcap B^{\alpha}_{Vi+1}\  \dleq \ B_{Wi}
\end{equation}
for $1\leqslant i \leqslant n-3$ and all $X,Y,U,V,W\in Q\cup\mathcal{A}$ such that $XYUV\overset{\delta_{\alpha}}{\mapsto} W$; \vspace{-0.2cm}

\begin{equation}\label{Eq_Step2}
B^{\alpha}_{U1} \dcap B^{\alpha}_{V2}\  \dleq \ B_{W 1}
\end{equation}
for all $U,V,W\in Q\cup\mathcal{A}$ such that $\b\b UV\overset{\delta_{\alpha}}{\mapsto} W$;

\begin{equation}\label{Eq_Step3}
B^{\alpha}_{Y1} \dcap B^{\alpha}_{U2} \dcap B^{\alpha}_{V3}\ \dleq \ B_{W 2}
\end{equation}
for all $Y,U,V,W\in Q\cup\mathcal{A}$ such that $\b YUV\overset{\delta_{\alpha}}{\mapsto} W$;

\begin{equation}\label{Eq_Step4}
B^{\alpha}_{Xn-2} \dcap B^{\alpha}_{Yn-1} \dcap B^{\alpha}_{Un} \  \dleq \ B_{Wn} 
\end{equation}
for all $X,Y,U,W\in Q\cup\mathcal{A}$ such that $XYU\b\overset{\delta_{\alpha}}{\mapsto} W$.

\medskip

For $1\leqslant i \leqslant n$, the next axioms initialize `local' marker $\bar{H}^\alpha$ and `global' marker $\bar{H}$ for a rejecting successor configuration wrt $\delta_\alpha$:

\begin{equation}\label{Eq_LocalAcceptMarkers}
B_{\qrej i} \dleq \bar{H}, \ \ \ \  \bar{H} \ \dleq\ \bar{H}^\alpha
\end{equation}

Let $\O$ be an ontology consisting of the following axioms: \vspace{-0.2cm}

\begin{gather}\label{Eq_GlobalAcceptMarker}
\bar{H}^1\dcap B_{q_\forall i} \ \dleq \ \bar{H}, \ \ \bar{H}^2\dcap  B_{q_\forall i}  \dleq \ \bar{H} \\
\bar{H}^1\dcap \bar{H}^2\dcap B_{q_\exists i}  \dleq \  \bar{H} \nonumber
\end{gather}
for $1\leqslant i \leqslant m$, $\q_\exists\in Q_\exists$, and $\q_\forall\in Q_\forall$ (i.e., these axioms implement the definition of accepting configuration depending on whether the state is existential or universal); \vspace{-0.2cm}

\begin{equation}\label{Eq_Init}
A \dleq \mathbin{\rotatebox[origin=c]{180}{$\bigsqcup$}}_{1\leqslant i \leqslant n+2} B_{\b i}  \dcap B_{q_0 n+3} \dcap \mathbin{\rotatebox[origin=c]{180}{$\bigsqcup$}}_{n+4\leqslant i \leqslant m} B_{\b i}
\end{equation}
representing the initial configuration $\c_{\init}$ of $M$; \vspace{-0.2cm}

\begin{equation}\label{Eq_CopySignature}
B_{a i} \dleq B^\alpha_{a i}
\end{equation}
for $\alpha=1,2$, $1\leqslant i \leqslant n$, and all $a\in Q\cup\mathcal{A}$ (which enforce `copying' a configuration `description' in signature $\sigma$ into signatures $\sigma^1$, $\sigma^2$).

\medskip

Consider ontology network $\N$ consisting of the import relations $\langle \O, \Sigma^\alpha, \O^\alpha \rangle$ and $\langle \O^\alpha, \Sigma, \O \rangle$, where $\Sigma^\alpha=\{\bar{H}^\alpha\}\!\cup\!\sigma^\alpha$, $\Sigma=\{\bar{H}\}\!\cup\!\sigma$, and $\alpha=1,2$. Informally, ontologies $\O^\alpha$ describe transitions between configurations, while $\O$ serves for `copying' configuration descriptions into signatures $\sigma^1, \sigma^2$ and `feeding' them back into $\O^\alpha$. It follows that models $\I\models_\N\O$ represent consequent configurations of $M$ and thus, network $\N$ implements a run tree of ATM. Intuitively, a point $x$ in the domain of a model $\I\models_\N\O$ represents a configuration $\c$, if $x$ belongs to the interpretation of $\sigma$-concept names, corresponding to the symbols in $\c$.

We demonstrate that $M$ does not accept the empty word iff $\O\models_\N A\dleq \bar{H}$. The `only if' direction is proved by induction by showing that in any model $\I\models_\N\O$, whenever a domain element $x$ represents a rejecting configuration $\c$ of $M$, it holds $x\in \bar{H}^\I$. Then it follows by axiom (\ref{Eq_Init}) that $x\in \bar{H}^\I$, whenever $x\in A^\I$. The `if' direction is proved by contraposition by defining a model agreement $\mu$ for $\N$ and a singleton interpretation $\I\in\mu(\O)$ such that $\I\not\models A\dleq \bar{H}$. For every ontology, $\mu$ gives a family of singleton interpretations such that each of them represents a configuration of $M$ and agreed interpretations correspond to consequent configurations. \QED

\begin{theorem}\label{Teo_Hardness_HAcyclic} %[Acyclic $\mathcal{H}$-Networks]
Entailment in acyclic $\mathcal{H}$-Networks is $\PSPACE$-hard.
\end{theorem}

\Proofsketch We show that the word problem for ATMs making polynomially many steps reduces to entailment in acyclic $\mathcal{H}$-ontology networks. Then, since $\AP=\PSPACE$, the claim follows.

Let $M=\langle Q,\mathcal{A},\delta_1, \delta_2 \rangle$ be an ATM. We use the definition of the network $\N$ from the proof sketch to Theorem \ref{Teo_Hardness_HGeneral} and define by induction an acyclic $\mathcal{H}$-ontology network $\N_n$, which can be viewed informally as a finite `unfolding' of $\N$.

For $n=1$, let $\N_1$ be a network consisting of import relations $\langle \O_1, \Sigma^\alpha, \O^{\alpha}_{1} \rangle$, for $\alpha=1,2$, where $\O_1$ is equivalent to ontology $\O$ (from the definition of network $\N$) and $\O^{\alpha}_1$ is equivalent to $\O^{\alpha}$. If $\N_{n-1}$ is a network already given for $n\geqslant 2$, then we define $\N_n$ as the union of $\N_{n-1}$ with the set consisting of import relations $\langle \O_n, \Sigma^\alpha, \O^{\alpha}_{n} \rangle$, $\langle \O^{\alpha}_{n}, \Sigma, \O_{n-1} \rangle$, for $\alpha=1,2$, where $\O_n$, $\O^{\alpha}_{n}$ are ontologies not present in $\N_{n-1}$ and $\O_n$ is equivalent to $\O$ and $\O^{\alpha}_{n}$ is equivalent to $\O^{\alpha}$. By using the arguments from the proof of Theorem \ref{Teo_Hardness_HGeneral} we show that for any $n\geqslant 1$, it holds $\O_n\models_{\N_n} A\dleq \bar{H}$ iff $M$ does not accept the empty word in $n$ steps. \QED

%----------------------------------------------------------------------%
\commentout { %%%%%% OLD HARDNESS PROOFS FOR H-NETWORKS 
\begin{lemma}[General $\mathcal{H}$-Networks]\label{Lem_ExptimeRoleFree-LowerBound}
The word problem for ATMs having a finite tape is reducible to the entailment problem in $\mathcal{H}$-ontology networks.
%The set of tuples $\langle M, n, w\rangle$, where $M$ is a 2ATM having a tape of finite length $n$ and $w$ is a word accepted by $M$, is polynomially 1-reducible to the set of tuples $\langle \N, v, \varphi\rangle$, where $\N$ is a role-free ontology network, $v$ is a node, and $(\N,v)\models\varphi$.
\end{lemma}

%\vspace{-0.8cm}

%\begin{center}
%\begin{figure}\label{Fig_MainReduction}
%\makebox[7cm] {
%\includegraphics[height=1.3in,keepaspectratio]{Images/exptime-simplified.pdf}
%\put(-157,43){$\O_{init}$}
%\put(-85,44){$\O_{step}$}
%\put(-22,82){{\tiny $\O^1_{copy}$}}
%\put(-22,12){{\tiny $\O^2_{copy}$}}

%\put(-127,50){{\tiny $\Si\cup\{A,\!D\}$}}

%\put(-79,80){{\footnotesize $\{A^1, D\}\cup \Sigma^1$}}
%\put(-26,53){{\footnotesize $\{A^1, D\}\cup \Sigma$}}

%\put(-26,33){{\footnotesize $\{A^2, D\}\cup \Sigma$}}
%\put(-79,5){{\footnotesize $\{A^2, D\}\cup \Sigma^2$}}

%}
%\caption{Cyclic ontology network implementing a 2ATM. Labels of vertices and edges are depicted.}
%\end{figure}
%\end{center}

Take an ATM $M$ and a word $w$ such that the tape of $M$ is bounded by $f(|w|)$, for a polynomial $f$. Since an ATM can be converted into 2ATM with no more than polynomial increase in the size of machine's description (and with no increase in the tape size), Lemma \ref{Lem_ExptimeRoleFree-LowerBound} shows a polynomial reduction of the acceptance problem for $M$ to entailment in an ontology network. Since the word problem for ATMs having a polynomially bounded tape is EXPTIME-hard, we obtain EXPTIME-hardness of the entailment problem in $\mathcal{H}$-ontology networks. 

%----------------------------------------------------------------------%

\begin{lemma}\label{Lem_Pspace-LowerBound} %[Hardness for Acyclic $\mathcal{H}$-Networks]
The word problem for ATMs making a polynomially bounded number of steps is reducible to the entailment problem in $\mathcal{H}$-ontology networks.
\end{lemma}

\begin{proof}
Consider a PSPACE-hard set of triples $\langle M, f(|w|), w\rangle$, where $M$ is an ATM, $w$ is a word accepted by $M$, and the minimal height of an accepting run tree of $M$ on $w$ is bounded by $f(|w|)$, for a polynomial $f$. W.l.o.g. we may also assume that $M$ has tape of length $f(|w|)$. Given a triple from this set, an ATM $M$ can be converted into 2ATM with no more than polynomial increase of the minimal height of an accepting run tree of $M$ on the word $w$ (and with no increase of the required tape size). Hence, there exists a PSPACE-hard set of triples as above, where $M$ is a 2ATM.  

Take a triple $T=\langle M, f(|w|), w\rangle$, where $M$ is a 2ATM having a tape of length $f(|w|)$, $w$ is a word in the alphabet of $M$, $f$ is some polynomial, and denote $m=f(|w|)$. Using the notations for signatures from Lemma \ref{Lem_ExptimeRoleFree-LowerBound}, denote $\gamma=\{A^1,B\}\cup\Si$, $\hat{\gamma}=\{A^1,B\}\cup\Si^1$, $\sigma=\{A^2,B\}\cup\Si$, $\hat{\sigma}=\{A^2,B\}\cup\Si^2$. Consider the ontology network $\N$ constructed by $T$ and depicted on Figure 1. The labels of nodes $u,v_1,\ldots , v_m$ and edges are shown in the figure (informally, $\N$ can be viewed as ``unfolding'' of  the graph from Figure 2). 
\commentout{
\begin{figure*}\label{Fig_PSPACEReduction}
\makebox[17.5cm] {
\put(-160,0){\includegraphics[height=0.8in,keepaspectratio]{Images/pspace_start_simplified.pdf}}

\put(55,27){{\Large $\ldots$}}

\put(85,0){\includegraphics[height=0.8in,keepaspectratio]{Images/pspace_end.pdf}}
%\put(-159,45){{\tiny $\O_{init}$}}
%\put(-153,27){{\small $u$}}
%\put(-143,35){{\tiny $\Si\!\cup\!\{\!A,\!\!B\!\}$}}
\put(-115,48){{\tiny $\O_{step}$}}
\put(-108,28){{\small $v_1$}}
\put(-82,47){{\small $\hat{\gamma}$}}
\put(-70,36){{\tiny $\O_{copy}$}}
\put(-43,47){{\small $\gamma$}}
\put(-84,10){{\small $\hat{\sigma}$}}
\put(-70,20){{\tiny $\O_{copy}$}} 
\put(-43,10){{\small $\sigma$}}

\put(-25,48){{\tiny $\O_{step}$}}
\put(-19,28){{\small $v_2$}}
\put(6,47){{\small $\hat{\gamma}$}}
\put(20,36){{\tiny $\O_{copy}$}}
\put(20,20){{\tiny $\O_{copy}$}}
\put(6,10){{\small $\hat{\sigma}$}}

\put(85,36){{\tiny $\O_{copy}$}}
\put(112,45){{\small $\gamma$}}
\put(85,20){{\tiny $\O_{copy}$}}
\put(112,10){{\small $\sigma$}}
\put(130,48){{\tiny $\O_{step}$}}
\put(133,27){{\small $v_m$}}

%\put(-170,-5){\line(100,0){340}}
%\put(-106,-14){1}
%\put(55,-14){{\Large $\ldots$}}
%\put(135,-14){m}

}
\caption{Acyclic ontology network implementing 2ATM with the minimal height of accepting run trees bounded by $m$.}
\end{figure*}
} %commented out

Let $u$ be the node labelled by $\O_{init}$. It can be shown by induction on $1\leqslant k\leqslant m$ by using Definition \ref{De_nAgreedModel} and the idea in the proof of Lemma \ref{Lem_RunTree&AgreedModel} that a configuration $c$ of the machine $M$ is $k$-accepting iff in any $(2k-2)$-agreed model of $\O_{v_{(m+1)-k}}$ representing $c$, the literal $A$ is true. It follows that the initial configuration $c$ of $M$ (on the input word $w$) is $m$--accepting iff in any model $\Mm\in\O_{v_1}$ representing $c$ the literal $A$ is true. The same argument as in the end of the proof of Lemma \ref{Lem_ExptimeRoleFree-LowerBound} shows that $(\N,u)\models C$ iff $M$ accepts $w$. \medskip

\end{proof}

} %%% COMMENTOUT

%%====================================================================%%

%%====================================================================%%

\section{Reduction to Classical Entailment}\label{Sect_Reduction2ClassicalEntailment}

%%====================================================================%%

As a tool for proving upper complexity bounds, we demonstrate that entailment in a network $\N$ can be reduced to entailment from (a possibly infinite) union of `copies' of ontologies appearing in $\N$.

Let $\N$ be an ontology network. We denote $\sig(\N)=\bigcup_{\tuple{\O_1,\Sigma,\O_2}\in\N}\ (\sig(\O_1)\cup\Sigma\cup\sig(\O_2))$.
An \emph{import path} in $\N$ is a sequence
$p=\{\O_0,\Si_1,\O_1,\dots,\O_{n-1},\Si_n,\O_n\}$, $n\ge 0$, such that
$\tuple{\O_{i-1},\Si_i,\O_i}\in\N$ for each $i$ with $(1\le i\le n)$. We
denote by $\length(p)=n$, $\first(p)=\O_0$ and $\last(p)=\O_n$ the \emph{length}
of $p$, the \emph{first} and, respectively, the \emph{last} ontologies on the
path $p$.
By $\paths(\N)$ we define the set of all paths in $\N$, and by
$\paths(\N,\O)=\set{p\in\paths(\N)\mid\first(p)=\O}$ the subset of paths that
originate in $\O$.
We say that $\O'$ is \emph{reachable from $\O$ in $\N$} if there exists a path
$p\in\paths(\N,\O)$ such that $\last(p)=\O'$.
The \emph{import closure} of an ontology $\O$ in $\N$ is defined by
$\icls{\O}=\cup_{p\in\paths(\N,\O)}\last(p)$. Note that by definition it holds $\{\O\}\in\paths(\N,\O)$ and thus, $\O\subseteq\icls{\O}$.

\begin{lemma}\label{lemma:import:closure:completeness}
If $\I\models\icls{\O}$ then $\I\models_\N\O$.
\end{lemma}

\begin{proof}
Consider a mapping $\mu$ defined for ontologies $\O'$ in $\N$ by
setting $\mu(\O')=\set{\I}$ if $\O'$ is reachable from $\O$ and
$\mu(\O')=\emptyset$ otherwise. Clearly, $\mu$ is a model agreement
for $\N$. Since $\I\in\mu(\O)$, we have $\I\models_\N\O$.
\end{proof}

%\todo{Discuss the converse of Lemma~\ref{lemma:import:closure:completeness}}

For every symbol $X\in\sig(\N)$ and every import path $p$ in $\N$,
take a distinct symbol $X_p$ of the same type (concept name, role name, or
individual) not occurring in $\sig(\N)$. For each import path $p$ in $\N$, define a
\emph{renaming} $\theta_p$ of symbols in $\sig(\N)$ inductively as follows.
If $\length(p)=0$, we set $\theta_p(X)=X$ for every $X\in\sig(\N)$.
Otherwise, $p=p'\cup\set{\O_{n-1},\Si_n,\O_n}$ for some path $p'$ and we
define $\theta_{p}(X)=\theta_{p'}(X)$ if $X\in\Si_n$ and 
 $\theta_{p}(X)=X_p$ otherwise.
A \emph{renamed import closure} of an ontology $\O$ in $\N$
is defined by
$\ricls{\O}=\bigcup_{p\in\paths(\N,\O)}\theta_p(\last(p))$.

\begin{lemma}\label{lemma:renamed:closure:completeness}
If $\I\models\ricls{\O}$ then $\I\models_\N\O$.
\end{lemma}

\begin{proof}
The proof is identical to the proof of Lemma~\ref{lemma:import:closure:completeness}.
\end{proof}

\begin{lemma}\label{lemma:renamed:closure:soundness}
For every $\I\models_\N\O$ there exists $\J\models\ricls\O$
such that $\J=_{\sig(\N)}\I$.
\end{lemma}

\begin{proof}
Assume that $\I\models_\N\O$. Then there exists a model agreement $\mu$ for $\N$
such that $\I\in\mu(\O)$.
We define $\J=(\Delta^\J,\cdot^\J)$ by setting $\Delta^\J=\Delta^\I$, and
$X^\J=X^\I$ for all symbols $X$ except for the symbols $X_p$, with $p\in\paths(\N,\O)$
and $X\in\sig(\N)$. For those symbols, we set $(X_p)^\J=X^{\I_p}$ where
$\I_p\in\mu(\last(p))$ is defined by induction on $\length(p)$ as follows.
If $\length(p)=0$, we set $\I_p=\I\in\mu(\O)=\mu(\last(p))$.
Otherwise, $p=p'\cup\set{\O_{n-1},\Si_n,\O_n}$ for some
$\tuple{\O_{n-1},\Si_n,\O_n}\in\N$, and $\I_{p'}\in\mu(\O_{n-1})$ is already
defined. Then pick any $\I_{p}\in\mu(\O_n)$
such that $\I_p=_{\Si_n}\I_{p'}$. Such $\I_p$ always exists since $\mu$ is a model agreement.
This completes the definition of $\J$. Obviously, $\J=_{\sig(\N)}\I$.

To prove that $\J\models\ricls\O$, we first show by induction on $\length(p)$
that for every $X\in\sig(\N)$ we have $(\theta_p(X))^\J=X^{\I_p}$. Indeed, if
$\length(p)=0$ then $(\theta_p(X))^\J=X^\J=X^\I=X^{\I_p}$.
If $p=p'\cup\set{\O_{n-1},\Si_n,\O_n}$ for some
$\tuple{\O_{n-1},\Si_n,\O_n}\in\N$, then if $X\in\Si_n$, we have
$(\theta_p(X))^\J=(\theta_{p'}(X))^\J=X^{\I_{p'}}=X^{\I_p}$ since
$\I_{p}=_{\Si_n}\I_{p'}$. If $X\notin\Si_n$ then
$(\theta_p(X))^\J=(X_p)^\J=X^{\I_p}$.

Now, since for every path $p\in\paths(\N,\O)$, we have $\I_p\in\mu(\last(p))$
hence, in particular, $\I_p\models\last(p)$, we have
$\J\models\theta_p(\last(p))$ by the property above. Hence $\J\models\ricls\O$.
\end{proof}

\begin{theorem}\label{theorem:reduction:classical}
Let $\N$ be an ontology network, $\O$ an ontology in $\N$,
and $\alpha$ an axiom such that $\sig(\alpha)\subseteq\sig(\N)$. 
Then $\O\models_\N\alpha$ iff $\ricls\O\models\alpha$.
\end{theorem}

\begin{proof}
Suppose that $\O\models_\N\alpha$. In order to prove that $\ricls\O\models\alpha$,
take any model $\I\models\ricls\O$. We need to show that $\I\models\alpha$.
Since $\I\models\ricls\O$, by Lemma~\ref{lemma:renamed:closure:completeness}, we have  
$\I\models_\N\O$. Since
$\O\models_\N\alpha$, we have $\I\models\alpha$, as required.

Conversely, suppose that $\ricls\O\models\alpha$. In order to show that
$\O\models_\N\alpha$, take any $\I\models_\N\O$. We need to show that
$\I\models\alpha$. Since $\I\models_\N\O$, by
Lemma~\ref{lemma:renamed:closure:soundness}, there exists $\J\models\ricls\O$
such that $\J=_{\sig(\N)}\I$. Since $\ricls\O\models\alpha$, we have
$\J\models\alpha$.
Since $\sig(\alpha)\subseteq\sig(\N)$, we have $\I\models\alpha$, as required.
\end{proof}

%%====================================================================%%

\section{Membership Results}\label{Sect_Membership}

%%====================================================================%%

Theorem~\ref{theorem:reduction:classical} provides a method for reducing the entailment problem in ontology networks to entailment from ontologies.
Note that, in general, the renamed closure $\ricls{\O}$ of an ontology $\O$ in a (cyclic) network $\N$ can be infinite (even if $\N$ and all ontologies in $\N$ are finite).
There are, however, special cases when $\ricls{\O}$ is finite. 
For example, if all import signatures in $\N$ include all symbols in $\sig(\N)$, then it is easy to see that $\ricls{\O}=\icls{\O}$. 
$\ricls{\O}$ is also finite if $\paths(\N,\O)$ is finite, e.g., if $\N$ is acyclic. 
In this case, the size of $\ricls{\O}$ is at most exponential in $\O$. If there is at most one import path between every pair of ontologies (i.e., if $\N$ is tree-shaped) then the size of $\ricls{\O}$ is the same as the size of $\N$. This immediately gives the upper complexity bounds on deciding entailment in acyclic networks. 

\begin{theorem}\label{theorem:upper:acyclic}
Let $\LL$ be a DL with the complexity of entailment in $\complclass{[co][N]TIME}(f(n))$ (\textsf{[co]} and \textsf{[N]} denote possible co- and N-prefix, respectively). 
Let $\N$ be an acyclic ontology network and $\O$ an ontology in $\N$ such that $\ricls{\O}$ is a $\LL$-ontology. Then for $\LL$-axioms $\alpha$, the entailment $\O\models_\N\alpha$ is decidable in $\complclass{[co][N]TIME}({f(2^n)})$. If $\N$ is tree-shaped then deciding $\O\models_\N\alpha$ has the same complexity as entailment in $\LL$.
\end{theorem}

Note that if $\O_1,\ldots,\O_m$ are some ontologies in a DL $\LL$, then in general, their union is not necessary a $\LL$-ontology. For instance, this is the case for logics containing DL $\R$, which restricts role inclusion axioms to regular ones. The regularity property can be easily lost when taking the union of ontologies and thus, reasoning over the union of ontologies may be harder than reasoning in the underlying DL. Hence, the requirement in Theorem \ref{theorem:upper:acyclic} that $\ricls{\O}$ must be a $\LL$-ontology. 

In the next theorem, we show that for arbitrary networks, checking entailment is, in general, semi-decidable, which is a consequence of the Compactness Theorem for First-Order Logic (since all standard DLs can be translated to FOL).

\begin{theorem}\label{Teo_UpperBoundCyclic}
Let $\LL$ be a DL, which can be translated to FOL, $\N$ an ontology network, 
and $\O$ an ontology in $\N$ such that $\ricls{\O}$ is a $\LL$-ontology. 
Then for $\LL$-axioms $\alpha$, the entailment $\O\models_\N\alpha$ is semi-decidable. 
\end{theorem}

\begin{proof}
By Theorem~\ref{theorem:reduction:classical}, $\O\models_\N\alpha$ iff $\ricls{\O}\models\alpha$. By the compactness theorem for first-order logic,
if $\ricls{\O}\models\alpha$ then there exists a finite subset $\O'\subseteq\ricls{\O}$ such that $\O'\models\alpha$.
Hence, $\ricls{\O}\models\alpha$ can be checked, e.g, by enumerating all finite subsets
$\ricls{\O}_n=\bigcup_{p\in\paths(\N,\O,n)}\theta_p(\last(p))\subseteq\ricls{\O}$, $n\ge 0$, where $\paths(\N,\O,n)=\set{p\in\paths(\N,\O)\mid \length(p)\le n}$ and running the (semi-decidable) test $\ricls{\O}_n\models\alpha$ with the timeout $n$. If $\ricls{\O}\models\alpha$ then, eventually, one of these tests succeeds. 
\end{proof}

Restricting the shape of the network is one possibility of establishing decidability results for entailment in ontology networks.
Another possibility is to restrict the language.
It turns out, for ontology networks expressed in the role-free DL $\P$, the entailment problem becomes decidable, even in the presence of cycles.  
Intuitively, this is because the entailment in $\P$ can be characterized by a bounded number of models. 

\begin{definition}\label{definition:singleton}
We say that an interpretation $\I=(\Delta^\I,\cdot^\I)$ is a \emph{singleton} if $\sizeof{\Delta^\I}=1$.
Let $\I=(\Delta^\I,\cdot^\I)$ be a DL interpretation and $d\in\Delta^\I$.
The \emph{singleton projection of $\I$ to $d$} is the interpretation $\J=(\set{d},\cdot^{\J})$ such that $A^\J=A^\I\cap\set{d}$ for each $A\in\Nc$, $R^\J=\emptyset$ for each $R\in\Nr$, and $a^\J=d$ for each $a\in\Ni$.
\end{definition}

\begin{lemma}\label{lemma:singleton:concept}
Let $C$ be a $\P$-concept, $\I=(\Delta^\I,\cdot^\I)$ an interpretation, and $\J=(\set{d},\cdot^\J)$ a singleton projection of $\I$ on some element $d\in\Delta^\I$.
Then $C^\J=C^\I\cap\Delta^\J$. 
\end{lemma}

\begin{corollary}\label{corollary:signleton:entailment}
Let $\alpha=C\sqsubseteq D$ be a $\P$-axiom, $\I=(\Delta^\I,\cdot^\I)$ an interpretation such that $\I\models\alpha$, and $\J$ a singleton projection of $\I$ on an element $d\in\Delta^\I$.
Then $\J\models\alpha$.
\end{corollary}

\begin{corollary}\label{corollary:signleton:non-entailment}
Let $\alpha=C\sqsubseteq D$ be a $\P$-axiom, 
$\I=(\Delta^\I,\cdot^\I)$ an interpretation such that $\I\not\models\alpha$. 
Then there exists $d\in\Delta^\I$ such that for the singleton projection $\J$ of $\I$ to $d$, $\J\not\models\alpha$.
\end{corollary}

Given an ontology network $\N$, a \emph{singleton model agreement} for $\N$ is a model agreement $\mu$ such that
for every $\O$ in $\N$ every interpretation $\I\in\mu(\O)$ is a singleton.

\begin{lemma}\label{lemma:singleton:model:agreement}
Let $\N$ be a $\P$-ontology network, $\O$ an ontology in $\N$, and $\alpha$ a $\P$-axiom such that $\O\not\models_\N\alpha$.
Then there exists a singleton model agreement $\mu'$ for $\N$ and a model $\I'\in\mu'(\O)$ such that $\I'\not\models\alpha$.
\end{lemma}

\begin{proof}
Since $\O\not\models_\N\alpha$, there exists a model agreement $\mu$ over a domain $\Delta$ such that for some $\I\in\mu(\O)$ we have $\I\not\models\alpha$.
By Corollary~\ref{corollary:signleton:non-entailment}, there exists an element $d\in\Delta$ such that for the singleton projection $\I'$ of
$\I$ to $d$, we have $\I'\not\models\alpha$. 
Now define a mapping $\mu'$ by setting $\mu'(\O)$, for every ontology $\O$ in $\N$, to consist of the singleton projections of the interpretations in $\mu(\O)$ to $d$.
By Corollary~\ref{corollary:signleton:entailment}, we have $\I'\models\O$, for all $\I'\in\mu'(\O)$.
To prove that $\mu'$ is a model agreement it remains to show that if $\tuple{\O_1,\Si,\O_2}\in\N$ and $\I_1'\in\mu'(\O_1)$
then there exists $\I_2'\in\mu'(\O_2)$ such that $\I_1'=_\Si\I_2'$.
Indeed, since $\I_1'\in\mu'(\O_1)$ then $\I_1'$ is a singleton projection of some $\I_1\in\mu(\O_1)$.
Since $\mu$ is a model agreement, there exists $\I_2\in\mu(\O_2)$ such that $\I_1=_\Si\I_2$.
Let $\I_2'$ be the singleton projection of $\I_2$ to $d$.
Then $\I_2'\in\mu'(\O_2)$ by definition of $\mu'$.
Furthermore, since $\I_1=_\Si\I_2$, it is easy to see by the Definition~\ref{definition:singleton} that $\I_1'=_\Si\I_2'$.    
\end{proof}

Lemma~\ref{lemma:singleton:model:agreement} implies, in particular, that to check the entailment in $\P$ networks, it is sufficient to restrict  attention only to singleton model agreements.
W.l.o.g., one can assume that these singleton model agreements have the same domain. %since renaming the domain element does not make any difference for the entailment.
Similarly, only interpretation of symbols that appear in $\N$ or in the checked axiom $\alpha$ counts. 
Since the number of interpretations of concept names over one element domain is at most exponential in the number of concept names, for checking entailment $\O\models_\N\alpha$ it is sufficient to restrict attention to only exponentially-many singleton model agreements in the size of $\N$ and $\alpha$.
This gives a simple $\NEXPTIME$ algorithm for checking whether $\O\not\models_\N\alpha$: guess a singleton model agreement $\mu$ (of an exponential size), and check whether $\I\not\models\alpha$ for some $\I\in\mu(\O)$.
It is, however, possible to find the required model agreement deterministically, thereby reducing the complexity to $\EXPTIME$.

\begin{theorem}\label{Teo_UpperBound_PGeneral}
There is an $\EXPTIME$ procedure that given a $\P$-ontology network $\N$, an ontology $\O$  in $\N$, and a $\P$-axiom $\alpha$,
checks whether $\O\models_\N\alpha$.
\end{theorem}

\begin{proof} 
Let $\Si$ be the set of all signature symbols appearing in $\N$ and $\alpha$. 
Since $\N$ and $\alpha$ are formulated in $\P$, $\Si$ consists of only concept names.
Let $d$ be a fixed (domain) element.
For every subset $s\subseteq\Si$, let $\I(s)=(\set{d},\cdot^{\I(s)})$ be a singleton interpretation defined by $A^{\I(s)}=\set{d}$ if $A\in s$ and $A^{\I(s)}=\emptyset$ otherwise.
Finally, let $m$ be a mapping defined by $m(\O)=2^\Si$ for all ontologies $\O$ in $\N$.
Clearly, the mapping $m$ can be constructed in exponential time in the size of $\N$.

The mapping $m$ corresponds to the assignment $\mu$ %$\mu=\mu(m)$% 
of singleton interpretations to ontologies in $\N$ defined as $\mu(\O)=\set{\I(s)\mid s\in m(\O)}$.
This assignment, however, is not necessarily a model agreement for $\N$ according to Definition~\ref{def:ontology:network}.
First, not all interpretations $\I(s)$ for $s\in m(\O)$ are models of $\O$.
Second, it is not guaranteed that for every $\tuple{\O_1,\Si,\O_2}\in\N$ and every $s_1\in m(\O_1)$ there exists $s_2\in m(\O_2)$ such that $\I(s_1)=_\Si\I(s_2)$.
To fix the defects of the first type, we remove from $m(\O)$ all sets $s$ such that $\I(s)\not\models\O$. 
(It is easy to check in polynomial time if a singleton interpretation is a model of an ontology).   
To fix the defects of the second type, we remove all $s_1\in m(\O_1)$ for which there exists no $s_2\in m(\O_2)$ such that $s_1\cap\Si=s_2\cap\Si$.
We repeat performing this operation until no defects are left.

Clearly, both operations can be performed in exponential time in $\N$ since there are at most exponentially-many values $s$ that can be removed.
Finally, to decide whether $\O\models_\N\alpha$, we check whether $\I(s)\models\alpha$ for all $s\in m(\O)$.
If this property holds, we return $\O\models_\N\alpha$; otherwise, we return $\O\not\models_\N\alpha$.

We claim that the above algorithm is correct.
Indeed, if $\O\not\models_\N\alpha$ is returned then there is a model agreement $\mu$ %$\mu=\mu(m)$ 
and an interpretation $\I\not\models\alpha$ such that $\I\in\mu(\O)$.

Conversely, if $\O\not\models_\N\alpha$ then there exists a model agreement $\mu$ for $\N$ such that $\I\not\models\alpha$ for some $\I\in\mu(\O)$. 
Then by Lemma~\ref{lemma:singleton:model:agreement}, there exists a singleton model agreement $\mu'$ for $\N$ such that $\I'\not\models\alpha$ for some $\I'\in\mu'(\O)$. 
For a singleton interpretation $\I$, let $s(\I)=\set{A\mid A^\I\neq\emptyset}$.
W.l.o.g., $\I(s(\I'))=\I'$ for each $\I'\in \mu'(\O)$, with $\O$ in $\N$.
Let $m'$ be a mapping defined by $m'(\O)=\set{s(\I)\mid \I\in\mu'(\O)}$ for each $\O$ in $\N$.
So $\I(s)\in\mu'(\O)$ for every $s\in m'(\O)$ and $\O$ in $\N$.
By induction over the construction of $m$, it is easy to show that $m'(\O)\subseteq m(\O)$ for every $\O$ in $\N$.
Indeed, since $m(\O)$ is initialized with all subsets of the signature $\Si$, $m'(\O)\subseteq m(\O)$ holds in the beginning.
Furthermore, for each $s\in m'(\O)$, we have $\I(s)\in\mu'(\O)$, and so, $\I(s)\models\O$. 
Thus, $s$ cannot be removed from $m(\O)$ as a defect of the first type.
Similarly, for every $\tuple{\O_1,\Si,\O_2}\in\N$ and every $s_1\in m'(\O_1)$, we have $\I_1=\I(s_1)\in\mu'(\O_1)$.
Since $\mu'$ is a model agreement, there exists $\I_2\in\mu'(\O_2)$ such that $\I_1=_{\Si}\I_2$.
Hence for $s_2=s(\I_2)\in m'(\O_2)$, we have $s_1\cap\Si=s_2\cap\Si$.
Therefore, $s_1$ cannot be removed from $m(\O_1)$ as a defect of the second type.
Finally, since $\I'\not\models\alpha$ for some $\I'\in\mu'(\O)$, we have $s'=s(\I')\in m'(\O)\subseteq m(\O)$.
Hence, our algorithm returns $\O\not\models_\N\alpha$.
\end{proof}

It is possible to improve the upper bound obtained in Theorem~\ref{Teo_UpperBound_PGeneral} for acyclic $\P$-ontology networks.

\begin{theorem}\label{Teo_UpperBound_PAcyclic}
There is a $\PSPACE$ procedure that given an acyclic $\P$-ontology network $\N$, an ontology $\O$ in $\N$, and a $\P$-axiom $\alpha$,
checks whether $\O\models_\N\alpha$.
\end{theorem}

\begin{proof}
As in the proof of Theorem~\ref{Teo_UpperBound_PGeneral}, let $\Si$ be the set of all signature symbols appearing in $\N$ and $\alpha$, and $d$ a fixed (domain) element.
For each $s\subseteq\Si$, let $\I(s)=(\set{d},\cdot^\I(s))$ be a singleton interpretation defined by $A^{\I(s)}=\set{d}$ if $A\in s$ and $A^{\I(s)}=\emptyset$ otherwise. 

We describe a recursive procedure $P(\O,s)$ that given an ontology $\O$ in $\N$ and $s\subseteq\Si$ returns \emph{true} if there exists a model agreement $\mu$ for $\N$ such that $\I(s)\in\mu(\O)$, and returns \emph{false} otherwise.
The procedure works as follows.
If $\I(s)\not\models\O$, $P(\O,s)$ return \emph{false}.
Otherwise, we iterate over all $\tuple{\O,\Si,\O'}\in\N$ and for every $s'\subseteq\Si$ such that $s\cap\Si=s'\cap\Si$ and run $P(\O',s')$ recursively. 
If for each $\tuple{\O,\Si,\O'}\in\N$ some of the recursive call $P(\O',s')$ returned \emph{true}, we return \emph{true} for $P(\O,s)$.
Otherwise, we return \emph{false}.

Clearly, $P(\O,s)$ always terminates since $\N$ is acyclic.
Furthermore, the procedure can be implemented in polynomial space in the size of $\N$ and $\alpha$, since the recursion depth is bounded by the size of $\N$ %$\sizeof{\N}$
 and at every recursive call, only the input values $\O$ and $s$ need to be saved (assuming the iterations over the import relations and subsets of $\Sigma$ use some fixed order).

Next we show that our procedure is correct.
Assume that $P(\O,s)$ returns \emph{true} for some input $\O$ and $s\subseteq\Si$.
For each ontology $\O$ in $\N$, let $m(\O)$ be the set of all $s\subseteq\Si$ such that there was a (recursive) call $P(\O,s)$ with the output \emph{true}.
Let $\mu$ be a mapping defined by $\mu(\O)=\set{\I(s)\mid s\in m(\O)}$.
We claim that $\mu$ is a model agreement for $\N$.

Indeed, since $P(\O,s)$ returns \emph{true} only if $\I(s)\models\O$, for every $\O$ in $\N$ and $\I\in\mu(\O)$ we have $\I\models\O$.
Furthermore, if $\tuple{\O,\Si,\O'}\in\N$ and $\I\in\mu(\O)$ then, by definition of $\mu$, there exists $s\in m(\O)$ such that $\I=\I(s)$ and $P(s,\O)$ returned \emph{true}. 
In particular, since $\tuple{\O,\Si,\O'}\in\N$, there was a recursive call $P(s',\O')$ that returned \emph{true} for some $s'\subseteq\Si$ such that $s\cap\Si=s'\cap\Si$.
By the definition of $m$, this means that $s'\in m(\O')$.
Therefore, $\I=_\Si\I(s')\in\mu(\O')$, by the definition of $\mu$.
Thus, $\mu$ is a model agreement for $\N$.

Conversely, assume that there exists a model agreement $\mu$ for $\N$ such that $\I(s)\in\mu(\O)$.
Since $\N$ is a $\P$-ontology network, w.l.o.g., for every $\O$ in $\N$ and every $\I\in\mu(\O)$, there exists $s\subseteq\Si$ such that $\I=\I(s)$.
We prove that $P(\O,s)$ returns \emph{true} for every $\O$ and $s$ such that $\I(s)\in\mu(\O)$.
Indeed, assume to the contrary that $P(\O,s)$ returns \emph{false} for some $s$ such that $\I(s)\in\mu(\O)$
and when executing $P(\O,s)$, there was no other recursive call to $P(\O',s')$ that returned \emph{false} for some $s'$ such that $\I(s')\in\mu(\O')$.
Since $\I(s)\in\mu(\O)$, we have $\I(s)\models\O$, thus $P(\O,s)$ cannot return \emph{false} due to $\I(s)\not\models\O$. 
Hence there exists $\tuple{\O,\Si,\O'}\in\N$ such that for every $s'\subseteq\Si$ with $s\cap\Si=s'\cap\Si$, the recursive call of $P(\O',s')$ returned \emph{false}.
Then $\I(s')\notin\mu(\O')$ for all such $s'$ by our assumption above.
Since $\mu$ is a model agreement, $\tuple{\O,\Si,\O'}\in\N$, and $\I(s)\in\mu(\O)$, there exists $\I'\in\mu(\O')$ such that $\I(s)=_\Si\I'$.  
Then $\I'=\I(s')$ for some $s'$ such that $s\cap\Si=s'\cap\Si$.
This gives us a contradiction since $\I(s')\in\mu(\O')$ for no $s'$, with $s\cap\Si=s'\cap\Si$.    

Now, to check whether $\O\models_\N\alpha$ using the procedure $P$, we enumerate all $s\subseteq\Sigma$ and check whether $\I(s)\not\models\alpha$ and $P(\O,s)$ returns \emph{true}.
We return $\O\not\models_\N\alpha$ if such $s$ exists, and $\O\models_\N\alpha$ otherwise.
This algorithm is correct.
Indeed, if such $s$ exists, then there exists a model agreement $\mu$ for $\N$ such that $\I(s)\in\mu(\O)$.
Hence, $\O\not\models_\N\alpha$.
Conversely, if $\O\not\models_\N\alpha$ then there exists a model agreement $\mu$ for $\N$ and some $\I\in\mu(\O)$ such that $\I\models\O$ and $\I\not\models\alpha$.
By Lemma~\ref{lemma:singleton:model:agreement}, w.l.o.g. one can assume that $\mu$ is a singleton model agreement.
Then there exists some $s\subseteq\Sigma$ such that $\I(s)\in\mu(\O)$ and $\I(s)\not\models\alpha$.
Since $\I(s)\in\mu(\O)$, $P(\O,s)$ should return \emph{true}. 
Hence, our algorithm returns $\O\not\models_\N\alpha$.
\end{proof}

%----------------------------------------------------------------------%

\commentout{ %%%%% OLD RESULTS: CORRESPONDENCE TO UNRAVELLING OF KRIPKE STRUCTURES, CONSTRUCTIVE DEFINITIONS TO SHOW ENUMERABILITY OF THE RENAMED IMPORT CLOSURE, SUFF. CONDITIONS TO GUARANTEE REDUC. 2 CLASSICAL ENTAILMENT
\Todo{Replace reachability by importing}

Whenever the logic $\LL$ is not specified, we omit the prefix $\LL$ when referring to a network. Given a network $\N=(V, E,\tau)$ and a node $v\in V$, we use the notation $\O_v$ for the label ontology $\tau(v)$.  We say that a node $v\in V$ is \textit{reachable} from some $u\in V$ in $\N$ if there exists a \textit{path} $p$ from $u$ to $v$, i.e. a sequence of nodes $v_1,\ldots , v_n$ in $\N$ such that $v_1=u$, $v_n=v$, and $\langle v_i \Si_i v_{i+1}\rangle\in\N$, for some signature $\Si_i$. If $\si$ is a signature which is a subset of every $\Si_i$, then $v$ is said to be $\si$-reachable from $u$ (via the path $p$). An ontology network is \textit{finite} if it has a finite set of nodes. A network $\N=(V, E,\tau)$ is called \textit{acyclic} if no node of $\N$ is reachable from itself. For a node $v\in V$, network $\N$ is called $v$-pointed if any other node $u\in V$ is reachable from $v$.  Network $\N$ is called \textit{tree-shaped} if it is $v$-pointed for some node $v\in V$ (called \textit{root}) and for any other vertex $u\in V$ there is exactly one path from $v$ to $u$. \medskip

The semantics of an ontology network $\N$ is defined in terms of a mapping of nodes of $\N$ to models of their label ontologies which respect the connections between the vertices. First, we introduce the notion of interpretation function for an ontology network and then define the semantics. \smallskip

Let $\N=(V,E,\tau)$ be an ontology network and $\mathfrak{I}$ be the class of all interpretations. A map $f: V \rightarrow 2^\mathfrak{I}$ is called interpretation function of $\N$ (or \textit{interpretation} of $\N$, for short) if it satisfies the following properties: 

\begin{itemize}
\item for all $u\in V$, $f(u)$ is a subset of models of $\O_u$;
\item for any model $\I\in f(u)$ and any $v\in V$, with $\langle u \Si v\rangle\in\N$, there exists a model $\J\in f(v)$ such that $\I=_\Si \J$.
\end{itemize}

Note that by this definition, the function mapping each vertex $v\in V$ to the empty set of models of $\O_v$ is an interpretation of $\N$. Consider the following relation between interpretations of $\N$: $f_1\leqslant f_2$ if for every $v\in V$, we have $f_1(v)\subseteq f_2(v)$. Clearly, $\leqslant$ is a partial order on interpretations. Define the point-wise union of interpretations as follows: if $\mathfrak{F}$ is a non-empty subset of interpretations of $\N$ then $g=\bigcup_{f\in \mathfrak{F}} f$ is a map such that $g(v)=\bigcup_{f\in\mathfrak{F}} f(v)$, for $v\in V$. It is easy to see that $g$ is an interpretation of $\N$ and the set of all interpretations of $\N$ with the relation $\leqslant$ is a complete lattice, with the supremum of a subset of interpretations equal to their point-wise union. 

\begin{definition}[Semantics and Local Entailment]\label{De_SupremumSemantics}
Let $\N=(V,E,\tau)$ be an ontology network and $L$ be the lattice of interpretations of $\N$. The semantics function $\nu$ for $\N$ (or \textit{semantics} of $\N$, for short) is the supremum of $L$.

We say that a concept inclusion $C\dleq D$ is (locally) entailed in $\N$ at a node $v\in V$ (in symbols, $(\N,v)\models C\dleq D$) if $C\dleq D$ holds in every model from $\nu(v)$.
\end{definition}

The semantics of $\N$ can be also characterized with the help of the following inductive definition.

\begin{definition}[$m$-agreed Model]\label{De_nAgreedModel}
Let $\N=(V,E,\tau)$ be an ontology network and $v\in V$ be a node. 
Any model of $\O_v$ is called $0$-agreed. For a natural number $m$, a model $\I$ of $\O_v$ is called $m$-agreed if for any $u\in V$ such that $\langle v \Si u\rangle\in\N$, there exists an $(m-1)$-agreed model $\J$ of $\O_u$ such that $\I=_\Si\ J$.
\end{definition}

It is easy to verify that if $v\in V$, and $\I$ is a model of $\O_v$, then $\I\in \nu(v)$ iff $\I$ is $m$-agreed for all $m$.

\Todo{Examples of different importing situations: transitivity, import by different paths (diamond), second-order effects}

\Todo{Show as examples: 
1) how to express axioms $A\sqsubseteq\exists R^{2^n}.B$ and $\exists R^{2^n}.B\sqsubseteq A$ in acyclic $\EL$ networks; 
2) the lack of finite model property for entailment in cyclic $\EL$ networks}

Note that entailment in ontology networks allows to obtain no more inferences then entailment from the union of label ontologies, as formulated below. 

\begin{lemma}[Relationship to Entailment from Union]\label{Lem_ModeloftheUnionIsAgreed}
Let $\N=(V,E,\tau)$ be a $v$-pointed ontology network for some node $v\in V$. For any interpretation $\I$, if $\I\models\bigcup_{u\in V}\O_u$ then $\I\in\nu(v)$ and hence, for any concept inclusion $C\dleq D$, $(\N,v)\models C\dleq D$ yields $\bigcup_{u\in V}\O_u\models C\dleq D$.
\end{lemma}

\begin{proof} 
Take a model $\I\models\bigcup_{u\in V}\O_u$ and consider the map $f:V\mapsto \{\I\}$. Clearly, this is an interpretation of $\N$, hence $f(v)\leqslant\nu(v)$ and $\I\in\nu(v)$. 
\end{proof}

From the computational point of view, we are interested in the complexity of the entailment problem in ontology networks, as well as restrictions on network topologies and labellings that can influence the complexity. 

\begin{definition}[The Problem of Entailment wrt Imports]\label{De_EntailmentProblem}
The entailment problem (wrt imports) is to decide for a finite ontology network $\N$, a node $v$, and a concept inclusion $C\dleq D$, whether $(\N,v)\models C\dleq D$.
\end{definition}

When studying the complexity of this problem, we consider the size of a network $\N=(V,E,\tau)$ as the sum of sizes of its label ontologies (assuming that an ontology size is the total length of its axioms considered as strings). 

\begin{lemma}[Reduction to Pointed Networks]\label{Lem_Reduction2PointedNetworks}
For a network $\N$, a node $v$, and a concept inclusion $C\dleq D$ it holds $(\N,v)\models C\dleq D$ iff $C\dleq D$ is entailed at $v$ in the network obtained by reducing $\N$ onto the nodes reachable from $v$.
\end{lemma}

For this reason, all the complexity proofs will be given for pointed networks.

\Todo{Illustrate how entailment in networks can be reduce to entailment from (a possibly infinite) set of DL axioms. Show that for acyclic networks this set can be bounded by an exponential function in the size of the network.}

\Todo{Add a short statement in this section, showing implementation of OWL:imports}

In this section, we define unfolding of an ontology network which is analogous to the well-known tree unravelling of Kripke models from the graph-theoretic point of view. 

\Todo{Do not remove these Kripke-like statements, if possible. It seems, this analogy can lead to deep results on theory networks.}

\begin{definition}[Tree Unfolding of Ontology Network]\label{De_NonRenamedTreeUnfolding}
Let $\N=(V,E,\tau)$ be a $v$-pointed ontology network. A tree unfolding of $\N$ is a network $\N'=(V',E',\tau')$ defined as follows. The set $V'$ consists of all finite sequences of nodes $u={[}v_0,\ldots v_n{]}$ corresponding to a path in $\N$ such that $v_0=v$ and $\tau'(u)=\tau(v_n)$. There is a directed edge $e$ from a node ${[}v_0,\ldots v_m{]}$ to ${[}u_0,\ldots u_n{]}$ in $\N'$ iff $m=n+1$, $v_i=u_i$, for all $i\leqslant n$, and $\tau'(e)=\Si$, where $\Si$ is a signature such that $\langle v_n\Si v_{n+1}\rangle\in\N$.
\end{definition}

Similar to the Kripke semantics, if $\N'$ is a tree unfolding of a $v$-pointed network $\N$, then $(\N,v)\models C\dleq D$ iff $(\N',[v])\models C\dleq D$, for any concept inclusion $C\dleq D$. We formulate this property in the following statement, which is a direct consequence of network semantics. 

\begin{lemma}[Relation to Semantics of Tree Unfolding]\label{Lem_SemanticsOfTreeUnfolding}
\Todo{Do not remove this lemma, it is used in the upper bounds}
Let $\N'=(V',E',\tau')$ be a tree unfolding of a $v$-pointed ontology network $\N=(V,E,\tau)$. Then for any interpretation $\I$, it holds $\I\in\nu(v)$ iff there is an interpretation function of $\N'$ such that for every $u\in V'$, $f(u)$ is a singleton and $f([v])=\{\I\}$. 
\end{lemma}

We will show that tree representations allow to obtain upper complexity bounds by reducing the entailment problem for ontology networks to entailment from the union of label ontologies. For this, we introduce a modification of the above notion, which differs in the definition of labelling.

\begin{definition}[Renamed Unfolding]\label{De_TreeUnfolding}
Let $\N=(V,E,\tau)$ be a $v$-pointed ontology network. A renamed (tree) unfolding of $\N$ is a network $\N_t=(V_t,E_t,\tau_t)$, with the sets of nodes and edges given as in the definition of tree unfolding and the labelling function $\tau_t$ satisfying the following properties:

\begin{itemize}
\item[1.] for every node $u={[}v_0,\ldots ,v_n{]}$, $\tau_t(u)$ is a ``copy'' of the label ontology $\tau(v_n)$ under an injective renaming of signature symbols, a permutation $\pi_u$ on $\Nc\cup\Nr$ mapping $\Nc$ to $\Nc$ and $\Nr$ to $\Nr$, such that: 

\begin{itemize} 

\item[1.1] $\pi_{{[}v{]}}$ is the identity function;

\item[1.2] for all nodes $u={[}v_0,\ldots , v_{n}{]}$, $w={[}v_0,\ldots ,v_{n+1}{]}$, and a signature $\Si$ such that $\langle v_n\Si v_{n+1}\rangle \in\N$, it holds $\pi_u|_\Si=\pi_w|_\Si$;

\item[1.3] if $\tau_t(u)\cap\tau_t(w)=\si^*\neq\varnothing$ for some nodes $u\neq w$ then $\pi_u^{-1}|_{\si^*}=\pi_w^{-1}|_{\si^*}$ (denote the preimage of $\si^*$ by $\si$) and $u,w$ have the form $u={[}u_0,\ldots, u_{k}, \ldots , u_n{]}$, $w={[}w_0,\ldots, w_{k},\ldots , w_m{]}$, for some $k$ such that $u_l=w_l$, for $0\leqslant l \leqslant k$, $u_n$ is $\si$-reachable from $u_k$ via the path  $u_{k}, \ldots , u_n$, and $w_m$ is $\si$-reachable from $w_k$ via the path  $w_{k}, \ldots , w_m$. %$\langle u_i\Si_i u_{i+1}\rangle\in\N$ and $\langle w_j\Gamma_j w_{j+1}\rangle\in\N$, where $k\leqslant i \leqslant n-1$, $k\leqslant j \leqslant m-1$ and $\Si_i$,$\Gamma_j$ are signatures containing $\si$;

\end{itemize}

\item[2.] for nodes $u={[}v_0,\ldots v_{n}{]}$, $w={[}v_0,\ldots v_{n+1}{]}$, and a signature $\Si$ such that $\langle v_n\Si v_{n+1}\rangle \in\N$, the label of an edge $e$ from $w$ to $u$ is given as $\tau_t(e)=\{\pi_{u}(s) \mid s\in\Si\}$;

%\item for nodes $u={[}v_0,\ldots v_{i}{]}$, $w={[}v_0,\ldots v_{j}{]}$ and signatures $\Si_1,\Si_2$, it holds $\pi_u(\Si_1)=\pi_w(\Si_2)$ iff $\Si_1=\Si_2$, $\pi_u(\Si_1)=\pi_w(\Si_2)$, and there is a node $z$ (possibly equal to $u$, or $w$) in $\N$ such that both $v_i$ and $v_j$ are $\Si_1$-reachable from $z$.

 %$\sig(\tau_t(u_1))\cap\sig(\tau_t(u_2))=\Si\neq\varnothing$ iff
%\begin{itemize}
%	 \item $\pi_{u_1}^{-1}(\si)=\pi_{u_2}^{-1}(\si)$; (denote the preimage of $\si$ as $\Si$); 
%	 \item $v_i$ is $\Si$-reachable from $v_j$ in $\N$ (or vice versa) and $\Si\subseteq\sig(\tau(u_i))\cap\sig(\tau(u_j))$.
\end{itemize}

\end{definition}

\Todo{Example of a renamed unfolding for a diamond}

The number of nodes in a tree unfolding of an ontology network $\N$ is infinite if $\N$ has cycles and otherwise can be exponential in the number of nodes in $\N$ (if $\N$ is not tree-shaped). In connection with the Theorem \ref{Teo_Reduction2Union} below, this can give a hint on the influence of network topology on the complexity of entailment. We address this question in Section \ref{Sect_MainResults} and describe an algorithmic procedure that constructs a renamed unfolding for any given ontology network. The theorem below will be used in Section \ref{Sect_MainResults} for proving upper complexity bounds for the entailment problem in ontology networks. 

\begin{theorem}[Characterization of Entailment]\label{Teo_Reduction2Union}
Let $\N$ be a $v$-pointed ontology network and $\N_t=(V_t,E_t,\tau_t)$ be a renamed unfolding of $\N$. Then for any concept inclusion $C\dleq D$ in signature $\sig(\O_v)$, it holds $(\N,v)\models C \dleq D$ iff $\bigcup_{u\in V_t}\O_u\models C\dleq D$.
\end{theorem}

\begin{proof} 
The proof of the theorem is based on the three statements below.

\begin{proposition}[Relation to Semant. of Renamed Unfolding]\label{Prop_RelationshipSemanticsTreeUnfolding}
Let $\N=(V,E,\tau)$ be a $v$-pointed ontology network and $\N_t=(V_t,E_t,\tau_t)$ be a renamed  unfolding of $\N$ with the root ${[}v{]}$. Then it holds $\nu(v)=\nu({[}v{]})$. 
\end{proposition}

\begin{proof} 
By the definition of network semantics, the claim is proved if we show that for any interpretation $\I$, it holds $\I\in\nu(v)$ iff there is an interpretation function $f$ of network $\N_t$ such that for every $u\in V_t$, $f(u)$ is a singleton and $f({[}v{]})=\{\I\}$. 

Let $\N'=(V',E',\tau')$ be a tree unfolding of $\N$. By Lemma \ref{Lem_SemanticsOfTreeUnfolding}, we have $\I\in\nu(v)$ iff there is an interpretation function $g$ of $\N'$ such that $g(w)$ is a singleton for each $w\in V'$, and $g({[}v{]})=\{\I\}$. 
For each node $u\in V_t$, let $\pi_u$ be the permutation associated with $u$ in the definition of the renamed unfolding $\N_t$. Now observe that the mappings $f$ and $g$ can be mutually defined via the following relation: for any $u\in V_t$, $f(u)=\{(\Delta, \cdot^f)\}$ iff $g(u)=\{(\Delta, \cdot^g)\}$, where $\cdot^f\circ\pi_u=\cdot^g$ 
\end{proof}

\begin{proposition}[Characterization of Reduction to Union]\label{Prop_CharacterizationReduction2Union}
For any $v$-pointed ontology network $\N=(V,E,\tau)$ the following conditions are equivalent:

\begin{trivlist}
\item - for the signature $\Si=\sig(\O_v)$, it holds\newline $\{\I|_{\Si} \ \mid \ \I\in\nu(v)\} = \{\J|_{\Si} \ \mid \ \J\models\bigcup_{u\in V}\O_u\}$\footnote{and hence, $(\N,v)\models C\dleq D$ iff $\bigcup_{u\in V}\O_u\models C\dleq D$, for any concept inclusion $C\dleq D$ in signature $\Si$};

\item - for any model $\I\in\nu(v)$, there exists an interpretation function $f$ of $\N$ such that for all $u,w\in V$, $f(u)=\{\I_u\}$ is a singleton, $\I_v=_\Si\I$, and $\I_u=_\si\I_w$, for $\si=\sig(\O_u)\cap\sig(O_w)$.
\end{trivlist}
\end{proposition}

\begin{proof} 
$(\Rightarrow):$ Let $\I\in\nu(v)$ and $\J$ be a model of the union of label ontologies of $\N$ such that $\I=_\Si\J$. Then the map $f:V\mapsto \{\J\}$ is the required interpretation of $\N$.

$(\Leftarrow):$ The containment $\supseteq$ follows from Lemma \ref{Lem_ModeloftheUnionIsAgreed}, so we need to prove the reverse inclusion. For a model $\I\in\nu(v)$, let $f$ be the corresponding interpretation function of $\N$. By the conditions in the definition of $f$, there exists an interpretation $\J$ such that for all $u\in V$,  $\J=_{\sig(\O_u)}\I_u$. We have $\J=_\Si\I$ and $\J\models \bigcup_{u\in V}\O_u$, so $\J$ is the required model corresponding to $\I$. 
\end{proof}

%\Todo{Simplify the next Proposition to tree-shaped networks, or shorten the proof}
\Todo{Strengthen the proposition below, if possible}
 
\begin{proposition}[Suffic. Condition for Reduction to Union]\label{Prop_Correspondence2Union}
Let $\N$ be a $v$-pointed ontology network having the following properties:
\begin{itemize}
\item any $u\in V$ has at most one parent, i.e. a node $w\in V$ such that $\langle w\Si v\rangle$, for some signature $\Si$;

\item for all distinct nodes $v_1,v_2\in V$, if $\sig(\O_{v_1})\cap \sig(\O_{v_2})=\si$, for a signature $\si\neq\varnothing$, then either there exists a node $u$ in $\N$ such that both, $v_1,v_2$, are $\si$-reachable from $u$, or for some $i=1,2$, $v_{3-i}$ is $\si$-reachable from $v_{i}$ and the labels of edges in every path from $v_{i}$ to $v_{3-i}$ contain $\si$.
\end{itemize}

Then $\N$ satisfies the second condition in Proposition \ref{Prop_CharacterizationReduction2Union}.
\end{proposition}

%Note that the first condition of the proposition means that the network $\N$ is ``almost tree-shaped'' and may have cycles with nodes only from the same branch. \medskip

\begin{proof} 
Let $\N'=(V',E',\tau')$ be a tree unfolding of $\N$. For a node $u=[v_0,\ldots v_k]\in V'$, denote $\lceil u\rceil=v_k$. The claim is proved if we show that for any model $\I\in\nu(v)$, there exists an interpretation function $f$ of $\N'$ satisfying the second condition in Proposition \ref{Prop_CharacterizationReduction2Union} such that $f(u)=f(w)$, for any $u,w\in V'$ such that $\lceil u\rceil = \lceil w\rceil$. Given $\I\in\nu(v)$, we prove this by induction for a network $\N_m=(V_m,E_m,\tau_m)$ obtained by restricting $\N'$ onto the nodes reachable from $[v]$ in $\leqslant m$ steps. For each $m$, we define a map $f_{m}$, the interpretation function of $\N_{m}$ satisfying the above mentioned properties. Let us call a node $u$ of $\N_m$ \textit{fresh} if there does not exist a node $w$ in $\N_{m}$ such that $\lceil u\rceil = \lceil w\rceil$. By the first condition of the proposition, a node $u$ of $\N_m$ is fresh iff there is no such node $w$ in $\N_{m-1}$.

The claim is obvious for $m=0$, so let us assume that $m\geqslant 1$. Since $\I\in\nu(v)$, there is an interpretation $g$ of $N'$ such that $g(u)$ is a singleton, for each $u\in V_m$, and $g([v])=\{\I\}$. We show that $g$ can be used to obtain the interpretation function $f_m$. Let $u_1,\ldots , u_k$ be the nodes reachable from $[v]$ in $m$ steps. Let us define $f_m$ as the map which coincides with $f_{m-1}$ on the nodes of $\N_{m-1}$ and is defined on $u_1,\ldots , u_k$ as follows: for $i=1,\ldots k$, if $u_i$ is fresh, then $f_m(u_i)=g(u_i)$ and otherwise $f_m(u_i)=f_{m-1}(u)$. 

Let us verify that $f_m$ defined in this way is an interpretation of $\N_m$. For a node $u\in V_m$, denote by $\I_u^m$ the model given by $f_m(u)$. For $i=1,\ldots k$, let $w_i$ be the parent nodes of $u_i$ and $\Si_i$ be the signatures such that $\langle w_i\Si_i u_i\rangle\in\N_m$. It suffices to verify that $\I^m_{u_i}=_{\Si_i}\I^m_{w_i}$, $i=1,\ldots k$. Let us fix an arbitrary $i$. It follows from the first condition of the proposition that for any nodes $v_1,v_2$ of $\N$, any two distinct paths $p_1=v_1,u_1,\ldots ,u_k,v_2$ and $p_2=v_1,w_1,\ldots ,w_l,v_2$ from $v_1$ to $v_2$ consist of the same nodes and for some $i=1,2$, $p_i$ is obtained from $p_{3-i}$ by replacing some nodes $x_1,\ldots, x_m$, for $m\geqslant 1$, with cycles $x_j,y_{1},\ldots y_{n_j}, x_j$, for $j=1,\ldots , m$ and $n_j\geqslant 0$. Therefore, if $u_i$ is fresh, then so is $w_i$, we have $f_m(w_i)=g(w_i)$ by induction and thus, $\I^m_{u_i}=_{\Si_i}\I^m_{w_i}$. If both $u_i$ and $w_i$ are non-fresh, then there exist nodes $v_u,v_w$ in $\N_{m-1}$ such that $\lceil v_u\rceil = \lceil u_i\rceil$, $\lceil v_w\rceil = \lceil w_i\rceil$, and $\langle v_w\Si_i v_u\rangle\in\N_{m-1}$. By the induction assumption, it holds that $\I^{m-1}_{v_w}=_{\Si_i}\I^{m-1}_{v_u}$, and by the definition of $f_m$, we have $f_m(w_i)=f_{m-1}(v_w)$, $f_m(u_i)=f_{m-1}(v_u)$. Finally, if $u_i$ is non-fresh and $w_i$ is fresh, then these nodes belong to a cycle, i.e. there is a path from $u_i$ to $w_i$. Hence, by the second condition of the proposition, there exists a node $v_u$ in $\N_{m-1}$ such that $\lceil v_u\rceil = \lceil u_i\rceil$, $f_m({v_u})=f_m(u_i)$, and $w_i$ is $\Si_i$-reachable from $v_u$. Then, by the induction assumption, we have $f_m({v_u})=f_m(w_i)$, which means that $f_m({w_u})=f_m(u_i)$. 

It remains to show that for any nodes $v_1,v_2$ of $\N_m$ it holds $\I^m_{v_1}=_\si\I^m_{v_2}$, whenever $\si=\sig(\O_{v_1})\cap\sig(\O_{v_2})$. By the second condition of the proposition, assume that $\lceil v_2\rceil$ is $\si$-reachable from $\lceil v_1\rceil$. Then $v_i$ is $\si$-reachable from $v_{3-i}$ in $\N_m$, for $i=1,2$, and since $f_m$ is an interpretation of $\N_m$, it follows that $\I^m_{v_1}=_\si\I^m_{v_2}$. Now assume there is a node $u$ in $\N$ such that $\lceil v_1\rceil$ and $\lceil v_2\rceil$ are $\si$-reachable from $u$. Then $u$ must be reachable from $v$ in $<m$ steps, because otherwise for at least one of the nodes $\lceil v_1\rceil$, $\lceil v_2\rceil$ the first condition of the proposition is violated. Hence, there exists a node $v_u$ in $\N_{m-1}$ such that $\lceil v_u\rceil=u$ and both $v_1,v_2$ are $\si$-reachable from $v_u$. Since $f_m$ is an interpretation of $\N_m$, it follows that $\I^m_{v_i}=_\si\I^m_{v_u}$, for $i=1,2$.  
\end{proof}

\smallskip

Let us compete the proof of Theorem \ref{Teo_Reduction2Union}. The entailment $(\N,v)\models C \dleq D$ means that $\I\models C\dleq D$, for any model $\I\in\nu(v)$. By Proposition \ref{Prop_RelationshipSemanticsTreeUnfolding}, we have $\I\in\nu(v)$ iff $\I\in\nu({[}v{]})$, hence $(\N,v)\models C \dleq D$ iff $(\N_t,{[}v{]})\models C \dleq D$. Let us show that the network $\N_t$ satisfies the both conditions in Proposition \ref{Prop_Correspondence2Union}, then by Proposition \ref{Prop_CharacterizationReduction2Union} the statement of the theorem will follow. The first condition is trivially satisfied, since $\N_t$ is tree-shaped. For the second condition, assume that  $\sig(\O_u)\cap\sig(\O_w)=\si^*$, for some distinct nodes $u,w\in V_t$ and a signature $\si^*\neq\varnothing$ and let us show that conditions 1.2, 1.3, and 2 in the definition of tree unfolding yield the required property. Let $\si$ be the preimage of $\si^*$ wrt $\pi_u$. By condition 1.3, $u,w$ have the form $u={[}u_0,\ldots, u_{k}, \ldots , u_n{]}$, $w={[}w_0,\ldots, w_{k},\ldots , w_m{]}$. It follows from conditions 1.2 and 1.3 that $\pi_x|_\si = \pi_y|_\si$ for all nodes $x,y\in V_t$ such that $x={[}u_0,\ldots, u_{i}{]}$ and $y={[}u_0,\ldots, u_{i+1}{]}$, $k\leqslant i < n$, or $x={[}w_0,\ldots, w_{j}{]}$ and $y={[}w_0,\ldots, w_{j+1}{]}$, $k\leqslant j < m$. Thus, for the node $z={[}u_0,\ldots, u_{k}{]}$ we have $\pi_z|_\si = \pi_u|_\si = \pi_w|_\si$. By condition 2, if $z = u$ (or $z = w$, respectively), then $w$ is $\si^*$-reachable from $u$ (or vice versa). Otherwise, if $z\neq u$ and $z\neq w$, then both, $u$ and $w$ are $\si^*$-reachable from $z$. Theorem \ref{Teo_Reduction2Union} is proved. 
\end{proof}

\begin{corollary}[Upper Bound for General Networks]\label{Cor_UpperBoundGeneral}
Given a finite $v$-pointed ontology network $\N=(V,E,\tau)$, the set of concept inclusions $C\dleq D$ in signature $\sig(\O_v)$ such that $(\N,v)\models C\dleq D$ is recursively enumerable.
\end{corollary}

\Todo{Finiteness is stressed, since we allow networks to be infinite in Definition \ref{De_TreeUnfolding}}

\begin{proof} 
Let $\N_t=(V_t,E_t,\tau_t)$ be a renamed unfolding of $\N$. By Theorem \ref{Teo_Reduction2Union}, $(\N,v)\models C\dleq D$ is equivalent to $\bigcup_{u\in V_t}\O_u\models C\dleq D$ which, by compactness, holds iff there is a finite subset of nodes $W\subseteq V_t$ such that $\bigcup_{w\in W}\O_w\models C\dleq D$. Thus, the proposition is proved if we show that $\N_t$ is recursively enumerable, i.e. the sets $(V_t,E_t)$ and the graph of the labelling function $\tau_t$ (considered as a set of pairs) are enumerable. 

We describe an algorithmic procedure which for a natural number $n\geqslant 0$ gives an ontology network $\N_t^n=(V_t^n,E_t^n,\tau_t^n)$, with the set $V_t^n$ consisting of sequences ${[}v_0\ldots v _n{]}$ of nodes from $\N$, such that $\N_t^n$ satisfies all the conditions of tree unfolding for $\N$. 

It is clear from Definition \ref{De_TreeUnfolding} that there is a recursive procedure to construct the graph $(V_t^n,E_t^n)$ for a given $n$, so we may assume that this graph is given and our goal now is to define the labelling function $\tau_t^n$. It suffices to define a permutation function $\pi_u$ for each node $u\in V_t^n$; we give a definition by simultaneous induction on $k$, $0\leqslant k \leqslant n$ and the number $m$ of nodes reachable from ${[}v_0{]}$ in $k$ steps. 

For $k=m=0$, we define $\tau_t^n({[}v_0{]})=\tau(v)$ and set $\pi_{{[}v_0{]}}$ to be the identity function on $\Nc\cup\Nr$.

For $k,m\geqslant 1$, let $\Gamma$ be the union of signatures of label ontologies for all those nodes $u$ for which $\tau_t$ is already defined. Let $w={[}v_0,\ldots v_k{]}$ be a $m$-th node reachable from ${[}v_0{]}$ in $k$ steps and $u={[}v_0,\ldots v_{k-1}{]}$ be the parent of $w$ wrt an edge $e$. For some signature $\Si$, it holds $\langle v_{k-1}\Si v_k\rangle\in\N$; denote $\Delta=\sig(\O_{v_k})\setminus\Si$. Then $\pi_w$ is defined as a permutation on $\Nc\cup\Nr$ mapping $\Nc$ to $\Nc$ and $\Nr$ to $\Nr$ and having the following properties:

\begin{itemize} 
\item[a.] $\pi_w|_\Si=\pi_u|_\Si$; 
\item[b.] for all $s\in\Delta$, it holds $\pi_w(s)\not\in\Gamma$ and $\pi_w(\pi_w(s))=s$ (denote the image of $\Delta$ wrt $\pi_w$ as $\Delta^*$);
%\item for all $s\in\Delta^*$, it holds $\pi_w(s)=\pi^-1_w(s)$.
\item[c.] $\pi_w$ is the identity function on $(\Nc\cup\Nr)\setminus(\sig(\O_{v_k})\cup\Delta^*)$.
\end{itemize}

Now, we define the label $\tau^n_t(w)$ as  a ``copy'' of ontology $\tau_{w})$ wrt the renaming given by $\pi_w$ and set $\tau^n_t(e)=\{\pi_u(s) \mid s\in\Si\}$. 

Since the alphabet $\Nc\cup\Nr$ is recursively enumerable, it follows that there is an algorithmic procedure which for a given tree $(V_t^n,E_t^n)$ enumerates the graph of the labelling function $\tau_t^n$. Note that $\tau_t^n$ satisfies all the conditions in Definition \ref{De_TreeUnfolding} of tree unfolding: all the conditions except 1.3 are immediate to verify, while 1.3 follows by induction from points a,b in the definition of $\tau_t^n$. \end{proof}

\begin{corollary}[Upper Bound for Acyclic Networks]\label{Cor_UpperBoundAcyclic}
The entailment in acyclic $\LL$-ontology networks is decidable with an exponential overhead wrt the complexity of entailment in the logic $\LL$. It is in $\EXPTIME$ for $\EL$, $\DEXPTIME$
for $\ALC$ and $\SHIQ$, $\NDEXPTIME$ for $\SHOIF$ and $\SHOIQ$, $\TEXPTIME$
for $\SRIQ$, and $\NTEXPTIME$ for $\SROIQ$.
\end{corollary}

\begin{proof}  
Given Lemma \ref{Lem_Reduction2PointedNetworks}, let $\N$ be a finite $v$-pointed acyclic ontology network. It follows from Theorem \ref{Teo_Reduction2Union} and the construction in the proof of Proposition \ref{Prop_ELUndec_UpperBound} that there exists a procedure which works in time exponential in the size of $\N$ and outputs an ontology $\O$ (of size exponential wrt $\N$) such that $(\N,v)\models C\dleq D$ iff $\O\models C\dleq D$, for any concept inclusion $C\dleq D$ in signature $\sig(\O_v)$. As entailment is tractable in the logic $\EL$, we obtain the required statement. 
\end{proof} 

%----------------------------------------------------------------------%

%\subsection{Avoiding Complexity Blowup for Special Types of Networks}
%\label{sec:membership:special:networks}

%----------------------------------------------------------------------%

\begin{corollary}[Tractability in Tree-Shaped Networks]\label{Cor_TractabilityTreeShaped}
Given a $v$-pointed tree-shaped ontology network $\N$, it can be decided in polytime whether $(\N,v)\models C\dleq D$, for a concept inclusion $C\dleq D$ in signature $\sig(\O_v)$.
\end{corollary}

\begin{proof} 
Since $\N$ is tree-shaped, it follows from the construction in the proof of Proposition \ref{Prop_ELUndec_UpperBound} that there is an algorithmic procedure which works in time polynomial in the size of $\N$ and outputs a renamed unfolding $\N_t=(V_t,E_t,\tau_t)$ of $\N$. By Theorem \ref{Teo_Reduction2Union}, $(\N,v)\models C\dleq D$ is equivalent to $\bigcup_{u\in V_t}\O_u\models C\dleq D$, which yields the required statement, since entailment is tractable in $\EL$. 
\end{proof}

%----------------------------------------------------------------------%

%\subsection{Upper Bound for $\P$}
%\label{sec:upper:complexity:propositional}

%----------------------------------------------------------------------%

We show that the complexity of reasoning in $\LL$-ontology networks drops, when $\LL$ is either of the role-free DLs, $\P$ or $\H$. Since the DL $\H$ is a fragment of $\P$, we formulate our claims only for the latter logic.  

%In this section, we prove that reasoning in propositional DL
%$\P$ is in $\EXPTIME$ for cyclic networks and in $\PSPACE$ for acyclic networks.
%Since the DL $\H$ is a fragment of the DL $\P$, the
%same complexity bounds hold for $\H$ as well.

\begin{lemma}[Upper Bound for General $\P$-Networks]
\label{lemma:P:cyclic:membership}
The entailment problem in $\P$-ontology networks is in $\EXPTIME$. 
\end{lemma}

\begin{proof} 
Let $\N=(V,E,\tau)$ be a $\P$-ontology network. By Lemma \ref{Lem_Reduction2PointedNetworks}, we may assume that $\N$ is $v$-pointed, for some node $v$. For a concept inclusion $C\dleq D$, we have $(\N,v)\not\models C\dleq D$ iff there is a model $\I\in\nu(v)$ such that $\I\not\models C\dleq D$. We show that there is an $\EXPTIME$-procedure to find such a countermodel for $C\dleq D$. Let $\I\in\nu(v)$ be a model having a domain element $a$ such that $a\in C^\I$ and $a\not\in D^\I$. For every model $\J\in\nu(u)$, $u\in V$, let $\J_a$ be the interpretation having domain $\{a\}$ such that for each $A\in\Nc$, it holds $A^{\J_a}=A^{\J}\cap\{a\}$. Clearly, we have  $\J_a\models\O_u$, since $\O_u$ is a $\P$-ontology. Consider the map $f$ defined as $f(u)=\{\J_a \mid \J\in\nu(v)\}$, for all $u\in V$. One can readily verify that this is an interpretation function for $\N$ and hence, we conclude that $(\N,v)\not\models C\dleq D$ iff there is a one-element countermodel $\I\in\nu(v)$ for $C\dleq D$. The number of one-element models for each node $u\in V$ is exponential in the size of $\O_u$ and for a given one-element model $\I\models\O_u$, the condition $\I\in\nu(u)$ can be verified in time exponential in the size of the network $\N$, which proves our claim. 
\end{proof}

\begin{lemma}[Upper Bound for Acyclic $\P$-Networks]
\label{lemma:P:acyclic:membership}
The entailment problem in acyclic $\P$-ontology networks is in $\PSPACE$. 
\end{lemma}

\begin{proof}
Let $\N$ be an acyclic $\P$-ontology network. By Lemma \ref{Lem_Reduction2PointedNetworks}, we may assume that $\N$ is $v$-pointed for some node $v$. It follows from the argument in the proof of Lemma \ref{lemma:P:cyclic:membership} that $(\N,v)\not\models C\dleq D$ iff there is a one-element counter model $\I\in\nu(v)$ for $C\dleq D$. Given a node $u$ and and a one-element model $\I$ of $\O_u$, it can be decided in $\NPSPACE$ whether $\I$ satisfies all importing relations defined for $\O_u$. Since the length of every path in $\N$ is bounded by the size of $\N$ and guessing can be done independently in every path, it follows that the condition $\I\in\nu(u)$ can be verified in $\NPSPACE$. Then by Savitch's theorem we obtain the required statement. 
%exists an interpretation function $f$ of $\N'$ such that $f(u)=\{\I_u\}$ is a singleton, $\I_u$ is a one-element model, for each $u\in V'$, and  $\I_{[v]}\not\models C\dleq D$.  Since $\N$ is acyclic, the height of the tree $\N'$ is bounded by the number of nodes in $\N$. On the other hand, we may assume that for all $u\in V'$, the ``size'' of a one-element model $\I_u$ of $\O_u$ is bounded by the size of $\O_u$. This yields that, given a concept inclusion $C\dleq D$, there is a PSPACE procedure that guesses recursively for every node $w\in V'$ reachable from $[v]$ in $m\geqslant 0$ steps a one-element model $\I_w\models\O_w$ such that $\I_{[v]}\not\models C\dleq D$ and $\I_u=_\Si\I_w$, for a (single) parent node $u$ of $w$ in $\N'$ such that $\langle u\Si w\rangle\in\N'$, for a signature $\Si$. 
\end{proof}

}

%%====================================================================%%

\section{Conclusions}\label{Sect_Conclusions}

We have introduced a new mechanism for ontology integration which is based on semantic import relations between ontologies and is a generalization of the standard OWL importing. In order to import an external ontology $\O$ into a local one, one has to specify an import relation, which defines a set of symbols, whose semantics should be borrowed from $\O$. The significant feature of the proposed mechanism, which comes natural in complex ontology integration scenarios, is that every ontology has its own view on ontologies it refines and the views on the same ontology are independent unless coordinated by import relations. We have shown that this feature can lead to an exponential increase of the time complexity of reasoning over ontologies combined with acyclic import relations. Intuitively, this is because one has to consider multiple views of the same ontology, each of which gives a different ontology. When cyclic importing is allowed, the complexity jumps to undecidability, even if every ontology in a combination is given in the DL $\EL$. Similarly, this is because one has to consider infinitely many views on the same ontology. These complexity results are shown for situations when the imported symbols include roles. It is natural to ask whether the complexity drops when the imported symbols are concept names. The second parameter which may influence the complexity of reasoning is the semantics which is `imported'. In the proposed mechanism, importing the semantics of symbols is implemented via agreement of models of ontologies. One can consider refinements of this mechanism, e.g., by carefully selecting the classes of models of ontologies which must be agreed. The third way to decrease the complexity is to restrict the language in which ontologies are formulated. We conjecture that reasoning with cyclic imports is decidable for ontologies formulated in the family of DL-Lite. 

%%====================================================================%%

\bibliography{references}

\begin{thebibliography}{}

\bibitem[\protect\citeauthoryear{Bao \bgroup \em et al.\egroup
  }{2009}]{PDL-Chapter}
Jie Bao, George Voutsadakis, Giora Slutzki, and Vasant Honavar.
\newblock Package-based description logics.
\newblock In {\em Modular Ontologies}, pages 349--371. 2009.

\bibitem[\protect\citeauthoryear{Borgida and Serafini}{2003}]{DDL-Journal}
Alexander Borgida and Luciano Serafini.
\newblock Distributed description logics: Assimilating information from peer
  sources.
\newblock {\em J. Data Semantics}, pages 153--184, 2003.

\bibitem[\protect\citeauthoryear{{Cuenca Grau} \bgroup \em et al.\egroup
  }{2008}]{GHM+:08:OWL}
Bernardo {Cuenca Grau}, Ian Horrocks, Boris Motik, Bijan Parsia, Peter~F.
  Patel-Schneider, and Ulrike Sattler.
\newblock {OWL} 2: The next step for {OWL}.
\newblock {\em J. Web Sem.}, 6(4):309--322, 2008.

\bibitem[\protect\citeauthoryear{Euzenat \bgroup \em et al.\egroup
  }{2007}]{AlignmentBasedModules}
J\'{e}r\^{o}me Euzenat, Antoine Zimmermann, and Frederico Luiz~Goncalves
  de~Freitas.
\newblock Alignment-based modules for encapsulating ontologies.
\newblock In {\em Proceedings of the 2nd International Workshop on Modular
  Ontologies, WoMO 2007, Whistler, Canada, October 28, 2007}, 2007.

\bibitem[\protect\citeauthoryear{Grau and Motik}{2012}]{ImportByQuery}
Bernardo~Cuenca Grau and Boris Motik.
\newblock Reasoning over ontologies with hidden content: The import-by-query
  approach.
\newblock {\em J. Artif. Intell. Res. (JAIR)}, pages 197--255, 2012.

\bibitem[\protect\citeauthoryear{Grau \bgroup \em et al.\egroup
  }{2009}]{EConnections-Chapter}
Bernardo~Cuenca Grau, Bijan Parsia, and Evren Sirin.
\newblock Ontology integration using epsilon-connections.
\newblock In {\em Modular Ontologies}, pages 293--320. 2009.

\bibitem[\protect\citeauthoryear{Homola and
  Serafini}{2010}]{TowardsFormalComparison}
Martin Homola and Luciano Serafini.
\newblock Towards formal comparison of ontology linking, mapping and importing.
\newblock In {\em Description Logics}, 2010.

\bibitem[\protect\citeauthoryear{Kazakov}{2008}]{Kazakov:08:RIQ:SROIQ}
Yevgeny Kazakov.
\newblock {$\mathcal{RIQ}$} and {$\mathcal{SROIQ}$} are harder than
  {$\mathcal{SHOIQ}$}.
\newblock In {\em Proc.\ 11th Int.\ Conf.\ on Principles of Knowledge
  Representation and Reasoning (KR'08)}, pages 274--284. AAAI Press, 2008.

\bibitem[\protect\citeauthoryear{Pan \bgroup \em et al.\egroup
  }{2006}]{Pan-SemanticImport}
Jeff~Z. Pan, Luciano Serafini, and Yuting Zhao.
\newblock Semantic import: An approach for partial ontology reuse.
\newblock In Peter Haase, Vasant Honavar, Oliver Kutz, York Sure, and Andrei
  Tamilin, editors, {\em WoMO}, volume 232 of {\em CEUR Workshop Proceedings}.
  CEUR-WS.org, 2006.

\bibitem[\protect\citeauthoryear{Shvaiko and
  Euzenat}{2013}]{OntologyMatchingSurvey}
Pavel Shvaiko and Jerome Euzenat.
\newblock Ontology matching: State of the art and future challenges.
\newblock {\em IEEE Trans. on Knowl. and Data Eng.}, 25(1):158--176, January
  2013.

\bibitem[\protect\citeauthoryear{Smith \bgroup \em et al.\egroup
  }{2007}]{SAR+:07:OBO}
Barry Smith, Michael Ashburner, Cornelius Rosse, Jonathan Bard, William Bug,
  Werner Ceusters, Louis~J Goldberg, Karen Eilbeck, Amelia Ireland,
  Christopher~J Mungall, Neocles Leontis, Philippe Rocca-Serra, Alan
  Ruttenberg, Susanna-Assunta Sansone, Richard~H Scheuermann, Nigam Shah,
  Patricia~L Whetzel, and Suzanna Lewis.
\newblock The {OBO} foundry: coordinated evolution of ontologies to support
  biomedical data integration.
\newblock {\em Nat Biotech}, 25(11):1251--1255, 2007.

\end{thebibliography}
\bibliographystyle{named}

%%====================================================================%%

\newpage

\begin{center}
{\LARGE \textbf{Appendix}}
\end{center}

\bigskip\bigskip

\noindent\textbf{{\Large Proofs for Section \ref{Sect_Hardness}}}

\medskip

\noindent\textbf{Theorem \ref{Teo_Hardness_EL-acyclic}.} \textit{Entailment in acyclic $\EL$-ontology networks is $\EXPTIME$-hard.}

\medskip

Let $M=\langle Q,{\mathcal{A}},\delta \rangle$ be a TM and $n=\IEXP{m}$ an exponential, for $m\geqslant 0$. Consider an ontology $\O$ defined for $M$ and $n$ by axioms (\ref{Eq_ELExptimeRchain})-(\ref{Eq_ELExptime_PropagateH}) below:

\begin{align} \tag{\ref{Eq_ELExptimeRchain}}
A \dleq \ex{r^{n\cdot (2n+3)}}.(\qo\dcap\ex{(r,\b)^{2n+2}} )
\end{align}
where $A\not\in Q\cup\mathcal{A}$.

\begin{align} \tag{\ref{Eq_ELExptimeTransition}}
\ex{r^{2n}}(X \dcap \ex{r}.(Y \dcap \ex{r}.(U \dcap \ex{r}.Z))) \dleq W, 
\end{align}
for all $X,Y,U,Z,W\in{Q\cup\mathcal{A}}$ such that $XYUZ\overset{\delta'}{\mapsto} W$.

\begin{align} \tag{\ref{Eq_ELExptime_PropagateH}}
\ex{r}.\qacc \dleq H, \ \ \ex{r}.H \dleq H  
\end{align}

\medskip

\begin{lemma}
$M$ accepts the empty word in $n$ steps iff $\O\models A\dleq H$. 
\end{lemma}

\begin{proof}
For the purpose of this proof,  we let \emph{configuration} of $M$ be a word of length $4n+3$ in the alphabet $Q\cup \mathcal{A}$. Then, given a configuration $\c$, the notion of successor configuration is naturally induced by $\delta'$. %For a word $w$ of length $k$ in the alphabet $Q\cup \mathcal{A}$ and $1\leqslant j \leqslant k$, let $w[j]$ denote the $j$-th symbol in $w$. 
Let us call the word of the form

\begin{equation*}
\c_{0} = \word{\b\ldots\b}{2n}{\qo}{\b\ldots\b}{2n+2}
\end{equation*} 
\emph{initial configuration} of $M$.

Then $M$ accepts the empty word in $n$ steps iff there is a sequence $\c_0,\ldots ,\c_{n}$ of configurations in the above sense such that for all $0\leqslant i < n$, $\c_{i+1}$ is a successor of $\c_{i}$ and $\qacc$ is the state symbol in $\c_{n}$.  

Let $\I=(\Delta,\cdot^\I)$ be a model of ontology $\O$ with a domain element $a\in A^\I$. Then by axiom (\ref{Eq_ELExptimeRchain}), there is a $r$-chain outgoing from $a$ which contains $n+1$ consequent segments $s_0,\ldots ,s_{n}$ of length $2n+3$. We consider each segment as a linearly ordered set of elements from $\Delta$ and for $1\leqslant j\leqslant 2n+3$, we denote by $s_i[j]$ the $j$-th element of $s_i$. Given a word $w$ of length $2n+3$, we say that segment $s_i$ \emph{represents} $w$ if $s_i[j]\in w[j]^\I$, for all $1\leqslant j \leqslant 2n+3$.  We assume the following enumeration of segments in the $r$-chain:

\begin{equation*}
\word{\ldots}{s_{n}}{\ldots}{\ldots}{s_1}\underbrace{\qo\b\ldots\b}_{s_0}
\end{equation*}

\noindent i.e. we let $s_0$ represent a fragment of the initial configuration. \smallskip

We show that for all $0\leqslant i \leqslant n$ and $1\leqslant j \leqslant 2n+3$, it holds $s_i[j]\in(\c_i[2n+j-i])^\I$. Then clearly, for every $0\leqslant i \leqslant n$, there exists $j$ such that $s_i[j]\in \q^\I$, where $\q$ is the state symbol from $\c_i$. In particular, $s_n[j]\in \qacc^\I$, for some $j$, and hence $a\in H^\I$, due to axiom (\ref{Eq_ELExptime_PropagateH}). We use induction on $i$. The case $i=0$ is obvious, so let us assume that the claim holds for $0\leqslant i < n$. Note that configuration $\c_i$ has the form

\begin{align*}
\word{\b\ldots\b}{\geqslant 2n-i}{\ldots\q\ldots}{\b\ldots\b}{\geqslant 2n+2-i}
\end{align*}

thus, by the induction hypothesis, $s_i$ represents a fragment $w_i$ of $\c_i$ having the form

\begin{align*}
\ldots \q \ldots\underbrace{\b\ldots\b}_{\geqslant 2n+2-2i}
\end{align*}
where $w_i[j]=\c_i[2n+j-i]$, for all $1\leqslant j \leqslant 2n+3$. 
Note that since $i<n$, we have $2n+2-2i\geqslant 4$. 

Then by axiom (\ref{Eq_ELExptimeTransition}), $s_{i+1}$ must represent a word $w_{i+1}$ of the form 

\begin{align*}
\word{\ldots}{3}{}{\ldots \ldots \underbrace{\b\ldots\ldots\ldots\b}_{\geqslant 2n+2-2(i+1)\geqslant 2}}{2n}
\end{align*}

where for $4\leqslant j \leqslant 2n+3$ and $k=j-3$, it holds 

\begin{align*}
w_i[k]\ w_i[k+1] \ w_i[k+2] \ w_i[k+3]\overset{\delta'}{\mapsto}w_{i+1}[j].
\end{align*}

Since we have $w_i[j]=\c_i[2n+j-i]$, for $1\leqslant j \leqslant 2n+3$, it follows by definition of $\delta'$ that $w_{i+1}[j]=\c_{i+1}[2n+j-(i+1)]$, for $4\leqslant j \leqslant 2n+3$. It remains to show how the first three symbols in $w_{i+1}$ are defined. 

By axiom (\ref{Eq_ELExptimeTransition}), it holds

\begin{align*}
w_{i+1}[2n+3] \ w_{i}[1] \  w_{i}[2] \ w_{i}[3] \overset{\delta'}{\mapsto} w_{i+1}[3] \\
w_{i+1}[2n+2] \ w_{i+1}[2n+3] \ w_{i}[1] \ w_{i}[2] \overset{\delta'}{\mapsto} w_{i+1}[2] \\
w_{i+1}[2n+1] \ w_{i+1}[2n+2] \ w_{i+1}[2n+3] \ w_{i}[1] \overset{\delta'}{\mapsto} w_{i+1}[1]
\end{align*} \smallskip

and we have $w_{i+1}[2n+2]=w_{i+1}[2n+3]=\b$. By the induction hypothesis, $w_i[k]=\c_i[2n+k-i]$, for $k=1,2,3$, hence, $\c_i[2n-i]=\c_i[2n-i-1]=\b$ and therefore, $w_{i+1}[j] = \c_i[2n+j-(i+1)]$, for $j=2,3$. Note that at most one of $w_{i+1}[2n+1]$, $w_{i}[1]$ is a state symbol, because otherwise $n=1$, which is not the case, since we have $n=\IEXP{m}$ and $m\geqslant 1$. Hence, we conclude that $w_{i+1}[1] = \c_i[2n+1-(i+1)]$.

\medskip

For the `if' direction, suppose $M$ does not accept the empty word in $n$ steps. Consider an interpretation $\I=(\Delta, \cdot^\I)$ having domain $\Delta=\{x_1,\ldots x_k\}$, for $k=(n+1)\cdot(2n+3)$, such that:

\begin{itemize}
\item $r^\I=\{\langle x_i,x_{i+1} \rangle\}_{1\leqslant i < k}$;
\item $A^\I=\{x_1\}$ and $H^\I=\qacc^\I=\varnothing$;
\item $\qo^\I=\{x_{p}\}$ and $x_l\in\b^\I$, for $p=n\cdot(2n+3)+1 \ $ and $p+1\leqslant l \leqslant k$;
\item for any $1\leqslant i \leqslant n\cdot(2n+3)$ and $W\in Q\cup\mathcal{A}$, it holds $x_i\in W^\I$ iff there exist $V_0,\ldots ,V_3\in Q\cup\mathcal{A}$ such that $x_{i+2n+j}\in V_j$, for $0\leqslant j \leqslant 3$, and $V_0 \ V_1 \ V_2 \ V_3\overset{\delta'}{\mapsto} W$. 
\end{itemize}
%Clearly, $\I$ is a model of axioms (\ref{Eq_ELExptimeRchain}),(\ref{Eq_ELExptime_PropagateH}). 
By using arguments from the proof of the `only if' direction, one can verify that $\I$ is well defined and $\I$ is a model of ontology $\O$ such that $\I\not\models A\dleq H$.
\end{proof}

To complete the proof of the theorem let us show that ontology $\O$ is expressible by an acyclic $\EL$-ontology network of size polynomial in $m$. Note that $\O$ contains axioms (\ref{Eq_ELExptimeRchain}), (\ref{Eq_ELExptimeTransition}) with concepts of size exponential in $m$. Consider axiom (\ref{Eq_ELExptimeRchain}) and a concept inclusion $\varphi$ of the form

\begin{equation*}
A\dleq \ex{r^{n\cdot(2n+3)}}.B
\end{equation*}
where $B$ is a concept name. Observe that it is equivalent to

\begin{equation*}
A\dleq \underbrace{\ex{r^p}.\ex{r^p}}_{2~\text{times}}.\underbrace{\ex{r^n}.\ex{r^n}.\ex{r^n}}_{3~\text{times}}.B
\end{equation*}
where $p=\IEXP{2m}$. Consider axiom $\psi$ of the form $A\dleq B$. By iteratively applying Lemma \ref{Lem_Expressibility_ExistExpSubstitution} we obtain that $\psi[B\mapsto \ex{r^p}.\ex{r^p}.B]$ is expressible by an acyclic $\EL$-ontology network of size polynomial in $m$. By repeating this argument we obtain that the same holds for $\varphi$. Further, by Lemma \ref{Lem_Expressibility_SimpleSubstitution} the axiom $\theta=\varphi[B\mapsto \qo \dcap B]$ is   expressible by an acyclic $\EL$-ontology network of size polynomial in $m$. Again, by iteratively applying Lemma \ref{Lem_Expressibility_ExistExpSubstitution} together with Lemma \ref{Lem_Expressibility_SimpleSubstitution} we conclude that $\theta[B\mapsto \exists (r,\b)^{2n+2}]$ is expressible by an acyclic $\EL$-ontology network of size polynomial in $m$ and hence, so is axiom (\ref{Eq_ELExptimeRchain}). The expressibility of axioms of the form (\ref{Eq_ELExptimeTransition}) is shown identically. The remaining axioms of ontology $\O$ are $\EL$-axioms whose size does not depend on $m$. By applying Lemma \ref{Lem_IteratedExpressibility} we obtain that there exists an acyclic $\EL$-ontology network  $\N$ of size polynomial in $m$ and an ontology $\O_\N$ such that $\O$ is $(\N,\O_\N)$-expressible and thus, it holds $\O_\N\models_\N A\dleq H$ iff $M$ accepts the empty word in $\IEXP{m}$ steps. Theorem \ref{Teo_Hardness_EL-acyclic} is proved.

%Note that the axioms of $\O$ are defined using nested concepts of the form $\exists (r,C)^{an^2+bn+c}.D$, for $a,b,c,k\geqslant 0$, which can be represented as

%\begin{align*}
%\underbrace{\exists (r,C)^{p}. \ .. \ \exists (r,C)^{p}.}_{a~\text{times}}\underbrace{\exists (r,C)^{n}. \ .. \ \exists (r,C)^{n}.}_{b~\text{times}}\exists(r,C)^{c}.D
%\end{align*}
%where $p=\IEXP{2m}$. Then by structural transformation, there exists a $\EL$-ontology $\O'$ such that $\O'=_{\sig(\O)}\O$, the number of axioms in $\O'$ is linearly bigger than that in $\O$, and in which every axiom containing a nested concept with role $r$ has the form $Z\dleq F$ or $F\dleq Z$, where $F=\ex(r,C)^{f(m)}.E$, $Z,C,E\in\Nc$, $f(m)=\IEXP{km}$, and $k\leqslant 2$. Then by Lemmas \ref{Lem_IteratedExpressibility} and \ref{Lem_Expressibility_EL+Exp}, $\O'$ is expressible by a $\EL$-ontology network of a size polynomial in $m$. Thus, there is a $\EL$-ontology network $\N$ of a polynomial size and an ontology $\O_\N$ such that $\O_\N\models_\N A\dleq H$ iff $M$ accepts the empty word in $\IEXP{m}$ steps.

\bigskip

\noindent\textbf{Theorem \ref{Teo_Hardness_EL-cyclic}.} \textit{Entailment in cyclic $\EL$-ontology networks is $\RE$-hard.}

\medskip

For a TM $M=\langle Q,{\mathcal{A}},\delta \rangle$, we define an infinite ontology $\O$, which contains variants of axioms (\ref{Eq_ELExptimeRchain}-\ref{Eq_ELExptimeTransition}) from Theorem \ref{Teo_Hardness_EL-acyclic} and additional axioms for a correct implementation of transitions of $M$. %The axioms are given using expressions of the form $\exists (r,C)^{ax+b}.D$, for $a\geqslant 1$ and $b\geqslant 0$, which is an abbreviation for $\exists (r,C)^{ax}.\exists(r,C)^b.D$.
Ontology $\O$ consists of the following families of axioms: \vspace{-0.2cm}

\begin{equation}\tag{\ref{Eq_ELUndecStartRchain}}
A\dleq\exists r^k.(\exists v^l.L\dcap\varepsilon\dcap\ex{r}.(\qo\dcap\exists (r,\b)^{2l+2}))
\end{equation}

%\begin{equation}\tag{\ref{Eq_ELUndecRchain}}
%\forall y  \ \ B \dleq \exists v^y.L \dcap \varepsilon \dcap \ex{r}.(\qo\dcap\exists (r,\b)^{2y+2}) 
%\end{equation}
for $A,L,\varepsilon\not\in Q\cup\mathcal{A}$ and all $k,l\geqslant 0$; \smallskip

\begin{equation} \tag{\ref{Eq_ELUndecPropagateLengthMarker}}
\ex{r}.\exists v^k.L \dleq \exists v^k.L, \ \ k\geqslant 0
\end{equation}

\begin{equation} \tag{\ref{Eq_ELUndecPropagateFragmentMarker}}
\exists v^k.L \dcap \exists r^{2k+4}.\varepsilon \dleq \varepsilon, \ \ k\geqslant 0
\end{equation}

\begin{align} \tag{\ref{Eq_ELUndecTransition}}
\exists v^k.L \ & \dcap \\ & \dcap \exists r^{2k+1}. (X \dcap \ex{r}.(Y \dcap \ex{r}.(U \dcap \ex{r}.Z))) \dleq W \nonumber
\end{align} %\nonumber
for $k\geqslant 0$ and all $X,Y,U,Z,W\in Q\cup{\mathcal{A}}$ such that $XYUZ\overset{\delta'}{\mapsto} W$;  

\begin{align} \tag{\ref{Eq_ELUndecEpsilonTransition1}}
\exists v^k. & L \ \dcap  \\ &  \dcap \exists r^{2k}. (\b \dcap \ex{r}.(\varepsilon \dcap \ex{r}.(Y \dcap \ex{r}.(U \dcap \ex{r}.Z)))) \dleq W \nonumber
\end{align} %\nonumber
for $k\geqslant 0$ and all $Y,U,Z,W\in Q\cup{\mathcal{A}}$ such that $\b YUZ\overset{\delta'}{\mapsto} W$; \vspace{-0.2cm}

\begin{align} \tag{\ref{Eq_ELUndecEpsilonTransition2}}
\exists v^k. & L \ \dcap \\ & \dcap \exists r^{2k}. (\b \dcap \ex{r}.(\b \dcap \ex{r}.(\varepsilon \dcap \ex{r}.(U \dcap \ex{r}.Z)))) \dleq W \nonumber
\end{align} %\nonumber
for $k\geqslant 0$ and all $U,Z,W\in Q\cup\mathcal{A}$ such that $\b\b UZ\overset{\delta'}{\mapsto} W$; \vspace{-0.2cm}

\begin{align} \tag{\ref{Eq_ELUndecEpsilonTransition3}}
\exists v^k. & L \ \dcap \\ & \dcap \exists r^{2k}. (X \dcap \ex{r}.(\b \dcap \ex{r}.(\b \dcap \ex{r}.(\varepsilon \dcap \ex{r}.Z)))) \dleq W \nonumber
\end{align} %\nonumber
for $k\geqslant 0$ and all $X,Z,W\in Q\cup\mathcal{A}$ such that $X\b\b Z\overset{\delta'}{\mapsto} W$;

\begin{equation} \tag{\ref{Eq_ELUndecid_PropagateH}}
\qacc \dleq H, \ \ \ex{r}.H \dleq H 
\end{equation}

\medskip

\begin{lemma}
It holds $\O\models A\dleq H$ iff $M$ halts.
\end{lemma}

\medskip

\begin{proof}
Suppose $M$ halts in $n$ steps; w.l.o.g. we assume that $n>1$. Let $\I$ be a model of $\O$ with a domain element $a\in A^\I$. Then by axioms (\ref{Eq_ELUndecStartRchain})-(\ref{Eq_ELUndecTransition}), $\I$ is a model of the concept inclusions: \vspace{-0.2cm}

\begin{align*}%\tag{\ref{Eq_ELUndecConcreteRchain}}
A \dleq \ex{r^{n\cdot (2n+4)}}.(\ex{v^n}.L \dcap \varepsilon \dcap \ex{r}.(\qo\dcap\ex{(r,\b)^{2n+2}}) )
\end{align*}

\begin{align}\label{Eq_ELUndecConcretePropagateLengthMarker}
\ex{r}.\exists v^n.L \dleq \exists v^n.L
\end{align}

\begin{equation}\label{Eq_ELUndecConcretePropagateFragmentMarker}
\exists v^n.L \dcap \exists r^{2n+4}.\varepsilon \dleq \varepsilon
\end{equation}

\begin{align}\label{Eq_ELUndecConcreteTransition}
\exists v^n.L \ \dcap \\ & \dcap \exists r^{2n+1}. (X \dcap \ex{r}.(Y \dcap \ex{r}.(U \dcap \ex{r}.Z))) \dleq W \nonumber
\end{align}
for all $X,Y,U,Z,W\in Q\cup{\mathcal{A}}$ such that $XYUZ\overset{\delta'}{\mapsto} W$;

\begin{align} \label{Eq_ELUndecConcreteEpsilonTransition1}
\exists v^n.L \ & \dcap  \\ &\dcap \exists r^{2n}. (\b \dcap \ex{r}.(\varepsilon \dcap \ex{r}.(Y \dcap \ex{r}.(U \dcap \ex{r}.Z)))) \dleq W \nonumber
\end{align} %\nonumber
for all $Y,U,Z,W\in Q\cup{\mathcal{A}}$ such that $\b YUZ\overset{\delta'}{\mapsto} W$; \vspace{-0.2cm}

\begin{align} \label{Eq_ELUndecConcreteEpsilonTransition2}
\exists v^n.L \ & \dcap \\ &\dcap \exists r^{2n}. (\b \dcap \ex{r}.(\b \dcap \ex{r}.(\varepsilon \dcap \ex{r}.(U \dcap \ex{r}.Z)))) \dleq W \nonumber
\end{align} %\nonumber
for all $U,Z,W\in Q\cup\mathcal{A}$ such that $\b\b UZ\overset{\delta'}{\mapsto} W$; \vspace{-0.2cm}

\begin{align} \label{Eq_ELUndecConcreteEpsilonTransition3}
\exists v^n.L \ & \dcap \\ &\dcap \exists r^{2n}. (X \dcap \ex{r}.(\b \dcap \ex{r}.(\b \dcap \ex{r}.(\varepsilon \dcap \ex{r}.Z)))) \dleq W \nonumber
\end{align} %\nonumber
for all $X,Z,W\in Q\cup\mathcal{A}$ such that $X\b\b Z\overset{\delta'}{\mapsto} W$.

\smallskip

Hence, $\I$ gives a $r$-chain outgoing from $a$, which contains $n+1$ consequent segments $s_0,\ldots s_n$ of length $2n+3$ separated by elements from the interpretation of $\varepsilon$. We use conventions and notations from the `only-if' part of the proof of Theorem \ref{Teo_Hardness_EL-acyclic} and assume the following enumeration of segments in the $r$-chain:

\begin{equation*}
\varepsilon\word{\ldots}{s_{n}}{\varepsilon\ldots\varepsilon}{\ldots}{s_1}\varepsilon\underbrace{\qo\b\ldots\b}_{s_0}
\end{equation*}
i.e. we let $s_0$ represent a fragment of the initial configuration $\c_0$ of $M$. 

If $M$ halts in $n$ steps then there is a sequence $\c_0,\ldots ,\c_{n}$ of configurations such that for all $0\leqslant i < n$, $\c_{i+1}$ is a successor of $\c_{i}$ and $\qacc$ is the state symbol in $\c_{n}$. We show that for all $0\leqslant i \leqslant n$ and $1\leqslant j \leqslant 2n+3$, it holds $s_i[j]\in(\c_i[2n+j-i])^\I$. Then clearly, for every $0\leqslant i \leqslant n$, there exists $j$ such that $s_i[j]\in \q^\I$, where $\q$ is the state symbol from $\c_i$. In particular, $s_n[j]\in \qacc^\I$, for some $j$, and hence $a\in H^\I$, due to axiom (\ref{Eq_ELUndecid_PropagateH}). We use induction on $i$. The case $i=0$ is obvious, so let us assume that the claim holds for $0\leqslant i < n$. By repeating the arguments from the proof of Theorem \ref{Teo_Hardness_EL-acyclic} one can verify that due to axioms (\ref{Eq_ELUndecConcretePropagateLengthMarker})-(\ref{Eq_ELUndecConcreteEpsilonTransition3}), $s_{i+1}$ must represent a word $w_{i+1}$ of the form \vspace{-0.2cm}

\begin{align*}
\word{\ldots}{3}{\ldots\ldots}{\b\ldots\ldots\ldots\b}{\geqslant 2n+2-2(i+1)\geqslant 2}
\end{align*}

where $w_{i+1}[j]=\c_{i+1}[2n+j-(i+1)]$, for all $4\leqslant j \leqslant 2n+3$, and $w_{i+1}[2n+2]=w_{i+1}[2n+3]=\b$. Note that $\I$ is a model of axioms (\ref{Eq_ELUndecEpsilonTransition1})-(\ref{Eq_ELUndecEpsilonTransition3}) for $y=n$ and thus it holds:

\begin{align*}
w_{i+1}[2n+3] \ w_{i}[1] \  w_{i}[2] \ w_{i}[3] \overset{\delta'}{\mapsto} w_{i+1}[3] \\
w_{i+1}[2n+2] \ w_{i+1}[2n+3] \ w_{i}[1] \ w_{i}[2] \overset{\delta'}{\mapsto} w_{i+1}[2] \\
w_{i+1}[2n+1] \ w_{i+1}[2n+2] \ w_{i+1}[2n+3] \ w_{i}[1] \overset{\delta'}{\mapsto} w_{i+1}[1]
\end{align*} \smallskip

By the induction hypothesis, we have $w_i[k]=\c_i[2n+k-i]$, for $k=1,2,3$, so $\c_i[2n-i]=\c_i[2n-i-1]=\b$ and therefore, $w_{i+1}[j] = \c_i[2n+j-(i+1)]$, for $j=2,3$. Note that at most one of $w_{i+1}[2n+1]$, $w_{i}[1]$ is a state symbol, because otherwise $n=1$, which is not the case, since we have assumed $n > 1$. Hence, we conclude that $w_{i+1}[1] = \c_i[2n+1-(i+1)]$.

\medskip
 
For the `if' direction, suppose that $M$ does not halt. Consider interpretation $\I=(\Delta,\cdot^\I)$ having an infinite domain $\Delta$, which is a union of sets $R_{m,n}=\{x_0^{m,n},\ldots , x_{m+2n+3}^{m,n}\}$ and $V_t=\{y_1^t,\ldots ,y_t^t\}$, for all $m,n\geqslant 0$,  $t\geqslant 1$. A set $R_{m,n}$ will be used to define a $r$-chain with a prefix of length $m+1$ representing fragments of consequent configurations of $M$ and a postfix of length $2n+3$ representing a fragment of the initial configuration. The elements of $V_t$ will be used to define a $v$-chain of length $t$ indicating the length of a configuration fragment. More precisely, we define $\I$ as an interpretation satisfying the following properties:

\begin{itemize}
\item there is an element $a\in\Delta$ such that $\{a\}=A^\I$ and $a=x_0^{m,n}$, for all $m,n\geqslant 0$;
% \item $B^\I=\{x_{m+1}^{m,n}\}_{m,n\geqslant 1}$;
\item the sets $R_{m,n}\setminus\{a\}$ and $V_t$, for $m,n\geqslant 0$,  $t\geqslant 1$, are pairwise disjoint;
\item $r^\I=\bigcup\{\langle x_i^{m,n},x_{i+1}^{m,n} \rangle \mid 0\leqslant i < m+2n+3, \ m,n\geqslant 0\}$;
\item $v^\I=\bigcup\{\langle y_i^{t},y_{i+1}^{t} \rangle \mid 1\leqslant i < t, \ t\geqslant 1\}\cup\{\langle x_i^{m,n}, y_1^n\rangle \mid 0\leqslant i \leqslant m,  \ m\geqslant 0, n\geqslant 1\}$;
\item $L^\I=\{a\}\cup\{x^{m,0}_i\}_{0\leqslant i \leqslant m}\cup\{y^t_t\}_{t\geqslant 1}$.
\end{itemize}

Then one can readily verify that $\I$ is a model of axioms (\ref{Eq_ELUndecPropagateLengthMarker}).  Now let us define interpretation of $\varepsilon$ and the alphabet symbols from $Q\cup\mathcal{A}$ as follows. Let $\varepsilon,\qo$, and $\b$ be interpreted in $\I$ as:

\begin{itemize}
\item $\varepsilon^\I=\{x_{m}^{m,n} \mid m,n\geqslant 0\}\cup\{x_i^{m,n} \mid i=m-k(2n+4), \ m,n\geqslant 0, \ k\geqslant 1\}$;
\item $\qo=\{x_i^{m,n} \mid i=m+1, \ m,n\geqslant 0\}$;
\item $x_i^{m,n}\in\b^\I$, for $m+2\leqslant i \leqslant m+2n+3$ and $m,n\geqslant 0$.
\end{itemize}

Then clearly, $\I$ is a model of axioms (\ref{Eq_ELUndecStartRchain})  and (\ref{Eq_ELUndecPropagateFragmentMarker}). 

Now, for $0\leqslant i \leqslant m$, $\ m,n\geqslant 0$, and $W\in Q\cup\mathcal{A}$, set $x_i^{m,n}\in W^\I$ iff there exist $W_0,\ldots ,W_4\in Q\cup\mathcal{A}\cup\{\varepsilon\}$ such that $x_{i+2n+j}^{m,n}\in W_j$, for $0\leqslant j \leqslant 4$, and either of the following holds:

\begin{itemize}
\item $W_1 \ W_2 \ W_3 \ W_4\overset{\delta'}{\mapsto} W$;
\item $W_0=\b$, $W_1=\varepsilon$, and $W_0 \ W_2 \ W_3 \ W_4\overset{\delta'}{\mapsto} W$;
\item $W_0=W_1=\b$, $W_2=\varepsilon$, and $W_0 \ W_1 \ W_3 \ W_4\overset{\delta'}{\mapsto} W$;
\item $W_1=W_2=\b$, $W_3=\varepsilon$, and $W_0 \ W_1 \ W_2 \ W_4\overset{\delta'}{\mapsto} W$.
\end{itemize}

It is not hard to verify that $\I$ defined in this way is a model of axioms (\ref{Eq_ELUndecTransition})-(\ref{Eq_ELUndecEpsilonTransition3}).

Finally, let $H^\I=\qacc^\I=\varnothing$. Then by using arguments from the proof of the `only if' direction, one can show that $\I$ is well defined and hence, $\I$ is a model of ontology $\O$ such that $\I\not\models A\dleq H$. 
\end{proof}

To complete the proof of Theorem \ref{Teo_Hardness_EL-cyclic} we now show that ontology $\O$ is expressible by a cyclic $\EL$-ontology network. Let us demonstrate that so is the family of axioms (\ref{Eq_ELUndecStartRchain}). Let $\varphi=A\dleq B$ be a concept inclusion and $B, B1, B_2$ concept names. By Lemma \ref{Lem_Expressibility_infty}, ontology $\O_1=\{\varphi[B\mapsto \exists r^k.B] \mid k\geqslant 0\}$ is expressible by a cyclic $\EL$-ontology network. Then by Lemma \ref{Lem_Expressibility_SimpleSubstitution}, ontology $\O_2=\O_1[B\mapsto B_1 \dcap \varepsilon \dcap \ex{r}.(\qo \dcap B_2)]$ is expressible by a cyclic $\EL$-ontology network. By applying Lemma \ref{Lem_Expressibility_infty} again, we conclude that so is ontology $\O_3=\bigcup_{l\geqslant 0}\O_2[B_1\mapsto \exists v^l.B_1, \ B_2\mapsto \exists (r,\b)^{2l}.B_2]$, i.e., the ontology given by axioms
\begin{equation*}
A\dleq\exists r^k.(\exists v^l.B_1\dcap\varepsilon\dcap\ex{r}.(\qo\dcap\exists (r,\b)^{2l}.B_2))
\end{equation*}
for $k,l\geqslant 0$. Further, by Lemma \ref{Lem_Expressibility_SimpleSubstitution}, we obtain that $\O_2[B_1\mapsto L, \ B_2\mapsto \exists (r,\b)^2]$ is expressible by a cyclic $\EL$-ontology network and hence, so is the family of axioms (\ref{Eq_ELUndecStartRchain}). A similar argument shows the expressibility of ontologies given by axioms (\ref{Eq_ELUndecPropagateLengthMarker})-(\ref{Eq_ELUndecEpsilonTransition3}). The remaining subset of axioms (\ref{Eq_ELUndecid_PropagateH}) of $\O$ is finite. By Lemma \ref{Lem_IteratedExpressibility}, there exists a cyclic $\EL$-ontology network $\N$ and an ontology $\O_\N$ such that $\O$ is $(\N,\O_\N)$-expressible and thus, it holds $\O_\N\models_\N A\dleq H$ iff $M$ halts. Theorem \ref{Teo_Hardness_EL-cyclic} is proved.

\bigskip

%===================================

\commentout {%%%%%% OLD PROOFS OF EXPTIME AND UNDECIDABILITY FOR EL

\noindent\textbf{Lemma \ref{Lem_Exptime_LowerBound}} (Hardness for Acyclic $\EL$-Networks). \textit{The word problem for TMs making exponentially many steps is reducible to entailment in acyclic $\EL$-ontology networks.}
\medskip

\begin{proof}
Given a Turing Machine $M=\langle Q,\Gamma,\delta\rangle$ and a natural number $n$, we define a $v$-pointed ontology network $\N$ and a concept inclusion $C\dleq D$ such that $(\N,v)\models C\dleq D$ iff the initial configuration $\b^+q_0b^+$ of $M$ is $k$-accepting, for some $1<k\leqslant \IEXP{n}$. Let $q_0,q_{acc}$ be the initial/accepting states of $M$ and $\delta'$ be the yield function corresponding to $\delta$.

The definition of the network $\N$ resembles the construction from Proposition \ref{Lem_ELUndec_LowerBound} and the label ontologies of $\N$ will be ``copies'' of subsets of axioms given in that construction. For this reason, the axioms of the label ontologies will be given without additional explanation. 

Consider the network $\N$ depicted on Figure 4. The subnetwork induced by the label ontologies $\O_{ts}$ and $\O_{tape}^i$, $i=2,\ldots ,n+2$, will be used to define models which $(x,2^{n+2})$-represent the initial configuration configuration of $M$, for some domain element $x$. The network induced by $\O_{ts}$ and $\O_{step}^i$, $i=2,\ldots ,n$, will be used to model $2^n$ transitions of the machine $M$. Observe that the tree unravelling of this subnetwork contains exactly $2^n$ nodes, each of them will represent a single transition of $M$. 

%Consider the concept signature ${\mathcal{A}}\cup\{A,C,D,E,F,G,H\}$ consisting of the set ${\mathcal{A}}$ of concept names, corresponding to the alphabet and state symbols of $M$, and the auxiliary concepts given in brackets. Denote by $\Si$ the union of this signature with a role $\{r\}$. 

%The concept names from ${\mathcal{A}}$ together with role $r$ will be used to represent configurations of $M$, $D$ will be used as a marker for defect situations (i.e. when interpretations of some concepts from ${\mathcal{A}}$ have a non-empty intersection), $H$ will be a halting marker, and $C,E,F,G$ auxiliary concepts. Let ${\mathcal{A}}^1$ be a ``copy'' of the concept signature ${\mathcal{A}}$ consisting of the concept names for alphabet and state symbols of $M$, all having the superscript $^1$. 

%We define the signatures labelling edges of network $\N$ as:

%\begin{center}
%$\ \ \gamma_1= \{D,H,E,r\}\cup {\mathcal{A}}, \ \ \ \   \gamma_2 = \{D,H,A,r\}\cup {\mathcal{A}}$

%\smallskip

%$\sigma_1=\{D,H,r\}\cup {\mathcal{A}}^1, \ \ \ \ \ \sigma_2=\{D,H,r\}\cup {\mathcal{A}} $
%\end{center}

%The ontologies labelling the nodes in the network $\N$ are defined as sets of axioms over the signature  $\Si\cup {\mathcal{A}}^1$ as follows.

First, we introduce a number of auxiliary ontologies which will be subsets of the label ontologies of the network $\N$. For $i,j\in\{1,\ldots ,n\}$ and $k,m\in\{1,2,3\}$, let us define  $\O_{\mathcal{A}^{i}_k\rightarrow\mathcal{A}^j_m}$ as the ontology which consists of the axioms implementing a transition between configurations described in signatures $\mathcal{A}^{i}_k\cup\{r\}$ and $\mathcal{A}^{j}_m\cup\{r\}$, respectively:

\begin{equation}\label{Eq_EL-Exptime_Transition}
X^i_k \dcap \ex{r}.(Y^^i_k \dcap \ex{r}.(U^i_k \dcap \ex{r}.Z^i_k)) \dleq W^j_m
\end{equation}
for all $X^i_k,Y^i_k,U^i_k,Z^i_k\in{\mathcal{A}^i_k}$ and $W^j_m\in{\mathcal{A}^j_m}$ such that $XYUZ\overset{\delta'}{\mapsto} W$,

\smallskip

and of the axioms defining markers for accepting and defect situations:

\begin{equation}\label{Eq_EL-Exptime_PropagateH}
(q_{acc})^i_k \dleq H, \ \ \ex{r}.H \dleq H
\end{equation}

\begin{equation}\label{Eq_EL-Exptime_DefectMarker}
X \dcap Y \dleq	 D, \ \ \ex{r}.D \dleq D
\end{equation}
for all concepts $X,Y\in \mathcal{A}^i_k$, with $X\neq Y$.

\medskip

For $i,j\in\{1,\ldots ,n\}$ and $k,m\in\{2,3\}$, let us define  $\O^{copy}_{\mathcal{A}^{i}_k\mid\mathcal{A}^j_m}$ as the ontology consisting of the equivalences $X^i_k\equiv X^j_m$, where $X^i_k\in\mathcal{A}^{i}_k$ is a concept name and  $X^j_m$ is the counterpart concept of $X^i_k$ in $\mathcal{A}^{j}_m$. These axioms will be used to ``copy''  descriptions of configurations of $M$ between signatures $\mathcal{A}^{i}_k$ and $\mathcal{A}^{j}_m$. 

\medskip

Now the ontology $\O_{ts}$ (tape/step ontology) is defined as the union of 

$$\O_{\mathcal{A}^{1}_1\rightarrow\mathcal{A}^1_2}, \ \ \O^{copy}_{\mathcal{A}^1_2\mid\mathcal{A}^1_3}, \ \ \O_{\mathcal{A}^1_3\rightarrow\mathcal{A}^{1}_4}$$

with the following set of axioms:

\begin{equation}
A^{1}_1\dleq\ex{r}.A^1_2, \ \ A^1_2\dleq A^1_3, \ \ A^1_3\dleq\ex{r}.A^{1}_4
\end{equation}

\begin{equation}
A^{2}_3 \dleq A^{2}_2
\end{equation}

\begin{equation}\label{Eq_EL-Exptime_TapeMain}
A^1_4 \dleq (q_0)^1_1
\end{equation}
%(informally, the intended purpose of these axioms is to place the marker for the initial state to the ``right'' of $A$);

\begin{equation}\label{Eq_ELUndecid_InitRight}
(q_0)^1_1 \dleq \ex{r}.F, \ \ F \dleq \ex{r}.F, \ \ F \dleq \b^1_1
\end{equation}
%(these axioms serve for representing an infinite string of blank symbols located to the ``right'' of $q_0$);

\begin{equation}\label{Eq_EL-Exptime_InitLeft}
\ex{r}.(q_0)^1_1 \dleq C, \ \ \ex{r}.C \dleq C, \ \ C \dleq \b^1_1
\end{equation}
%(used for representing an infinite string of blank symbols to the ``left'' of $q_0$);

\begin{equation}\label{Eq_EL-Exptime_GoodMarker}
A^1_1 \dcap D \dleq	 G, \ \ A^1_1 \dcap H \dleq G.
\end{equation}
%(the marker $G$ is initialised, whenever an element from interpretation of $A$ belongs to interpretation of either $D$ or $H$).

\medskip

For $i=2,\ldots ,n+2$, we define $\O_{tape}^i$ as the set of axioms:

\begin{equation}\label{Eq_EL-Exptime_TapeIncrease}
A^{i-1}_2\dleq\ex{r}.A^i_2, \ \ A^i_2\dleq A^i_3, \ \ A^i_3\dleq\ex{r}.A^{i-1}_3,
\end{equation}

\begin{equation}\label{Eq_EL-Exptime_TapeAuxCopy}
A^{i+1}_3 \dleq A^{i+1}_2.
\end{equation}

\medskip

For $i=2,\ldots ,n$, $\O_{step}^i$ is defined as the union of

$$\O_{\mathcal{A}^{i-1}_2\rightarrow\mathcal{A}^i_2}, \ \ \O^{copy}_{\mathcal{A}^i_2\mid \mathcal{A}^i_3}, \ \ \O_{\mathcal{A}^i_3\rightarrow\mathcal{A}^{i-1}_3}$$

and 

$$\O^{copy}_{\mathcal{A}^{i+1}_3\mid \mathcal{A}^{i+1}_2}.$$

Note the similarity of axioms (\ref{Eq_EL-Exptime_TapeIncrease}), (\ref{Eq_EL-Exptime_TapeAuxCopy}) and the notations of ontologies in $\O_{step}^i$. 

Let $v$ be the node in the network $\N$ labeled by the ontology $\O_{ts}$. 

\Todo{Todo}

\end{proof}

\noindent\textbf{Lemma \ref{Lem_Undecidability}} (Undecidability for General $\EL$-Networks). \textit{The halting problem for TMs is reducible to entailment in $\EL$-ontology networks.}
\medskip

\Todo{Add Taut() to the label ontologies}

\begin{proof}
Let $M=\langle Q,\Gamma,\delta \rangle$ be a TM and $\delta'$ be the yield function corresponding to $\delta$. Consider the ontology network $\N$ constructed for $M$ and depicted on Figure 1. The ontology $\O_{ts}$ (``tape/step-ontology'') will serve two purposes: in combination with $\O_{tape}$ it will model the tape and the initial configuration of $M$, while in combination with $\O_{copy}$ it will model transitions of $M$. In the following, we provide a precise definition of these ontologies.

%\vspace{-0.8cm}

Consider the concept signature ${\mathcal{A}}\cup\{A,C,D,E,F,G,H\}$ consisting of the set ${\mathcal{A}}$ of concept names, corresponding to the alphabet and state symbols of $M$, and the auxiliary concepts given in brackets. Denote by $\Si$ the union of this signature with a role $\{r\}$. The concept names from ${\mathcal{A}}$ together with role $r$ will be used to represent configurations of $M$, $D$ will be used as a marker for defect situations (i.e. when interpretations of some concepts from ${\mathcal{A}}$ have a non-empty intersection), $H$ will be a halting marker, and $C,E,F,G$ auxiliary concepts. Let ${\mathcal{A}}^1$ be a ``copy'' of the concept signature ${\mathcal{A}}$ consisting of the concept names for alphabet and state symbols of $M$, all having the superscript $^1$. We define the signatures labelling edges of network $\N$ as:

\begin{center}
$\ \ \gamma_1= \{D,H,E,r\}\cup {\mathcal{A}}, \ \ \ \   \gamma_2 = \{D,H,A,r\}\cup {\mathcal{A}}$

\smallskip

$\sigma_1=\{D,H,r\}\cup {\mathcal{A}}^1, \ \ \ \ \ \sigma_2=\{D,H,r\}\cup {\mathcal{A}} $
\end{center}

The ontologies labelling the nodes in the network $\N$ are defined as sets of axioms over the signature  $\Si\cup {\mathcal{A}}^1$ as follows.

\medskip

$\O_{ts}$ consists of the axioms:

\begin{equation}\label{Eq_ELUndecid_TapeMain}
A \dleq \ex{r}.E, \ \ E \dleq q_0
\end{equation}
(informally, the intended purpose of these axioms is to place the marker for the initial state to the ``right'' of $A$);

\begin{equation}\label{Eq_ELUndecid_InitRight}
q_0 \dleq \ex{r}.F, \ \ F \dleq \ex{r}.F, \ \ F \dleq \b
\end{equation}
(these axioms serve for representing an infinite string of blank symbols located to the ``right'' of $q_0$);

\begin{equation}\label{Eq_ELUndecid_InitLeft}
\ex{r}.q_0 \dleq C, \ \ \ex{r}.C \dleq C, \ \ C \dleq \b
\end{equation}
(used for representing an infinite string of blank symbols to the ``left'' of $q_0$);

\begin{equation}\label{Eq_ELUndecid_Transition}
X \dcap \ex{r}.(Y \dcap \ex{r}.(U \dcap \ex{r}.Z)) \dleq W^1,
\end{equation}
for all $X,Y,U,Z,W\in{\mathcal{A}}$ such that $XYUZ\overset{\delta'}{\mapsto} W$ (these axioms are used for implementing transitions of $M$); 

\begin{equation}\label{Eq_ELUndecid_PropagateH}
q^1_{h} \dleq H, \ \ \ex{r}.H \dleq H 
\end{equation}
(used to initialise the halting marker and propagate it to the ``left'');

\begin{equation}\label{Eq_ELUndecid_DefectMarker}
X \dcap Y \dleq	 D, \ \ \ex{r}.D \dleq D
\end{equation}
for all concepts $X,Y\in \mathcal{A}$, with $X\neq Y$ (to initialize the defect marker and propagate it to the ``left'');

\begin{equation}\label{Eq_ELUndecid_GoodMarker}
A \dcap D \dleq	 G, \ \ A \dcap H \dleq G
\end{equation}
(the marker $G$ is initialised, whenever an element from interpretation of $A$ belongs to interpretation of either $D$ or $H$).

\smallskip

The ontology $\O_{tape}$ consists of the single axiom 

\begin{equation}\label{Eq_ELUndecid_TapeIncrease}
E \dleq A
\end{equation}
which, together with axioms (\ref{Eq_ELUndecid_TapeMain}), will be used to model initial configurations of $M$ with different positions of $q_0$.

\medskip

Finally, the ontology $\O_{copy}$ is the set of equivalences $X\equiv X^1$, where $X\in {\mathcal{A}}$ and $X^1$ is the counterpart concept of $X$ from ${\mathcal{A}}^1$. 

\bigskip Let $v_{ts}, v_{tape}$, and $v_{step}$ be the vertices of $\N$ labeled by ontologies $\O_{ts}$, $\O_{tape}$, and $\O_{step}$, respectively. We claim that $M$ halts iff $(\N,v_{ts})\models A\dleq G$. \medskip

To prove the claim, let us introduce a number of auxiliary notions and statements that will help to clarify the key properties of the network $\N$. 

%\begin{lemma}\label{EL_Undec_Lem_PropertyOfAgreedModel}
%Let $\I$ be a model of $\O_{ts}$ such that the following conditions hold:
%\begin{itemize}
%\item there exists $\J\in g(v_{copy})$ such that $\I=_{\Si_2}\J$;
%\item if $x\in E^\I$, then there is an infinite $r$-chain outgoing from $x$.
%\end{itemize}
%Then we have $\I\in g(v_{ts})$.
%\end{lemma}

%\ProofLemma The lemma is proved if we show that a model $\I$ satisfying the properties above is $m$-agreed for any natural number $m$. We use induction on $m$; the induction base $m=0$ is trivial, since $\I$ is a model of $\O_{ts}$. In the induction step, we assume that the model $\I$ is $(m-1)$-agreed and show that it is $m$-agreed as well. It suffices to demonstrate that there exists a model $\I'$ of $\O_{tape}$ which is $(m-1)$-agreed and for which $\I=_{r,E}=\I'$ holds. Take the reduct $\I\mid_{\{r,E\}}$ Consider an interpretation $\I'$ obtained from $\I$ by the following transformation: 

The most important feature of $\N$ is that certain models given by $g(v_{ts})$ will represent configurations of the Turing Machine $M$. Let $\c$ be a configuration of $M$. We say that a model $\I=(\Delta,\cdot^\I)$ of $\O_{ts}$ \textit{represents} $\c$ if $\Delta$ is infinite and there is an enumeration of elements of $\Delta$ by integers such that $r^\I=\{\langle x_i,x_{i+1} \rangle \mid i \text{ is an integer} \}$, $D^\I=\varnothing$, and for all $x_i\in\Delta$, we have $x_i^{\mathcal{A}}=\{a\}$ iff the $i$-th symbol of $\c$ is $a$. 

If a model $\I$ represents a configuration $\c$ then, by axioms (\ref{Eq_ELUndecid_Transition}) (since the transition function of $M$ is total) the interpretation of some $\mathcal{A}^1$-concepts in $\I$ is not empty and corresponds to a successor configuration of $\c$. We call an interpretation $\I'$  \textit{successor} of $\I$ if it represents a successor configuration $\c'$ of $\c$ such that for each $\mathcal{A}$-concept $X$, the interpretation of $X$ in $\I'$ is the same as interpretation of $X^1$ in $\I$. 

The ontology $\O_{ts}$ together with $\O_{copy}$ will be used to ``generate'' models which represent subsequent configurations of $M$. We will also consider models that ``approximately'' represent configurations. For a given natural number $m$, they will represent a substring of configuration $\c$ which starts $m$ positions to the left of the state symbol in $\c$. To illustrate the purpose of these models, fix a number $m$ and consider the infinite substring $\tt{s}$ of $\c_{init}$, which starts $4m$ positions to the left of $q_{0}$. Note that if $\c_{init}$ is $m$-accepting then it is possible to obtain a string with $q_{h}$ from $\tt{s}$ by applying the yield function $\delta'$ $m$-times. This observation will be used in the proof of the proposition and leads us to the following notion.

Let $m\geqslant 1$ be an arbitrary natural number and $\c=vqw$ be a configuration of $M$. Let $\I$ be a model of $\O_{ts}$ and $x$ be an element in its domain. We say that the model $\I$ $(x,m)$-\textit{represents} $\c$ if there is an infinite sequence of elements $x_{-m},\ldots , x_{-1},x_0,x_{1},\ldots$ (not necessarily distinct ones) in $\I$, with $x_{-m}=x$, such that for all integers $i\geqslant -m$:
\begin{itemize}
\item $\langle x_i, x_{i+1}\rangle\in r^\I$ and  $x_{i}\not\in D^\I$;
\item $x_i^{\mathcal{A}}=\{a\}$ iff the $i$-th symbol in $\c$ is $a$. %\newline (thus, $x_0\in q^\I$).
\end{itemize} 

Note that it follows from axioms (\ref{Eq_ELUndecid_TapeMain}, \ref{Eq_ELUndecid_InitRight}, \ref{Eq_ELUndecid_InitLeft}, \ref{Eq_ELUndecid_DefectMarker}) that any model $\I$ of $\O_{ts}$, with $D^\I=\varnothing$, $(x,1)$-represents $\c_{init}$ for every $x\in A^\I$.

\medskip

For proving properties of network $\N$ by induction we will use Definition \ref{De_nAgreedModel} and  some refinements thereof. For a natural number $m$, we call a model $\I$ of $\O_{ts}$ $(\gamma,m)$-agreed if there exists a sequence of interpretations $\I_0,\I_1,\ldots ,\I_m$, with $\I_0=\I$, such that: 
\begin{itemize}
\item for odd indices $0 < j \leqslant m$,  $\I_j$ is a model of $\O_{tape}$, $\I_{j-1}=_{\gamma_1}\I_j$, and $\I_{j}=_{\gamma_2}\I_{j+1}$, if $j<m$;
\item for even indices $0 < j \leqslant m$,  $\I_j$ is a model of $\O_{ts}$, $\I_{j-1}=_{\gamma_2}\I_j$, and $\I_{j}=_{\gamma_1}\I_{j+1}$, if $j<m$.
\end{itemize}
The model $\I_m$ is called $(\gamma,m)$-reachable from $\I$.

\medskip

A model $\I$ of $\O_{ts}$ is called $(\sigma,m)$-agreed if there exists a model $\J$ of $\O_{copy}$ such that $\I=_{\sigma_1}\J$ and $\J$ is $(m-1)$-agreed. 

It is not hard to verify that $\I\in g(v_{ts})$ iff for all $m,n$, the interpretation $\I$ is $(\sigma,m)$-agreed and there exists a model $(\gamma,n)$-reachable from $\I$ which is $(\sigma,m)$-agreed. 

\medskip

First, we formulate an easy property of network $\N$ which demonstrates the interplay between $\O_{ts}$ and $\O_{tape}$.

\begin{lemma}\label{EL_Undec_Lem_ReachableModel}
Let $\I$ be a model of $\O_{ts}$ which $(x,1)$-represents $\c_{init}$ for some element $x$ such that $x\in A^\I$. Then for any $m$, if a model $\J$ is $(\gamma,2m-2)$-reachable from $\I$ then $\J\models\O_{ts}$ and $\J$ $(x,m)$-represents $\c_{init}$.

%Let $\I$ be a model of $\O_{ts}$ which represents $\c_{init}$. Then for any natural number $m$, there exists a model $\J\in g(v_{ts})$ is $(2m-1)$-reachable from $\I$ then $(x,m)$-represents $c_{init}$. 
%Let $\I$ be a model of $\O_{ts}$ with elements $x,y$ such that $y$ is a $k$-successor of $x$ wrt $r$ in $\I$  and $y\in E^\I$. Then for any natural $m \geqslant k$, any model $\J\in g(v_{ts})$ $(m-k)$-reachable from $\I$ $(x,m)$-represents $c_{init}$ and there is an $m$-successor $z$ of $x$ wrt $r$ in $\J$ such that  $z\in E^\J$.
\end{lemma}

\begin{proof} 
We show by induction on $m$ that any model $\J$ satisfying the above conditions $(x,m)$-represents $\c_{init}$, $x\not\in D^\I$, and there is a $(r,m)$-successor $y$ of $x$ in $\J$ such that $y\in E^\J$.  

The induction base $m=1$ is trivial: as $\I\models\O_{ts}$, by the axioms in (\ref{Eq_ELUndecid_TapeMain}), there is an element $y$ such that $\langle x,y \rangle\in r^\I$ and $y\in E^\I$, so $\I$ is the required model. 

In the induction step, take a model $\J'$ of $\O_{ts}$ which is $(\gamma,2(m-1)-2)$-reachable from $\I$ and such that $x\not\in D^{\J'}$ and there is a $(r,m-1)$-successor $z$ of $x$ in $\J'$, with $z\in E^{\J'}$. Then, by definition of the network $\N$ and $\O_{tape}$, if a model $\J''$ is $(\gamma,1)$-reachable from $\J'$, then $\J''\models\O_{tape}$ and it holds that $x\not\in D^{\J''}$, $z\in A^{\J''}$. Similarly, if a model $\J$ is $(\gamma,1)$-reachable from $\J''$ then $\J\models\O_{ts}$,  $x\not\in D^{\J}$, and by the axioms in (\ref{Eq_ELUndecid_TapeMain}), there must be an element $y$ such that $\langle z,y \rangle\in r^\J$ (i.e. $y$ is $(r,m)$-successor of $x$), $y\in E^\J$, and $y\in (q_0)^{\J}$. Then it follows from axioms (\ref{Eq_ELUndecid_InitRight}, \ref{Eq_ELUndecid_InitLeft}, \ref{Eq_ELUndecid_DefectMarker}) that the model $\J$ $(x,m)$-represents $\c_{init}$ and $y$ is the required element. The model $\J$ is $(\gamma,2)$-reachable from $\J'$, hence $\J$ is $(\gamma,2m-2)$-reachable from $\I$ and the lemma is proved.

\end{proof}

%Consider a model $\J$ which is $1$-reachable from $\J'$ (and thus, $(n+1)$-reachable from $\I$). By the definition of ontology $\O_{tape}$, in any model $\J''$, with ${\J''}=_{\{E,r\}}\J'$, it must hold that $z\in A^{\J''}$. But then, by axioms (\ref{Eq_ELUndecid_TapeMain}) of $\O_{ts}$, there must be a $1$-successor $z'$ of $z$ (and hence, $(n+1)$-successor of $x$) in $\J$ such that $z\in E^{\J'}$ and $z\in (q_0)^{\J'}$. It follows from axioms (\ref{Eq_ELUndecid_InitLeft,Eq_ELUndecid_InitRight}) that the model $\J'$ $(x,n+1)$-represents $c_{init}$ and $z'$ is the required element. \QED

Now we formulate the key statement in the proof of Proposition \ref{Lem_ELUndec_LowerBound}. 

%\begin{definition}[(x,k)-represented configuration]\label{De_(x,k)-model}
%Let $\I$ be a model of $\O_{ts}$, $x$ be a point from its domain, and $k$ be a natural number.  Then $\I$ is called $(x,k)$-model if there is an infinite $r$-chain outgoing from $x$ in $\I$ such that $z\in(q_0)^\I$, for the point $z$ reachable from $x$ in $k$ steps, and $y\in(b)^\I$ for any point $y$ from the chain different from $z$.

%A $(x,k)$-model is said to be representing the initial configuration of $M$ if it interprets the concept $D$ as the empty set.
%\end{definition}

%Informally speaking, a model $\I$ $(x,k)$-representing a configuration of $c$ contains a $r$-chain outgoing from $x$ which ``corresponds'' to the string $\{b\}^{m} c$, where $\{b\}^{m}$ denotes the string of blank symbols having length $m$. Clearly, there can be several different chains of this kind in $\I$, also for different points $x$.

\begin{lemma}\label{EL_Undec_Lem_PropertyOfRepresentingModel}
A configuration $\c$ of $M$ is $m$-accepting for a natural number $m$ iff in any model $\I\in g(v_{ts})$, which $(x,4m)$-represents $\c$ for some element $x$, it holds that $x\in H^\I$.
\end{lemma}

\begin{proof} 
$(\Rightarrow):$  We prove the ``only-if'' direction of the lemma by induction on $m$. Let $\I\in g(v_{ts})$ be a model which $(x,4m)$-represents configuration $\c$.

Induction base, $m=1$. Let $\c=uaqbw$, where $q\in Q$, $a,b\in\Gamma$ and $u,w\in\Gamma^+$. Since $\c$ is $1$-accepting, it has a successor configuration $\c'$ of the form $uq_{h}a'bw$ if $\delta(q,a)=\langle q',a',-1\rangle$ or of the form $ua'bq_{h}w$ if $\delta(q,a)=\langle q',a',1\rangle$. Then, by transition axioms (\ref{Eq_ELUndecid_Transition}) of $\O_{ts}$, there is an element $y\in (q^1_{h})^\I$ which is $(r,k)$-successor of $x$, for $k$ equal to $(4+|u|-1-2)$ or $(4+|u|+1-2)$ (i.e. for $k=1+|u|$ or $k=3+|u|$), respectively, depending on the form of $\c'$. Then it follows from axioms (\ref{Eq_ELUndecid_PropagateH}) that $x\in H^\I$ . 

Induction step: we assume that the claim is proved for all $n < m$ and show that it holds for $m$. Since $\c$ is $m$-accepting, there is a successor configuration $\c'$ which is $(m-1)$-accepting. We have $\I\in g(v_{ts})$, hence, by axioms (\ref{Eq_ELUndecid_Transition}) and by the definition of ontology $\O_{copy}$, there must exist a model $\J\in g(v_{ts})$ such that $\J=_H \I $ and $\J$ $(x,4m-i)$-represents $\c'$, where $i=3$ or $i=1$, depending on the form of $\c'$. Then there exists an element $y$, which is $r$-reachable from $x$, such that $\J$ $(y,4(m-1))$-represents $c'$. By the induction hypothesis, we have $y\in H^\J$ and hence, $x\in H^\I$, by the second axiom in (\ref{Eq_ELUndecid_PropagateH}).

$(\Leftarrow):$ is proved by contraposition and a similar induction on $m$. We show that if $\c$ is not $m$-accepting then there exists a $2m-2$-agreed model $\I$ of $\O_{ts}$ representing $\c$ such that $H^\I=\varnothing$ and moreover, $A^\I\neq\varnothing$ in case $\c$ is initial and $A^\I = \varnothing$, otherwise. 

Induction base $m=1$. If $\c\neq\c_{init}$ then consider an interpretation $\I$ representing $\c$ such that $A,H,E,q_0,F,C$ are interpreted as the empty set in $\I$ and the interpretation of ${\mathcal{A}}^1$-symbols in $\I$ is completely defined by axioms (\ref{Eq_ELUndecid_Transition}), i.e. for any ${\mathcal{A}}^1$-symbol $W$ and an integer $i$, we have $x_i^\I\in W^\I$ iff $x_{i}^{\mathcal{A}}x_{i+1}^{\mathcal{A}}x_{i+2}^{\mathcal{A}}x_{i+3}^{\mathcal{A}}\overset{\delta'}{\mapsto} W$. In other words, the interpretation of ${\mathcal{A}}^1$-symbols in $\I$ corresponds to a successor configuration $\c'$ of $\c$. Since $\c$ is not $1$-accepting, the state symbol in $\c'$ is different from $q_{h}$, thus $(q_h^1)^\I=\varnothing$, the axioms (\ref{Eq_ELUndecid_Transition}, \ref{Eq_ELUndecid_PropagateH}) are satisfied in $\I$, and it is easy to verify that $\I$ is a model of $\O_{ts}$. 

If $\c$ is initial then consider an interpretation $\I$ representing $\c_{init}$ such that $H^\I=G^\I=\varnothing$, $A^\I=\{x\}$, for an element $x$, which has a $(r,1)$-successor $y$ with $y=q_0^\I$, and the interpretation of ${\mathcal{A}}^1$-symbols in $\I$ is completely defined by the transition axioms (\ref{Eq_ELUndecid_Transition}) as formulated above. Set $E^\I=\{y\}$, $F^\I=\{z \mid z$ is $r$-reachable from $y\}$, and $C^\I=\{z \mid y$ is $r$-reachable from $z\}$. Since $\c_{init}$ is not $1$-accepting, the state in a successor configuration of $\c_{init}$ is different from $q_{h}$, the axioms (\ref{Eq_ELUndecid_Transition}, \ref{Eq_ELUndecid_PropagateH}) are satisfied, and one can readily verify that $\I$ is a model of $\O_{ts}$. 

Induction step: we assume that the claim is proved for all $n<m$ and show that it holds for $m$ as well. Let configuration $\c$ have the form $uaqbw$, where $q\in Q$, $a,b\in\Gamma$, $u,w\in\Gamma^+$. Since $\c$ is not $m$-accepting, it has a successor configuration $\c'$ which is not $(m-1)$-accepting and has the form $uq'a'bw$ if $\delta(q,a)=\langle q',a',-1\rangle$, or the form $ua'bq'w$ if $\delta(q,a)=\langle q',a',1\rangle$. Let $\I'$ be a $(2(m-1)-2)$-agreed model of $\O_{ts}$ representing $\c'$ and satisfying the conditions of the claim. Consider an interpretation $\I=(\Delta,\cdot^\I)$ representing $\c$ such that $\I'$ is a successor of $\I$, $D^\I=H^\I=G^\I=\varnothing$, the interpretation of $A,E,F,C$ is empty in $\I$ if $\c$ is not initial, and is defined otherwise as follows: $A^\I=\{x\}$ and $E^\I=\{y\}$, for elements $x,y$ such that $\langle x,y \rangle\in r^\I$ and $\{y\}=q_0^\I$, while $F^\I=\{z \mid z$ is $r$-reachable from $y\}$, and $C^\I=\{z \mid y$ is $r$-reachable from $z\}$. It is straightforward to verify that $\I$ defined in this way is a model of $\O_{ts}$ and there is a model $\J\models\O_{copy}$ such that $\I=_{\sigma_1}\J$ and $\J=_{\sigma_2}\I'$. Therefore, $\J$ is $2m-3$-agreed, hence $\I$ is $(\sigma,2m-2)$-agreed and it suffices to verify that $\I$ is $(\gamma,2m-2)$-agreed. The claim is proved if we show that for any $k$, there exists a model $\I_{k}$ of $\O_{ts}$, which is $(\gamma,2k)$-reachable from $\I$ and is $(\sigma,2m-2)$-agreed. We show by induction on $k$ that there exists a model $\I_{k}\models\O_{ts}$ such that $\I_k$ is $(\gamma,2k)$-reachable from $\I$ and $\I_k$ is isomorphic to $\I$ (then clearly, $\I_{k}$ will be $(\sigma,2m-2)$-agreed, because so is $\I$). 

For $k=0$, there is nothing to prove, since we can take $\I_{k}=\I$. Now let $k\geqslant 2$ be arbitrary. Consider $\I_{k-2}$, a model $(\gamma,k-2)$-reachable from $\I$ and isomorphic to $\I$. If $\c$ is not initial then define models $\I_{k-1}$ and $\I_{k}$ equal to $\I$. We have $\I_k\models\O_{ts}$ and $\I_{k-1}\models\O_{tape}$, since $A^\I=E^\I=\varnothing$, which means that $\I_{k-2}$ is $2$-agreed and yields that $\I_{k}$ is $(\gamma,k)$-reachable from $\I$.  If $\c=\c_{init}$ then let us define interpretations $\I_{k}$ and $\I_{k-1}$ as follows. The domains of $\I_{k-2}$, $\I_{k-1}$, and $\I_{k}$ coincide, $r^{\I_{k-2}}=r^{\I_{k-1}}=r^{\I_{k}}$, for any concept name $X$ and an element $x_{i}\in\Delta$, we have $x_{i+1}\in X^{\I_{k}}$ iff $x_i\in X^{\I_{k-2}}$ (recall that $\I_{k-2}$ is isomorphic to the model $\I$ representing $\c_{init}$ and hence, there is enumeration of elements of $\I_{k-2}$ and $\I_{k}$ by integers), while interpretation of concept names in $\I_{k-1}$ is defined as follows. For concepts $A$ and $E$, we define $A^{\I_{k-1}}=E^{\I_{k-1}}=E^{\I_{k-2}}$ and set $X^{\I_{k-1}}=\varnothing$ for any other concept name. One can readily verify that $\I_{k-2}$ and $\I_{k}$ are isomorphic, $\I_{k-1}\models\O_{tape}$, $\I_{k}\models\O_{ts}$, $\I_{k-2}=_{\gamma_1}\I_{k-1}$, and $\I_{k-1}=_{\gamma_2}\I_{k}$. This means that $\I_{k-2}$ is $2$-agreed and hence, $\I_{k}$ is $(\gamma,k)$-reachable from $\I$.

%Let us show that for any $k$, there exists a $(R,2m-2)$-agreed model $\I_2k$ of $\O_{ts}$ representing $c$, which is $(\Sigma_1,2k)$-reachable from $\I$; then the latter claim will follow. We show by induction on $k$ that there exists such a model which is $k$-shifted wrt $\I$. 

%For $k=0$, there is nothing to prove, so let $k\geqslant 2$ be arbitrary. Consider $\I_{2k-2}$, a $2(k-1)$-shifted model wrt $\I$. Let $\I_2k$ be a model which is $2$-shifted wrt $\I_{2k-2}$. By definition of $\I_{2k-2}$ and $\I_2k$, one can readily check that there exists a model $\I_{2k-1}$ of $\O_{tape}$ such that $\I_{2k-2}=_{D,E,H,r}\I_{2k-1}$ and $\I_{2k-1}=_{A,D,H,r}\I_2k$, hence $\I_2k$ is $2$-reachable from $\I_{2k-2}$ (an thus, $2k$-reachable from $\I$). On the other hand, since $\I_2k$ is isomorphic to $\I$ and $I$ is $(R,2m-2)$-agreed, it follows that $\I_2k$ is $(R,2m-2)$-agreed too. 

Therefore, we have proved that the model $\I$ is $(2m-2)$-agreed for any $m$, and thus $\I$ is as required: $\I\in g(v_{ts})$, it $(x,4m)$-represents configuration $\c$ for some $x\in\Delta$ (since $\I$ represents $\c$), and $H^\I=\varnothing$. 
\end{proof}

Now we are ready to complete the proof of Proposition \ref{Lem_ELUndec_LowerBound} by showing that $M$ halts iff $(\D,v_{ts})\models A\dleq G$. 

\medskip

For the ``only if'' direction, assume that $M$ halts (the initial configuration $\c_{init}$ is $m$-accepting for some number $m$), but $(\D,v_{ts})\not\models A\dleq G$, i.e. there is a model $\I\in g(v_{ts})$ and an element $x$ such that $x\in A^\I$, but $x\not\in G^\I$. Then, by axioms (\ref{Eq_ELUndecid_GoodMarker}), we must have $x\not\in H^\I$ and $x\not\in D^\I$, and hence, by axioms (\ref{Eq_ELUndecid_TapeMain}, \ref{Eq_ELUndecid_InitRight}, \ref{Eq_ELUndecid_InitLeft}), the model $\I$ $(x,1)$-represents $\c_{init}$. Take a model $\J\in g(v_{ts})$ which is $(\gamma,8m-2)$-reachable from $\I$; such a model exists, since $\I\in g(v_{ts})$. By Lemma \ref{EL_Undec_Lem_ReachableModel}, $\J$ $(x,4m)$-represents $\c_{init}$ and it holds that $x\not\in H^\J$, since $x\not\in H^\I$ and $H\in\gamma_i$, for $i=1,2$. As $\J\in g(v_{ts})$, by Lemma \ref{EL_Undec_Lem_PropertyOfRepresentingModel}, we arrive at contradiction.

For the ``if'' direction, assume by contraposition that $\c_{init}$ is not $m$-accepting for any $m$, and consider the model $\I$ representing $\c_{init}$ defined in the $(\Leftarrow)$-part of the proof of Lemma \ref{EL_Undec_Lem_PropertyOfRepresentingModel}. We have $\I\in g(v_{ts})$, $G^\I=\varnothing$, but $A^\I\neq\varnothing$, which means $\I\not\models A\dleq G$, a contradiction. 
\end{proof}

%For the ``if'' direction, as the initial configuration $c_{init}$ is not $m$-accepting for any $m$, by Lemmas \ref{EL_Undec_Lem_PropertyOfRepresentingModel} and \ref{EL_Undec_Lem_ReachableModel}, there is a model $\I\in g(v_{ts})$ which $(x,1)$-represents $c_{init}$ for an element $x$ such that $x\in A^\I$, $x\not\in H^\I, \ D^\I, \ G^\I$. Thus, we arrive at contradiction with  $(\D,v_{ts})\models A\dleq G$. \QED 

}%%%%%%% END OLD PROOFS OF EXPTIME AND UNDECIDABILITY FOR EL

%---------------------------------------%

\noindent\textbf{Theorem \ref{Teo_Hardness_ALC}.} \textit{Entailment in $\ \ALC$-ontology networks is $\DEXPTIME$-hard.}

\medskip

We prove the theorem by showing a reduction from the word problem for alternating Turing machines working with words of length $\IEXP{n}$. For $n\geqslant 0$, let $M=\langle Q,\mathcal{A},\delta_1, \delta_2 \rangle$ be such ATM and let $\O$ be an ontology consisting of the following axioms, which implement a computation of $M$.  %(42)-(43), (45)-(48), (51)-(54) in \cite{Kazakov:08:RIQ:SROIQ} and the axioms listed below. Axiom (\ref{Eq_ALCEXP_InitChain}) replaces axioms (4)-(8) in \cite{Kazakov:08:RIQ:SROIQ} which initialize an exponentially long $r$-chain ending with $E$. Axioms (\ref{Eq_ALCEXP_ExistentialState}-\ref{Eq_ALCEXP_UniversalState}) replace (49)-(50) in \cite{Kazakov:08:RIQ:SROIQ} and use concepts $V,H$ instead of roles $v,h$, respectively, to designate successor configurations. %They express the existence of successor configurations depending on the types of the configuration. 
%Axioms  (\ref{Eq_ALCEXP_PropagateV})-(\ref{Eq_ALCEXP_PropagateH}) propagate $V$ and $H$, which is needed for the implementation of transition functions of ATM given by axioms (\ref{Eq_ALCEXP_1Transition})-(\ref{Eq_ALCEXP_2Transition}) below (these axioms are modifications of (55)-(56) in \cite{Kazakov:08:RIQ:SROIQ}). Finally, axiom (\ref{Eq_ALCEXP_PropagateBlankSymbol}) is a modification of (44) in \cite{Kazakov:08:RIQ:SROIQ}) which propagates the blank symbol in the initial configuration. 

The first two axioms are used to initialize a $r$-chain (with the end marker $E$) used for `storing' configurations of the ATM:

\begin{equation}\label{Eq_ALCEXP_InitChain}
Z\dleq \exists (r,C)^{\IEXP{n}}.\ex{r}.E 
\end{equation}\vspace{-0.7cm}

\begin{equation}\label{Eq_ALCEXP_InitNegE}
C \dleq \neg E
\end{equation}%\vspace{-0.7cm}

%---------------------------

The next two axioms define a $r$-chain `storing' the initial configuration of the form $\b\qo\b\ldots\b$:

\begin{equation}\label{Eq_ALCEXP_InitialConfig1}
A\dleq Z\dcap \all{r}.\b \dcap \all{r}.\all{r}. \qo \dcap \all{r}.\all{r}.\all{r}.B
\end{equation}\vspace{-0.7cm}

\begin{equation}\label{Eq_ALCEXP_InitialConfig2}
B\dleq \b \dcap \all{r}.(E\dcup B)
\end{equation}%\vspace{-0.7cm}

%---------------------------

Axioms (\ref{Eq_ALCEXP_InitStateMarkers})-(\ref{Eq_ALCEXP_PropagateStateMarkers}) initialize markers $S_\exists, S_\forall$ for configuration types and propagate them to the end of a $r$-chain:

\begin{equation}\label{Eq_ALCEXP_InitStateMarkers}
q_\exists\dleq S_\exists,  \ \ q_\forall\dleq S_\forall
\end{equation}%\vspace{-0.7cm}
for all $q_\exists\in Q_\exists$ and $q_\forall\in Q_\forall$;

\begin{equation}\label{Eq_ALCEXP_PropagateStateMarkers}
\neg E \dcap S_\exists\dleq \all{r}.S_\exists, \ \ \neg E \dcap S_\forall\dleq \all{r}.S_\forall
\end{equation}%\vspace{-0.7cm}

%---------------------------

\noindent Axioms (\ref{Eq_ALCEXP_ExistentialState})-(\ref{Eq_ALCEXP_PropagateCi}) initialize labels $C_1,C_2$ to distinguish between successor configurations and enforce that every $r$-chain, which represents a $\exists$-configuration ($\forall$-configuration, respectively), has a subsequent $r$-chain (two subsequent $r$-chains, respectively) representing successor configuration(s): 

\begin{equation}\label{Eq_ALCEXP_ExistentialState}
E\dcap S_\exists \dleq \ex{r}.(Z\dcap C_1) \sqcup \ex{r}.(Z\dcap C_2)
\end{equation}\vspace{-0.7cm}

\begin{equation}\label{Eq_ALCEXP_UniversalState}
E\dcap S_\forall \dleq \ex{r}.(Z\dcap C_1) \sqcap \ex{r}.(Z\dcap C_2)
\end{equation}\vspace{-0.7cm}

\begin{equation}\label{Eq_ALCEXP_PropagateCi}
 C_\alpha \dleq \all{r}.(E \dcup C_\alpha), \ \ \ \alpha=1,2
\end{equation}%\vspace{-0.7cm}

%---------------------------

\medskip

Axioms (\ref{Eq_ALCEXP_PrepareTransition1})-(\ref{Eq_ALCEXP_PrepareTransition4}) initialize markers $S_{XYUV}$ ($X,Y,U,V\in{Q\cup\mathcal{A}}$), which encode 4-tuples of symbols from configurations, while respecting the end points of $r$-chains:

\begin{equation}\label{Eq_ALCEXP_PrepareTransition1}
X \dcap \ex{r}.(Y \dcap \ex{r}.(U \dcap \ex{r}.V)) \dleq \all{r}.\all{r}.S_{XYUV}
\end{equation}\vspace{-0.7cm}

\begin{equation}\label{Eq_ALCEXP_PrepareTransition2}
Z \dcap \ex{r}.(U \dcap \ex{r}.V) \dleq \all{r}.S_{\b\b UV}
\end{equation}\vspace{-0.9cm}

\begin{equation}\label{Eq_ALCEXP_PrepareTransition3}
Z \dcap \ex{r}.(Y \dcap \ex{r}.(U \dcap \ex{r}.V)) \dleq \all{r}.\all{r}.S_{\b YUV}
\end{equation}\vspace{-0.7cm}

\begin{equation}\label{Eq_ALCEXP_PrepareTransition4}
X \dcap \ex{r}.(Y \dcap \ex{r}.(U \dcap \ex{r}.E)) \dleq \all{r}.\all{r}.S_{XYU\b}
\end{equation}
for all $X,Y,U,V\in{Q\cup\mathcal{A}}$.

%---------------------------

\medskip

Finally, for $f(n)=\IEXP{n}+2$ and $\alpha=1,2$, axioms (\ref{Eq_ALCEXP_Transition}) implement transitions of $M$ by initializing label concepts on the corresponding successor $r$-chains, and axiom (\ref{Eq_ALCEXP_ForbidRejectingState}) forbids the rejecting state:

\begin{equation}\label{Eq_ALCEXP_Transition}
S_{XYUV} \dleq \all{r^{f(n)}}.(\neg C_\alpha \sqcup W)
\end{equation}
for all $X,Y,U,V,W\in{Q\cup\mathcal{A}}$ such that $XYUV\overset{\delta_\alpha}{\mapsto} W$;

%\begin{equation}\label{Eq_ALCEXP_2Transition}
%S_{XYUV} \dleq \all{r^{f(n)}}.(\neg C_2 \sqcup W)
%\end{equation}
%for all $X,Y,U,V,W\in{Q\cup\mathcal{A}}$ such that $XYUV\overset{\delta_2}{\mapsto} W$;
%---------------------------
\begin{equation}\label{Eq_ALCEXP_ForbidRejectingState}
q_r \dleq \bot
\end{equation}

\medskip

\begin{lemma}
It holds $\O\not\models A\dleq\bot$ iff $M$ accepts the empty word. 
\end{lemma}

\begin{proof}
We assume that every configuration of $M$ is a word of length $\IEXP{n}$ in the alphabet $Q\cup\mathcal{A}$. 

$(\Leftarrow):$ Let $\mathcal{C}$ be the set of configurations in an accepting run tree of $M$, with the root being the initial configuration $\c_0=\b\qo\b\ldots\b$. We define a model of ontology $\O$, in which the interpretation of concept $A$ is not empty. Let $\I=(\Delta, \cdot^\I)$ be an interpretation, with $\Delta=\{x_{\c,i} \mid \c\in\mathcal{C}, \ 0\leqslant i \leqslant \IEXP{n}+1\}$. Let us define $r^\I=\{\tuple{x_{\c,i},x_{\c,i+1}} \mid 0\leqslant i \leqslant \IEXP{n} \} \cup \{\tuple{x_{\c,\IEXP{n}+1},x_{\c',0}} \mid \c' \ \text{is a successor of} \ \c\}$. Further, set $A^\I=\{x_{\c_0,0}\}$, $Z^\I=\{x_{\c,0} \mid \c\in\mathcal{C}\}$, $E^\I=\{x_{\c,\IEXP{n}+1} \mid \c\in\mathcal{C}\}$, $C_k^\I=\{x_{\c,i} \mid 0\leqslant i \leqslant \IEXP{n}, \  \exists \c'\in\mathcal{C} \ \text{s.t.} \  \c \ \text{is a } \ \delta_\alpha\text{-successor of} \ \c'\}$, for $\alpha=1,2$, and for all $X\in Q\cup\mathcal{A}$, set $X^\I=\{x_{\c,i} \mid \c[i]=X\}$. For all $X,Y,U,V\in Q\cup\mathcal{A}$, concepts $S_{XYUV}$ are interpreted in a clear way. Finally, let $C^\I=\{x_{\c,i} \mid 1 \leqslant i \leqslant \IEXP{n}\}$, $B^\I=\{x_{\c_0,i} \mid i\in\{1\}\cup \{3,\ldots, \IEXP{n}\}\}$, $S_\exists^\I=\{x_{\c,i} \mid k \leqslant i \leqslant \IEXP{n}+1, \ \c[k]\in Q_\exists \}$,  $S_\forall^\I=\{x_{\c,i} \mid k \leqslant i \leqslant \IEXP{n}+1, \ \c[k]\in Q_\forall \}$, and $q_r^\I=\varnothing$. Since every configuration from $\mathcal{C}$ is accepting, $q_r$ appears in no $\c\in\mathcal{C}$ and it is straightforward to verify that $\I$ is a model of $\O$. 

$(\Rightarrow):$ Let $\I=(\Delta, \cdot^\I)$ be a model of $\O$ and $x\in\Delta$ an element such that $x\in A^\I$. We say that a segment $x_1, \ldots x_{\IEXP{n}+1}$ of a $r$-chain in $\I$ represents a configuration $\c=uqw$ if $x_i\in \c[i]$, for all $1\leqslant i \leqslant \IEXP{n}$, $x_{\IEXP{n}+1}\in E^\I$, and $x_{\IEXP{n}+1}\in S_\exists^\I$ if $q\in Q_\exists$ and $x_{\IEXP{n}+1}\in S_\forall^\I$ if $q\in Q_\forall$. We show how to use $\I$ to define an accepting run tree of $M$ in which for every configuration $\c$ there is a segment in $\I$, which represents $\c$. We use induction on the height of the accepting run tree. For the induction base, observe that by axioms (\ref{Eq_ALCEXP_InitChain})-(\ref{Eq_ALCEXP_PropagateStateMarkers}), since $x\in A^\I$, there is a $r$-chain outgoing from $x$, which contains a segment representing the initial configuration $\c_0$. We set $\c_0$ to be the root of the tree. In the induction step, consider an arbitrary configuration $\c=uqw$ being a leaf in the tree constructed so far. By the induction assumption, there is a segment $x_1, \ldots x_{\IEXP{n}+1}$ in $\I$ which represents $\c$. If $q\in Q_\exists$ then $x_{\IEXP{n}+1}\in S_\exists^\I$ and by axioms (\ref{Eq_ALCEXP_ExistentialState}), (\ref{Eq_ALCEXP_InitChain}) there exist $y_0, \ldots y_{\IEXP{n}+1}\in\Delta$ such that $\tuple{x_{\IEXP{n}+1}, y_0}\in r^\I$ and  $\tuple{y_i,y_{i+1}}\in r^\I$, for $0\leqslant i \leqslant \IEXP{n}$. Then by axioms (\ref{Eq_ALCEXP_PropagateCi})-(\ref{Eq_ALCEXP_Transition}), there is a successor configuration $\c'$ of $\c$ such that $y_i\in \c'[i]$, for $1 \leqslant i \leqslant \IEXP{n}$. Due to axiom (\ref{Eq_ALCEXP_InitChain}) we have $y_{\IEXP{n}+1}\in E^\I$ and by axioms (\ref{Eq_ALCEXP_InitStateMarkers})-(\ref{Eq_ALCEXP_PropagateStateMarkers}) it holds $y_{\IEXP{n}+1}\in S_\exists^\I$ if $q\in Q_\exists$ and $y_{\IEXP{n}+1}\in S_\forall^\I$ if $q\in Q_\forall$, i.e., the segment $y_1, \ldots y_{\IEXP{n}+1}$ represents $\c'$. Since $\I$ is a model of axiom (\ref{Eq_ALCEXP_ForbidRejectingState}), $\c'$ is an accepting configuration. We extend the tree by adding a child node $\c'$ for the node $\c$. Similarly, in case $q\in Q_\forall$ we have $x_{\IEXP{n}+1}\in S_\forall^\I$ and hence by axioms (\ref{Eq_ALCEXP_UniversalState})-(\ref{Eq_ALCEXP_Transition}), there exist two segments in $\I$, each representing an accepting successor configuration $\c'_i$ of $\c$, for $i=1,2$. Then we extend the tree by adding child nodes $\c'_i$, $i=1,2$, for the node $\c$.
\end{proof}

To conclude the proof of Theorem \ref{Teo_Hardness_ALC} let us show that ontology $\O$ is expressible by an acyclic $\ALC$-ontology network of size polynomial in $n$. Note that $\O$ contains axioms (\ref{Eq_ALCEXP_InitChain}), (\ref{Eq_ALCEXP_Transition}) with concepts of size exponential in $n$. Note that by Lemma \ref{Lem_Expressibility_EL+dleqExp}, axiom (\ref{Eq_ALCEXP_InitChain}) is expressible by an acyclic $\EL$-ontology network of size polynomial in $n$. Now consider an axiom $\varphi$ of the form (\ref{Eq_ALCEXP_Transition}) 

\begin{equation*}
S_{XYUV} \dleq \forall r^{\IEXP{n}}.\forall r.\forall r. (\neg C_\alpha\dcup W)
\end{equation*}

where $\alpha\in\{1,2\}$, and a concept inclusion $\psi$ defined as

\begin{equation*}
S_{XYUV} \dleq D
\end{equation*}
where $D$ is a concept name. By Lemma \ref{Lem_Expressibility_ForallExpSubstitution}, $\psi[D\mapsto \forall r^{\IEXP{n}}.\forall r.\forall r. (\neg C_\alpha\dcup W)]$ is expressible by an acyclic $\ALC$-ontology network of size polynomial in $n$ and hence, so is $\varphi$. 

We conclude that each of axioms (\ref{Eq_ALCEXP_InitChain}), (\ref{Eq_ALCEXP_Transition}) is expressible by an acyclic $\ALC$-ontology network of size polynomial in $n$. The remaining axioms of $\O$ are $\ALC$ axioms, whose size does not depend on $n$. Then by applying Lemma \ref{Lem_IteratedExpressibility} we obtain that there exists an acyclic $\ALC$-ontology network  $\N$ of size polynomial in $n$ and an ontology $\O_\N$ such that $\O$ is $(\N,\O_\N)$-expressible and thus, it holds $\O_\N\models_\N A\dleq \bot$ iff $M$ does not accept the empty word. Theorem \ref{Teo_Hardness_ALC} is proved.

\bigskip

\noindent\textbf{Theorem \ref{Teo_Hardness_ALCHOIF}.} \textit{Entailment in $\ALCHOIF$-ontology networks is $\coNDEXPTIME$-hard.}

\medskip

The theorem is proved by a reduction from the complement of the (bounded) domino tiling problem. A \emph{domino system} is a triple $\D=(T,V,H)$, where $T=\{1,\ldots ,p\}$ is a finite set of tiles and $H,V\subseteq T\times T$ are horizontal and vertical matching relations. A \emph{tiling} of size $m\times m$ for a domino system $\D$ with initial condition $c^0=\tuple{t_1^0,\ldots , t_k^0}$, where $t_i^0\in T$, for $1\leqslant i \leqslant k$, is a mapping $t:\{1,\ldots ,m\}\times\{1,\ldots ,m\}\rightarrow T$ such that $\tuple{t(i-1,j), \ t(i,j)}\in V$, for $1 < i \leqslant m$, $1 \leqslant j \leqslant m$, $\tuple{t(i,j-1), \ t(i,j)}\in H$, for $1 \leqslant i \leqslant m$, $1 < j \leqslant m$, and $t(1,j)=t_j^0$, for $1\leqslant j \leqslant k$. It is well known that it is $\NDEXPTIME$-complete to decide whether a domino system admits a tiling of size $\IIEXP{n}\times\IIEXP{n}$, $n\geqslant 0$, with an initial condition $c^0$.

Let $\D$ be a domino system and $c^0=\tuple{t_1^0,\ldots , t_n^0}$ an initial condition. We define an ontology $\O$ consisting of the axioms, which encode the tiling problem for $\D$ with $c^0$ using a grid of dimension $\IIEXP{n}\times \IIEXP{n}$, $n\geqslant 0$, to be `tiled'. The first axiom of $\O$ defines the initial point of the grid, while the second one initializes a $r$-chain (having the end marker $E$ and containing $\IEXP{n}+1$ many points) used to represent a sequence of $\IEXP{n}$ many bits for a binary counter:

\begin{equation}\label{Eq_ALCOIFEXP_InitGrid}
A\dleq Z \dcap \forall r. (Z_v \dcap Z_h)
\end{equation}

\begin{equation}\label{Eq_ALCOIFEXP_InitZChain}
Z\dleq \exists r^{\IEXP{n}}.E
\end{equation}

%---------------------

The next axioms set the bits given by $X,Y$ on the $r$-chain outgoing from $A$ to `zero':

\begin{equation}\label{Eq_ALCOIFEXP_InitZeroXY}
Z_v \dleq \neg X \dcap \all{r}.Z_v, \ Z_h \dleq \neg Y \dcap \all{r}.Z_h
\end{equation}

%---------------------

The following three axioms define markers $E_v$ and $E_h$, which `hold' at the end of a $r$-chain, iff all the bits (except possibly the last bit) given by $X$, respectively by $Y$, are 1:

\begin{align}
Z \dleq  & \ \forall r.(E_v \dcap E_h)  \label{Eq_ALCOIFEXP_InitEmarker}  \\
E_v \dcap X \dleq \all{r}.E_v, & \ \ \ \neg (E_v \dcap X) \dleq \all{r}.\neg E_v  \label{Eq_ALCOIFEXP_PropagateVerticalEmarkers} \\
E_h \dcap Y \dleq \all{r}.E_h, & \  \ \ \neg (E_h \dcap Y) \dleq \all{r}.\neg E_h \label{Eq_ALCOIFEXP_PropagateHorizontalEmarker}
\end{align}

%---------------------

The next axioms say that every $r$-chain `containing' at least one zero bit has a `vertical' and `horizontal' successor $r$-chains:

\begin{equation}\label{Eq_ALCOIFEXP_InitZVertically}
E\dcap \neg(E_v \dcap X) \dleq \ex{v}.Z
\end{equation}

\begin{equation}\label{Eq_ALCOIFEXP_InitZHorizontally}
E\dcap \neg(E_h \dcap Y) \dleq \ex{h}.Z
\end{equation}

%The axioms also enforce that the very first element in the $r$-chain outgoing from $Z$ is labeled with $X$ ($Y$,respectively) and hence, this value does not influence whether there exists a vertical or horizontal successor $r$-chain or not.

%---------------------

Axioms (\ref{Eq_ALCOIFEXP_XFlipConditions})-(\ref{Eq_ALCOIFEXP_YFlipConditions}) initialize markers $X^f,Y^f$ corresponding to the flipping conditions of a binary counter:

\begin{align}
E \dcup \ex{r}.(X\dcap X^f) \dleq X^f, & \ \  \ex{r}.\neg (X\dcap X^f) \dleq \neg X^f  \label{Eq_ALCOIFEXP_XFlipConditions} \\
E \dcup \ex{r}.(Y\dcap Y^f) \dleq Y^f, & \ \ \ex{r}.\neg (Y\dcap Y^f) \dleq \neg Y^f \label{Eq_ALCOIFEXP_YFlipConditions}
\end{align}

Axioms (\ref{Eq_ALCOIFEXP_RoleHierarchy}) define the role hierarchy needed for a correct incrementation of binary counters across role chains:

\begin{equation}\label{Eq_ALCOIFEXP_RoleHierarchy}
r \dleq v, \ \ r\dleq h
\end{equation}

For $f(n)=\IEXP{n}+1$, axioms (\ref{Eq_ALCOIFEXP_FlipX})-(\ref{Eq_ALCOIFEXP_NoFlipY}) define the values of bits in the vertical/horizontal successor $r$-chains given the values of the flipping markers:

\begin{equation}\label{Eq_ALCOIFEXP_FlipX}
X^f \dleq (X \dcap \forall v^{f(n)}.\neg X) \dcup (\neg X \dcap \forall v^{f(n)}.X)
\end{equation}\vspace{-0.9cm}

\begin{equation}\label{Eq_ALCOIFEXP_NoFlipX}
\neg X^f\!\dleq\!(X \dcap \forall v^{f(n)}. X) \dcup (\neg X \dcap \forall v^{f(n)}.\neg X)
\end{equation}\vspace{-0.9cm}

\begin{equation}\label{Eq_ALCOIFEXP_FlipY}
Y^f \dleq (Y \dcap \forall h^{f(n)}.\neg Y) \dcup (\neg Y \dcap \forall h^{f(n)}.Y)
\end{equation}\vspace{-0.9cm}

\begin{equation}\label{Eq_ALCOIFEXP_NoFlipY}
\neg Y^f \dleq (Y \dcap \forall h^{f(n)}. Y) \dcup (\neg Y \dcap \forall h^{f(n)}.\neg Y)
\end{equation}

%---------------------

The next four axioms enforce the grid structure composed by vertical and horizontal $r$-chains: for $f(n)=\IEXP{n}+1$, axioms (\ref{Eq_ALCOIFEXP_GridPreserveX})-(\ref{Eq_ALCOIFEXP_GridPreserveY}) propagate the bit values given by $X$ horizontally and those given by $Y$ vertically; axiom (\ref{Eq_ALCOIFEXP_GridUniqueEnd}) states that there is a unique common final element of the vertical and horizontal $r$-chains, in which all bits are 1; axiom (\ref{Eq_ALCOIFEXP_GridFuncRoles}) states that the roles $v,h$ are inverse functional:

\begin{align}
\top & \dleq (X \dcap \forall h^{f(n)}.X) \dcup (\neg X \dcap \forall h^{f(n)}.\neg X) \label{Eq_ALCOIFEXP_GridPreserveX} \\
\top & \dleq (Y \dcap \forall v^{f(n)}.Y) \dcup (\neg Y \dcap \forall v^{f(n)}.\neg Y) \label{Eq_ALCOIFEXP_GridPreserveY}
\end{align}\vspace{-0.7cm}

\begin{equation}\label{Eq_ALCOIFEXP_GridUniqueEnd}
E \dcap E_v \dcap X \dcap E_h \dcap Y \dleq \{a\}
\end{equation}\vspace{-0.7cm}

\begin{equation}\label{Eq_ALCOIFEXP_GridFuncRoles}
\Fun(v^-), \ \ \Fun(h^-)
\end{equation}

%---------------------

Finally, axioms (\ref{Eq_ALCOIFEXP_Tiles})-(\ref{Eq_ALCOIFEXP_TilesInitialCond3}), where $f(n)=\IEXP{n}+1$, declare tile types $D_1,\ldots ,D_p$ stating that there is a unique tile for every element labelled by $E$, and define tiling conditions as well as the initial tiling given by tile types $D_{t^0_k}$, $1 \leqslant k < m $:

\begin{align}
E & \dleq D_1 \dcup \ldots \dcup D_p \label{Eq_ALCOIFEXP_Tiles} \\
D_i \dcap D_j & \dleq\bot \hspace{2.2cm} 1\leqslant i < j \leqslant p   \label{Eq_ALCOIFEXP_TilesDisjoint} \\
D_i \dleq \all{v^{f(n)}}.D_i^v & \dcap \all{h^{f(n)}}.D_i^h \hspace{1cm}  1\leqslant i \leqslant p \label{Eq_ALCOIFEXP_PropagateTile} \\
D_i \dcap D_j^v & \dleq \bot \hspace{2.8cm} \langle i, j \rangle \not\in V  \label{Eq_ALCOIFEXP_TilesVMatching} \\
D_i \dcap D_j^h & \dleq \bot \hspace{2.8cm} \langle i, j \rangle \not\in H \label{Eq_ALCOIFEXP_TilesHMatching} \\
E \dcap Z_v \dcap Z_h & \dleq I_1   \label{Eq_ALCOIFEXP_TilesInitialCond1} \\
I_j & \dleq \all{h^{f(n)}}.I_{j+1} \hspace{1cm} 1 \leqslant j < k   \label{Eq_ALCOIFEXP_TilesInitialCond2}\\
I_j & \dleq D_{t_j^0} \hspace{2.2cm} 1 \leqslant j \leqslant k   \label{Eq_ALCOIFEXP_TilesInitialCond3}
\end{align}

%---------------------

Let us prove two auxiliary lemmas. The first one shows that the axioms of $\O$ enforce that every model $\I\models O$, in which $A^\I\neq\emptyset$, contains a structure `implementing' a grid of size $\IIEXP{n}\times\IIEXP{n}$. The second lemma demonstrates the reduction of the tiling problem to entailment from $\O$. Finally we show that $\O$ is expressible by an acyclic $\ALCHOIF$-ontology network of size polynomial in $n$, which proves the Theorem. 

\medskip

First, we introduce some auxiliary notations. For a natural number $m\geqslant 0$, we denote by $i[m]_2$ the value of the $i$-th bit in the binary representation of $m$. Let $(\Delta, \cdot^\I)$ be an interpretation, $r$ a role, and $X$, $Y$ concept names. For $x,y\in\Delta$ and $k\geqslant 1$, the notation $x[r]^k y$ means that there is a sequence of elements $x_1,\ldots x_{k+1}\in\Delta$ such that $x_1=y$, $x_{k+1}=x$, and $\tuple{x_{l+1},x_l}\in r^\I$, for $1\leqslant l \leqslant k$. We say that an element $y\in\Delta$ \emph{represents} a tuple $\tuple{i,j}$, where $1 \leqslant i,j \leqslant m$, for $m=\IIEXP{n}$ and $n\geqslant 0$, if there is $x\in\Delta$ such that $x[r]^k y$, for $k=\IEXP{n}$, i.e., there is a sequence of elements $x_1,\ldots x_{k+1}\in\Delta$ such that $x_1=y$, $x_{k+1}=x$ and $\tuple{x_{l+1},x_l}\in r^\I$, for $1\leqslant l \leqslant k$, and it holds:
\begin{itemize}
\item[] $x_l\in X^\I$ iff $l[i]_2=1$ and $x_l\in Y^\I$ iff $l[j]_2=1$.
\end{itemize}

A subset $X\subseteq\Delta$ represents a tuple $\tuple{i,j}$ if so does every element $x\in X$. 

\medskip

\begin{lemma}\label{Lem_ALCHOIF_GridModel}
For any model $(\Delta,\cdot^\I)$ of ontology $\O$ and any $x\in A^\I$ there exist elements $x_{i,j}\in\Delta$, $1\leqslant i,j \leqslant \IIEXP{n}$ such that 
\begin{itemize}
\item $x_{i,j}\in E^\I$;
\item $x_{1,1}\in (Z_v \dcap Z_h)^\I$;
\end{itemize}
and for any $y\in\Delta$, it holds: 
\begin{itemize}
\item $x[r]^{\IEXP{n}}y$ iff $y=x_{1,1}$;
\item $x_{i,j}[h]^{\IEXP{n}+1}y$ iff $y=x_{i,j+1}$;
\item $x_{i,j}[v]^{\IEXP{n}+1}y$ iff $y=x_{i+1,j}$. 
\end{itemize}
\end{lemma}
 
\begin{proof}
We use induction on $1\leqslant i,j \leqslant \IIEXP{n}$ and define non-empty sets $X_{i,j}$ satisfying the following properties:

\begin{itemize}
\item[(a)] $X_{1,1}=\{y\in\Delta \mid x[r]^{\IEXP{n}} y \ \text{and} \ y\in (E\dcap Z_v \dcap Z_h)^\I\}$;
\item[(b)] every $X_{i,j}$ represents $\tuple{i,j}$;
\item[(c)] $\forall x_{i,j}\in X_{i,j}$ it holds $x_{i,j}\in E^\I$ and $x_{i,j}\in (\neg (E_v \dcap X))^\I$ iff $i<\IIEXP{n}$ and $x_{i,j}\in (\neg (E_h \dcap Y))^\I$ iff $j<\IIEXP{n}$;
\item[(d)] $\forall x_{i-1,j}\in X_{i-1,j} \ \exists x_{i,j}\in X_{i,j}$ and $\forall x_{i,j}\in X_{i,j} \ \exists x_{i-1,j}\in X_{i-1,j}$ such that $x_{i-1,j}[v]^{\IEXP{n}}x_{i,j}$, when $i\geqslant 2$;
\item[(e)] $\forall x_{i,j-1}\in X_{i,j-1} \ \exists x_{i,j}\in X_{i,j}$ and $\forall x_{i,j}\in X_{i,j} \ \exists x_{i,j-1}\in X_{i,j-1}$ such that $x_{i,j-1}[h]^{\IEXP{n}}x_{i,j}$, when $j\geqslant 2$.
\end{itemize}

After that we show that the axioms of $\O$ enforce that every $X_{i,j}$ is a singleton, which proves the lemma. Initially we let every $X_{i,j}$ be equal to the empty set.  

In the induction base, for $i=j=1$, note that since $x\in A^\I$, by axioms (\ref{Eq_ALCOIFEXP_InitGrid})-(\ref{Eq_ALCOIFEXP_InitZeroXY}), there is a sequence of elements $x_1,\ldots x_{k+1}$, with $k=\IEXP{n}$ and $x_1=x$, such that $\tuple{x_m,x_{m+1}}\in r^\I$, for $1\leqslant m \leqslant k$, and $x_m\not\in X^\I\cup Y^\I$, for $2\leqslant m \leqslant k+1$, and $x_{k+1}\in (E\dcap Z_v \dcap Z_h)^\I$. We put every element $x_{k+1}$ of such sequence into the set $X_{1,1}$. By definition, $X_{1,1}$ satisfies conditions (a),(b) and it follows from axioms (\ref{Eq_ALCOIFEXP_InitEmarker})-(\ref{Eq_ALCOIFEXP_InitZHorizontally}) that $X_{1,1}$ satisfies condition (c) as well. Conditions (d),(e) are trivially satisfied, since $i=j=1$.

In the induction step, for $3 \leqslant i+j \leqslant 2\cdot \IIEXP{n}$, we assume that the statement is proved for all sets $X_{i',j'}$, with $i'+j' < i+j$. 

If $i\geqslant 2$, consider the set $X_{i-1,j}$. Since $i-1<\IIEXP{n}$, by the induction assumption, for every $x\in X_{i-1,j}$, we have $x\in [E\dcap \neg(E_v\dcap X)]^\I$. Hence by axioms (\ref{Eq_ALCOIFEXP_InitZVertically}), (\ref{Eq_ALCOIFEXP_InitZChain}), for every $x\in X_{i-1,j}$ there exists a sequence of elements $x_1,\ldots x_{k+1}$, with $k=\IEXP{n}$, such that $\tuple{x,x_1}\in v^\I$, $\tuple{x_m,x_{m+1}}\in r^\I$, for $1\leqslant m \leqslant k$, and $x_{k+1}\in E^\I$. We put every element $x_{k+1}$ into the set $X_{i,j}$. By axiom (\ref{Eq_ALCOIFEXP_RoleHierarchy}) and the definition of $X_{i,j}$, condition (d) holds for $X_{i,j}$. By the induction assumption, $X_{i-1,j}$ represents $\tuple{i-1,j}$ and hence, by axioms (\ref{Eq_ALCOIFEXP_XFlipConditions}), (\ref{Eq_ALCOIFEXP_RoleHierarchy}), (\ref{Eq_ALCOIFEXP_GridPreserveY}), $X_{i,j}$ represents $\tuple{i,j}$. i.e. $X_{i,j}$ satisfies condition (b). Moreover, it is easy to see that by axioms (\ref{Eq_ALCOIFEXP_InitEmarker}), (\ref{Eq_ALCOIFEXP_PropagateVerticalEmarkers}), for every $x\in X_{i,j}$, it holds $x\in (\neg(E_v\dcap X))^\I$ iff $i<\IIEXP{n}$ and $x\in (\neg(E_h\dcap Y))^\I$ iff $j<\IIEXP{n}$, i.e., $X_{i,j}$ also satisfies (c). 

Similarly, if $j\geqslant 2$, we consider the set $X_{i,j-1}$. Since $j-1<\IIEXP{n}$, by the induction assumption, for every $x\in X_{i,j-1}$, it holds $x\in [E\dcap \neg(E_h\dcap X)]^\I$. Hence, by axioms (\ref{Eq_ALCOIFEXP_InitZHorizontally}), (\ref{Eq_ALCOIFEXP_InitZChain}), for every $x\in X_{i,j-1}$ there exists a sequence of elements $x_1,\ldots x_{k+1}$, with $k=\IEXP{n}$, such that $\tuple{x,x_1}\in h^\I$, $\tuple{x_m,x_{m+1}}\in r^\I$, for $1\leqslant m \leqslant k$, and $x_{k+1}\in E^\I$. We put every element $x_{k+1}$ into the set $X_{i,j}$. By axiom (\ref{Eq_ALCOIFEXP_RoleHierarchy}) and the definition of $X_{i,j}$, condition (e) holds for $X_{i,j}$. By the induction assumption, $X_{i,j-1}$ represents $\tuple{i,j-1}$ and hence, by axioms (\ref{Eq_ALCOIFEXP_YFlipConditions}), (\ref{Eq_ALCOIFEXP_RoleHierarchy}), (\ref{Eq_ALCOIFEXP_GridPreserveX}), $X_{i,j}$ represents $\tuple{i,j}$, i.e. $X_{i,j}$ satisfies condition (b). Moreover, it is easy to see that by axioms (\ref{Eq_ALCOIFEXP_InitEmarker}), (\ref{Eq_ALCOIFEXP_PropagateHorizontalEmarker}), for every $x\in X_{i,j}$, it holds $x\in (\neg(E_v\dcap X))^\I$ iff $i<\IIEXP{n}$ and $x\in (\neg(E_h\dcap Y))^\I$ iff $j<\IIEXP{n}$, i.e., $X_{i,j}$ also satisfies (c). 

Now let us show by induction on $1\leqslant i,j \leqslant \IIEXP{n}$ that every set $X_{i,j}$ is a singleton. In the induction base for $i=j=\IIEXP{n}$, by condition (c), for every $x\in X_{i,j}$ we have $x\in (E\dcap E_v \dcap X \dcap E_h \dcap Y)^\I$.  Hence, it follows from axiom (\ref{Eq_ALCOIFEXP_GridUniqueEnd}) that the set $X_{i,j}$ is a singleton. In the induction step we assume that the statement is shown for all $X_{i',j'}$, with $i'+j' > i+j$. If $i<\IIEXP{n}$ then consider the set $X_{i+1,j}$. Since, $X_{i+1,j}$ satisfies property (d), it follows from the inverse functionality of $v$ that $X_{i,j}$ is a singleton. If $j<\IIEXP{n}$ then we consider the set $X_{i,j+1}$. Since, $X_{i,j+1}$ satisfies property (e), it follows from the inverse functionality of $h$ that $X_{i,j}$ is a singleton. 
\end{proof}

\begin{lemma}\label{Lem_ALHOIF_ReductionTiling2Entailment}
It holds $\O\not\models A\dleq \bot$ iff the domino system $\D$ admits a tiling of size $\IIEXP{n}\times\IIEXP{n}$. 
\end{lemma}

\begin{proof}
($\Rightarrow$): Let $(\Delta, \cdot^\I)$ be a model of ontology $\O$ and $x$ an element such that $x\in A^\I$. Then there exist elements $x_{i,j}\in\Delta$, for $1\leqslant i,j \leqslant \IIEXP{n}$, having the properties as in Lemma \ref{Lem_ALCHOIF_GridModel}. By axioms (\ref{Eq_ALCOIFEXP_Tiles})-(\ref{Eq_ALCOIFEXP_TilesDisjoint}), for every $x_{i,j}$ there is a unique $D_k$, $1 \leqslant k \leqslant p$, such that $x_{i,j}\in D_k^\I$. By axioms (\ref{Eq_ALCOIFEXP_TilesInitialCond1})-(\ref{Eq_ALCOIFEXP_TilesInitialCond3}) and the property of $x_{1,1}$ from Lemma \ref{Lem_ALCHOIF_GridModel}, we have $x_{1,j}\in D_{t_j^0}$, for $1 \leqslant j \leqslant k$. Finally, axioms (\ref{Eq_ALCOIFEXP_PropagateTile})-(\ref{Eq_ALCOIFEXP_TilesHMatching}) enforce that for any elements $x_{i,j}$, $x_{i',j'}$, with $1 \leqslant i,j,i',j' \leqslant \IIEXP{n}$, any $D_k$, $D_l$ such that $1 \leqslant k,l \leqslant p$, $x_{i,j}\in D_k$, and $x_{i',j'}\in D_l$, we have $i'=i+1$ iff  $\tuple{k,l}\in V$ and $j'=j+1$ iff $\tuple{k,l}\in H$. Therefore, we conclude that the domino system $\D$ admits a tiling. 

($\Leftarrow$): For $m=\IIEXP{n}$, let $t:\{1,\ldots , m\}\times\{1,\ldots , m\}\rightarrow T$ be a tiling for $\D$. We define a model $\I=(\Delta, \cdot^\I)$ of $\O$, in which the interpretation of $A$ is non-empty. Let $\Delta=\{x_{k,i,j} \mid 1 \leqslant k \leqslant \IEXP{n}+1, \ 1 \leqslant i,j \leqslant \IIEXP{n}\}$ and define interpretation of roles $r,v,h$ as follows:
\begin{itemize}
\item $r^\I=\{\tuple{x_{k-1,i,j}, x_{k,i,j}} \mid 2 \leqslant k\}$; 
\item $v^\I=r^\I\cup\{\tuple{x_{k,i-1,j}, x_{1,i,j}} \mid k = \IEXP{n}+1, \ 2 \leqslant i\}$; 
\item $h^\I=r^\I\cup\{\tuple{x_{k,i,j-1}, x_{1,i,j}} \mid k = \IEXP{n}+1, \ 2 \leqslant j\}$. 
\end{itemize}

Further, we set $A^\I=\{x_{1,1,1}\}$, $Z^\I=\{x_{1,i,j} \mid 1 \leqslant i,j\}$, $E^\I=\{x_{k,i,j} \mid k=\IEXP{n}+1\}$, $Z_v^\I=Z_h^\I=\{x_{k,1,1} \mid 2\leqslant k \leqslant \IEXP{n}+1\}$, $X^\I=\{x_{\IEXP{n}+2-k,i,j} \mid k[i]_2=1\}$, $Y^\I=\{x_{\IEXP{n}+2-k,i,j} \mid k[j]_2=1\}$, and $a^\I=x_{k,i,j}$, for $k=\IEXP{n}+1$ and $i=j=\IIEXP{n}$. Finally, we define $D_l^\I=\{x_{k,i,j} \mid k=\IEXP{n}+1, \ t(i,j)=l\}$, for $1\leqslant l \leqslant p$, and $I_j^\I=\{x_{i,1,j} \mid i=\IEXP{n}+1\}$, for $1\leqslant j \leqslant k$. Other concepts $X^f, Y^f, E_v, E_h$, and $D_i^v$, $D_i^h$, for $1 \leqslant i \leqslant p$, are interpreted in a clear way. It is straightforward to verify that $\I$ is a model of every axiom of $\O$.
\end{proof}

To complete the proof of Theorem \ref{Teo_Hardness_ALCHOIF} let us show that ontology $\O$ is expressible by an acyclic $\ALCHOIF$-ontology network of size polynomial in $n$. Note that $\O$ contains axioms (\ref{Eq_ALCOIFEXP_InitZChain}), (\ref{Eq_ALCOIFEXP_FlipX})-(\ref{Eq_ALCOIFEXP_GridPreserveY}), (\ref{Eq_ALCOIFEXP_PropagateTile}) with concepts of size exponential in $n$. By Lemma \ref{Lem_Expressibility_EL+dleqExp}, axiom (\ref{Eq_ALCOIFEXP_InitZChain}) is expressible by an acyclic $\EL$-ontology network of size polynomial in $n$. Now consider an axiom $\varphi$ of the form (\ref{Eq_ALCOIFEXP_FlipX}): \vspace{-0.2cm}

\begin{equation*}
X^f \dleq (X \dcap \forall v^{\IEXP{n}}.\forall v. \neg X) \dcup (\neg X \dcap \forall v^{\IEXP{n}}.\forall v. X)
 \end{equation*}
 
Let $\psi$ be concept inclusion of the form

\begin{equation*}
X^f \dleq (X \dcap \bar{B}) \dcup (\neg X \dcap B)
 \end{equation*}
where $\bar{B},B$ are concept names. By Lemma \ref{Lem_Expressibility_ForallExpSubstitution}, $\psi[\bar{B}\mapsto \forall v^{\IEXP{n}}.\forall v. \neg X, \ B\mapsto \forall v^{\IEXP{n}}.\forall v. X]$ is expressible by an acyclic $\ALC$-ontology network of size polynomial in $n$ and hence, so is $\varphi$. The expressibility of axioms of the form (\ref{Eq_ALCOIFEXP_FlipY})-(\ref{Eq_ALCOIFEXP_GridPreserveY}), and (\ref{Eq_ALCOIFEXP_PropagateTile}) is proved analogously. 
 
We conclude that each of the axioms (\ref{Eq_ALCOIFEXP_InitZChain}), (\ref{Eq_ALCOIFEXP_FlipX})-(\ref{Eq_ALCOIFEXP_GridPreserveY}), (\ref{Eq_ALCOIFEXP_PropagateTile})-(\ref{Eq_ALCOIFEXP_TilesHMatching}) is expressible by an acyclic $\ALC$-ontology network of size polynomial in $n$. The remaining axioms of $\O$ are $\ALCHOIF$ axioms, whose size does not depend on $n$. Then by applying Lemma \ref{Lem_IteratedExpressibility} we obtain that there exists an acyclic $\ALCHOIF$-ontology network $\N$ of size polynomial in $n$ and an ontology $\O_\N$ such that $\O$ is $(\N,\O_\N)$-expressible and thus, it holds $\O_\N\models_\N A\dleq \bot$ iff the domino system $\mathcal{D}$ does not admit a tiling of size $\IIEXP{n}\times\IIEXP{n}$. Theorem \ref{Teo_Hardness_ALCHOIF} is proved.
 
\bigskip

\noindent\textbf{Theorem \ref{Teo_Hardness_HGeneral}.} \textit{Entailment in $\mathcal{H}$-Networks is $\EXPTIME$-hard.}

\begin{proof} We show that the word problem for ATMs working with words of a polynomial length $n$ reduces to entailment in cyclic $\mathcal{H}$-ontology networks. Then, since $\APSPACE=\EXPTIME$, the claim follows.

Let $M=\langle Q,\mathcal{A},\delta_1,\delta_2 \rangle$ be a ATM. We call the word of the form $\b\qo\b\ldots\b$ %For the purpose of the proof, let \emph{configuration} of $M$ be a word of length $m=2n+4$ in the alphabet $Q\cup \mathcal{A}$.  Then, given a configuration $\c$, the notion of successor configuration is naturally induced by $\delta_1$ and $\delta_2$. We call the word of the form
%\begin{equation*}
%\word{\b\ldots \b}{n+2 \ \text{times}}{\qo}{\b\ldots\b}{n+1 \ \text{times}}
%\end{equation*}
\emph{initial configuration} $\c_{\init}$ of $M$. %and for a configuration $\c$ we denote by $\c[j]$ the $j$-th symbol in $\c$. 

Consider signature $\sigma$ consisting of concept names $B_{ai}$, for $a\in Q\cup{\mathcal{A}}$ and $1\leqslant i \leqslant n$ (with the informal meaning that the i-th symbol in a configuration of $M$ is $a$). Let $\sigma^1$, and $\sigma^2$ be `copies' of signature $\sigma$ consisting of the above mentioned concept names with the superscripts $^1$ and $^2$, respectively.

For $\alpha=1,2$, let $\O^\alpha$ be an ontology consisting of the following axioms:

\begin{equation}\tag{\ref{Eq_Step1}}
B^{\alpha}_{Xi-2} \dcap B^{\alpha}_{Yi-1} \dcap B^{\alpha}_{Ui} \dcap B^{\alpha}_{Vi+1}\  \dleq \ B_{Wi}
\end{equation}
for $3\leqslant i \leqslant n-1$ and all $X,Y,U,V,W\in Q\cup\mathcal{A}$ such that $XYUV\overset{\delta_{\alpha}}{\mapsto} W$; \vspace{-0.2cm}

\begin{equation}\tag{\ref{Eq_Step2}}
B^{\alpha}_{U1} \dcap B^{\alpha}_{V2}\  \dleq \ B_{W 1}
\end{equation}
for all $U,V,W\in Q\cup\mathcal{A}$ such that $\b\b UV\overset{\delta_{\alpha}}{\mapsto} W$;

\begin{equation}\tag{\ref{Eq_Step3}}
B^{\alpha}_{Y1} \dcap B^{\alpha}_{U2} \dcap B^{\alpha}_{V3}\ \dleq \ B_{W 2}
\end{equation}
for all $Y,U,V,W\in Q\cup\mathcal{A}$ such that $\b YUV\overset{\delta_{\alpha}}{\mapsto} W$;

\begin{equation}\tag{\ref{Eq_Step4}}
B^{\alpha}_{Xn-2} \dcap B^{\alpha}_{Yn-1} \dcap B^{\alpha}_{Un} \  \dleq \ B_{Wn} 
\end{equation}
for all $X,Y,U,W\in Q\cup\mathcal{A}$ such that $XYU\b\overset{\delta_{\alpha}}{\mapsto} W$;

\begin{equation}\tag{\ref{Eq_LocalAcceptMarkers}}
B_{\qrej i} \dleq \bar{H}, \ \ \ \  \bar{H} \ \dleq\ \bar{H}^\alpha
\end{equation}
for $1\leqslant i \leqslant n$.

\medskip

Let $\O$ be an ontology consisting of the following axioms: \vspace{-0.2cm}

\begin{gather}\tag{\ref{Eq_GlobalAcceptMarker}}
\bar{H}^1\dcap B_{q_\forall i} \ \dleq \ \bar{H}, \ \ \bar{H}^2\dcap  B_{q_\forall i}  \dleq \ \bar{H} \\
\bar{H}^1\dcap \bar{H}^2\dcap B_{q_\exists i}  \dleq \  \bar{H} \nonumber
\end{gather}
for $1\leqslant i \leqslant m$, $\q_\exists\in Q_\exists$, and $\q_\forall\in Q_\forall$;

\begin{equation}\tag{\ref{Eq_Init}}
A \dleq \mathbin{\rotatebox[origin=c]{180}{$\bigsqcup$}}_{1\leqslant i \leqslant n+2} B_{\b i}  \dcap B_{q_0 n+3} \dcap \mathbin{\rotatebox[origin=c]{180}{$\bigsqcup$}}_{n+4\leqslant i \leqslant m} B_{\b i}
\end{equation}

\begin{equation}\tag{\ref{Eq_CopySignature}}
B_{a i} \dleq B^\alpha_{a i}
\end{equation}
for $\alpha=1,2$, $1\leqslant i \leqslant m$, and all $a\in Q\cup\mathcal{A}$.

\medskip

Consider ontology network $\N$ consisting of the import relations $\langle \O, \Sigma^\alpha, \O^\alpha \rangle$ and $\langle \O^\alpha, \Sigma, \O \rangle$, where $\Sigma^\alpha=\{\bar{H}^\alpha\}\!\cup\!\sigma^\alpha$, $\Sigma=\{\bar{H}\}\!\cup\!\sigma$, and $\alpha=1,2$.  

\medskip

We claim that $M$ does not accept the empty word iff $\O\models_\N A\dleq \bar{H}$. \medskip

($\Rightarrow$): Let $\c$ be a configuration of $M$, $\I$ an interpretation, and $x$ a domain element. We say that $\I$ \emph{$x$-represents} $\c$, if $x$ belongs to the interpretation of the concept $\mathbin{\rotatebox[origin=c]{180}{$\bigsqcup$}}_{1\leqslant i \leqslant m} B_{\c[i] i}$. We show by induction that for any $k$-rejecting configuration $\c$ and any model $\I\models_\N\O$, if $\I$ $x$-represents $\c$, then $x\in \bar{H}^\I$. Then it follows that $\O\models_\N A\dleq \bar{H}$, whenever $M$ does not accept the empty word, because every model $\I\models_\N\O$ $x$-represents $\c_\init$, for $x\in A^\I$, due to axiom (\ref{Eq_Init}).

Let $q$ be the state symbol in $\c$. 
In the induction base $k=0$, we have $q=\qrej$ and hence, by the first axiom in  (\ref{Eq_LocalAcceptMarkers}), it holds $x\in \bar{H}^\I$. In the induction step $k\geqslant 1$, by the definition of network $\N$, for $\alpha=1,2$, there exists a model $\J_\alpha\models_\N\O^\alpha$, which agrees with $\I$ on $\Sigma^\alpha$. Then by axioms (\ref{Eq_CopySignature}), (\ref{Eq_Step1})-(\ref{Eq_Step4}), for $\alpha=1,2$, there is a model $\I_\alpha\models_\N\O$, which agrees with $\J_\alpha$ on $\Sigma$ and $x$-represents a successor configuration $\c_\alpha$ of $\c$ wrt $\delta_\alpha$. If $q\in Q_\forall$ then, since $\c$ is $k$-rejecting, at least one of $\c_1$, $\c_2$ is $(k-1)$--rejecting and thus, by the induction assumption, we must have $x\in \bar{H}^{\I_\alpha}$, for some $\alpha=1,2$. Then by the second axiom in (\ref{Eq_LocalAcceptMarkers}), it holds $x\in (\bar{H}^\alpha)^{\I_\alpha}$ and hence, $x\in (\bar{H}^\alpha)^{\J_\alpha}$, for some $\alpha=1,2$. Then $x\in \bar{H}^\I$, by the first two axioms in (\ref{Eq_GlobalAcceptMarker}). The case $q\in Q_\exists$ is considered similarly.

($\Leftarrow$): Assume $M$ accepts the empty word. Let $\c$ be a configuration of $M$ and let $\I$ be a singleton interpretation with a domain element $x$. We say that $\I$ \emph{represents} configuration $\c$ if $\I$ $x$-represents $\c$. We show that there exists a singleton model agreement $\mu$ for $\N$ and a model $\I\in\mu(\O)$ such that $\I\not\models A\dleq H$. By induction on $k\geqslant 0$ we define families of singleton interpretations $\{\F^k\}_{k\geqslant 0}$ and $\{\F^{\alpha, k}\}_{k\geqslant 0}$, for $\alpha=1,2$, having the following properties:

\begin{itemize}
\item[(a)] for all $\I\in\F^k$ and $\J_\alpha\in\F^{\alpha, k}$, $k\geqslant 0$ and $\alpha=1,2$, it holds $\I\models\O$ and $\J_\alpha\models\O^\alpha$;
\item[(b)] for any $\I\in\F^k$ and $k\geqslant 0$, there exists $\J_\alpha\in\F^{\alpha, k}$, where $\alpha=1,2$, such that $\I=_{\Sigma^\alpha}\J_\alpha$;
\item[(c)] for any $\J_\alpha\in\F^{\alpha, k}$, $k\geqslant 0$, and $\alpha=1,2$, there exists $\I_\alpha\models\F^{k+1}$ such that $\J_\alpha=_{\Sigma}\I_\alpha$;
\item[(d)] any $\I\in\F^k$, $k\geqslant 0$, represents a configuration $\c$ of $M$ and it holds $(\bar{H}^\alpha)^\I\neq\varnothing$ iff either the state symbol in $\c$ is $\qrej$ or $\c$ has a rejecting successor configuration wrt $\delta_\alpha$;
%\item every $\I\in \F^0$ represents $\c_{init}$;
\end{itemize}

In the induction base, let $\F^0$ consist of a singleton interpretation $\I$, which represents $\c_\init$ and has the properties:

\begin{itemize}
\item[(i)] $A^\I\neq\varnothing$, $\bar{H}^\I=\varnothing$;
\item[(ii)] for all $a\in Q\cup\mathcal{A}$, $1\leqslant i \leqslant n$, and $\alpha=1,2$, the interpretations of $(B_{a i})$ and $B^\alpha_{a i}$ coincide;
\item[(iii)] for all $\alpha=1,2$, the interpretation of $\bar{H}^\alpha$ is not empty iff $\c_\init$ has a rejecting successor configuration wrt $\delta_\alpha$.
\end{itemize}

Clearly, $\I$ is a model of axioms (\ref{Eq_Init}),(\ref{Eq_CopySignature}) and since $\c_\init$ is an accepting configuration, $\I$ is a model of axioms (\ref{Eq_GlobalAcceptMarker}), which means that $\I\models\O$.

In the induction step for $k\geqslant 1$, take an arbitrary model $\I\in\F^{k-1}$. Let $\I$ represent a configuration $\c$ and let $\c_\alpha$ be a successor configuration of $\c$ wrt $\delta_\alpha$, for $\alpha=1,2$ (in the case, when $\c$ does not have a successor configuration wrt $\delta_\alpha$, we set $\c_\alpha=\c$, for $\alpha=1,2$). 

For $\alpha=1,2$, let $\J_\alpha$ be an interpretation which agrees with $\I$ on $\Sigma^\alpha$ and has the following properties: 
\begin{itemize}
\item $\J_\alpha$ represents $\c_\alpha$;
\item $\bar{H}^{\J_\alpha}=(\bar{H}^\alpha)^{\I}$.
\end{itemize}

One can readily verify that $\J_\alpha$ is a model of axioms (\ref{Eq_Step1})-(\ref{Eq_Step4}). By the induction assumption, $\I$ has property (d), hence, $\J_\alpha$ is a model of axioms  (\ref{Eq_LocalAcceptMarkers}) and therefore, $\J_\alpha\models\O^\alpha$. For $\alpha=1,2$, let $\I_\alpha$ be an interpretation, which agrees with $\J_\alpha$ on $\Sigma$ and has the following properties:

\begin{itemize}
\item $A^{\I_\alpha}=\varnothing$;
\item $\I_\alpha$ represents $\c_\alpha$;
\item $\I_\alpha$ satisfies (ii);
\item  $(\bar{H}^\alpha)^{\I_\alpha}\neq\varnothing$ iff either the state symbol in $\c$ is $\qrej$ or $\c_\alpha$ has a rejecting successor configuration wrt $\delta_\alpha$.
\end{itemize}

Clearly, $\I_\alpha$, $\alpha=1,2$, is a model of axioms (\ref{Eq_Init})-(\ref{Eq_CopySignature}). Since $\J_\alpha=_\Sigma\I_\alpha$, we have $\bar{H}^{\I_\alpha}=\bar{H}^{\J_\alpha}$. Then by the definition of interpretation $\bar{H}^{\J_\alpha}$ it follows that $\I_\alpha$ is a model of axioms (\ref{Eq_GlobalAcceptMarker}) and thus, $\I_\alpha\models\O$. 

\medskip

For all $\alpha=1,2$ and $k\geqslant 0$, let $\F^{\alpha, k}$ be the family of interpretations $\J_\alpha$ defined for a model $\I\in\F^{k}$, as described above. Similarly, let $\F^{k+1}$ be the family of interpretations $\I_\alpha$ defined for a model $\J_\alpha\in\F^{\alpha, k}$, for some $\alpha=1,2$, as above. By definition, the families of interpretations $\{\F^k\}_{k\geqslant 0}$ and $\{\F^{\alpha, k}\}_{k\geqslant 0}$, $\alpha=1,2$, have properties (a)-(c). Then a mapping $\mu$ defined as $\mu(\O)=\bigcup_{k\geqslant 0}\F^k$, $\mu(\O^\alpha)=\bigcup_{k\geqslant 0}\F^{\alpha,k}$, for $\alpha=1,2$, is a model agreement for $\N$ and there is a model $\I\in\F^0\subseteq\mu(\O)$ such that $\I\not\models A\dleq \bar{H}$, which means that $\O\not\models_\N A\dleq \bar{H}$.
\end{proof}

%===================================

\commentout{ %%% OLD PROOFS OF UPPER BOUNDS

\noindent\textbf{Corollary \ref{Cor_UpperBoundGeneral}} (Upper Bound for General Networks).
\textit{Given a finite $v$-pointed ontology network $\N=(V,E,\tau)$, the set of concept inclusions $C\dleq D$ in signature $\sig(\O_v)$ such that $(\N,v)\models C\dleq D$ is recursively enumerable.}

\medskip

\Todo{Finiteness is stressed, since we allow networks to be infinite in Definition \ref{De_TreeUnfolding}}

\begin{proof} 
Let $\N_t=(V_t,E_t,\tau_t)$ be a renamed unfolding of $\N$. By Theorem \ref{Teo_Reduction2Union}, $(\N,v)\models C\dleq D$ is equivalent to $\bigcup_{u\in V_t}\O_u\models C\dleq D$ which, by compactness, holds iff there is a finite subset of nodes $W\subseteq V_t$ such that $\bigcup_{w\in W}\O_w\models C\dleq D$. Thus, the proposition is proved if we show that $\N_t$ is recursively enumerable, i.e. the sets $(V_t,E_t)$ and the graph of the labelling function $\tau_t$ (considered as a set of pairs) are enumerable. 

We describe an algorithmic procedure which for a natural number $n\geqslant 0$ gives an ontology network $\N_t^n=(V_t^n,E_t^n,\tau_t^n)$, with the set $V_t^n$ consisting of sequences ${[}v_0\ldots v _n{]}$ of nodes from $\N$, such that $\N_t^n$ satisfies all the conditions of tree unfolding for $\N$. 

It is clear from Definition \ref{De_TreeUnfolding} that there is a recursive procedure to construct the graph $(V_t^n,E_t^n)$ for a given $n$, so we may assume that this graph is given and our goal now is to define the labelling function $\tau_t^n$. It suffices to define a permutation function $\pi_u$ for each node $u\in V_t^n$; we give a definition by simultaneous induction on $k$, $0\leqslant k \leqslant n$ and the number $m$ of nodes reachable from ${[}v_0{]}$ in $k$ steps. 

For $k=m=0$, we define $\tau_t^n({[}v_0{]})=\tau[v]$ and set $\pi_{{[}v_0{]}}$ to be the identity function on $\Nc\cup\Nr$.

For $k,m\geqslant 1$, let $\Gamma$ be the union of signatures of label ontologies for all those nodes $u$ for which $\tau_t$ is already defined. Let $w={[}v_0,\ldots v_k{]}$ be a $m$-th node reachable from ${[}v_0{]}$ in $k$ steps and $u={[}v_0,\ldots v_{k-1}{]}$ be the parent of $w$ wrt an edge $e$. For some signature $\Si$, it holds $\langle v_{k-1}\Si v_k\rangle\in\N$; denote $\Delta=\sig(\O_{v_k})\setminus\Si$. Then $\pi_w$ is defined as a permutation on $\Nc\cup\Nr$ mapping $\Nc$ to $\Nc$ and $\Nr$ to $\Nr$ and having the following properties:

\begin{itemize} 
\item[a.] $\pi_w|_\Si=\pi_u|_\Si$; 
\item[b.] for all $s\in\Delta$, it holds $\pi_w(s)\not\in\Gamma$ and $\pi_w(\pi_w(s))=s$ (denote the image of $\Delta$ wrt $\pi_w$ as $\Delta^*$);
%\item for all $s\in\Delta^*$, it holds $\pi_w(s)=\pi^-1_w(s)$.
\item[c.] $\pi_w$ is the identity function on $(\Nc\cup\Nr)\setminus(\sig(\O_{v_k})\cup\Delta^*)$.
\end{itemize}

Now, we define the label $\tau^n_t(w)$ as  a ``copy'' of ontology $\tau_{w})$ wrt the renaming given by $\pi_w$ and set $\tau^n_t(e)=\{\pi_u(s) \mid s\in\Si\}$. 

Since the alphabet $\Nc\cup\Nr$ is recursively enumerable, it follows that there is an algorithmic procedure which for a given tree $(V_t^n,E_t^n)$ enumerates the graph of the labelling function $\tau_t^n$. Note that $\tau_t^n$ satisfies all the conditions in Definition \ref{De_TreeUnfolding} of tree unfolding: all the conditions except 1.3 are immediate to verify, while 1.3 follows by induction from points a,b in the definition of $\tau_t^n$. \end{proof}

} %%% commentout %%% OLD PROOFS OF UPPER BOUNDS

%%====================================================================%%

\end{document}